\documentclass{article}

\usepackage{microtype}
\usepackage{graphicx}
\usepackage{subcaption}
\usepackage{booktabs} % for professional tables
\usepackage{hyperref}

\usepackage[accepted]{icml2023}

\usepackage{amsmath}
\usepackage{amssymb}
\usepackage{mathtools}
\usepackage{amsthm}

\usepackage[capitalize,noabbrev]{cleveref}

\theoremstyle{plain}
\newtheorem{theorem}{Theorem}[section]
\newtheorem{proposition}[theorem]{Proposition}

\theoremstyle{definition}

\theoremstyle{remark}

\usepackage[textsize=tiny]{todonotes}

\hypersetup{
    colorlinks=true,
    linkcolor=blue,
    filecolor=magenta,      
    urlcolor=cyan,
}

\usepackage{url}            % simple URL typesetting
\usepackage{booktabs}       % professional-quality tables
\usepackage{amsfonts}       % blackboard math symbols
\usepackage{nicefrac}       % compact symbols for 1/2, etc.
\usepackage{microtype}      % microtypography

\usepackage{float}
\usepackage{mathtools}
\usepackage{amsthm}
\usepackage{amssymb}
\usepackage{amsfonts}
\usepackage{graphicx}
\usepackage{booktabs}
\usepackage{siunitx}
\usepackage{arydshln}
\usepackage{caption}
\usepackage{subcaption}
\usepackage{soul}
\usepackage{epstopdf}
\usepackage{fancyvrb}
\usepackage{cancel}
\usepackage{makecell}

\newcommand{\tr}{\text{tr}}
\newcommand{\Var}{\text{Var}}

\usepackage{pifont}
\def\num#1{%
        \raisebox{.9pt}{\textcircled{\raisebox{-.9pt}{#1}}}%
}

\usepackage{tcolorbox}
\definecolor{tabblue}{RGB}{76,156,201}

\makeatletter
\newcommand{\longdash}[1][2em]{%
  \makebox[#1]{$\m@th\smash-\mkern-7mu\cleaders\hbox{$\mkern-2mu\smash-\mkern-2mu$}\hfill\mkern-7mu\smash-$}}
\makeatother
\newcommand{\omitskip}{\kern-\arraycolsep}

\usepackage{algorithm}
\usepackage{algpseudocode}

\usepackage{overpic}

\usepackage{amssymb}% http://ctan.org/pkg/amssymb
\usepackage{pifont}% http://ctan.org/pkg/pifont
\newcommand{\cmark}{\ding{51}}%
\newcommand{\xmark}{\ding{55}}%

\definecolor{lightgraysquare}{rgb}{0.7,0.7,0.7}

\usepackage{listings}

\definecolor{dkgreen}{rgb}{0,0.6,0}
\definecolor{dkred}{rgb}{0.6,0.0,0}
\definecolor{gray}{rgb}{0.5,0.5,0.5}
\definecolor{mauve}{rgb}{0.58,0,0.82}
\definecolor{lightgray}{HTML}{EEEEEE}

\usepackage{wrapfig}
\usepackage{tabularx}
\usepackage{arydshln}
\usepackage{array}
\usepackage{multirow}

\newcommand{\boldtheta}{{\boldsymbol{\theta}}}
\newcommand{\boldepsilon}{\boldsymbol{\epsilon}}
\newcommand{\boldxi}{\boldsymbol{\xi}}

\newcommand{\bolds}{\boldsymbol{s}}
\newcommand{\boldx}{\boldsymbol{x}}

\newcommand{\boldc}{\mathbf{c}}
\newcommand{\boldA}{\mathbf{A}}

\newcommand{\boldu}{\mathbf{u}}

\newcommand{\boldg}{\boldsymbol{g}}

\newcommand{\boldv}{\boldsymbol{v}}
\newcommand{\boldI}{\mathbf{I}}

\newcommand{\boldzero}{\boldsymbol{0}}
\newcommand{\boldone}{\boldsymbol{1}}

\newcommand{\hatg}{\hat{\boldg}}
\newcommand{\ges}{\hat{\boldg}^{\text{ES}}}
\newcommand{\gpes}{\hat{\boldg}^\text{PES}}
\newcommand{\gessingle}{\hat{\boldg}^\text{ES-Single}}
\newcommand{\gesgen}{\hat{\boldg}^\text{ES-Gen}}

\newcommand{\gesanti}{\hat{\boldg}^{\text{ES-A}}}

\makeatletter
\def\gpesanti{%
  \@ifnextchar^%
    {\@gpesanti}
    {\@gpesanti^{}}%
}
\def\@gpesanti^#1{%
  \hat{\boldg}^{\text{PES-A} #1}%
}
\makeatother

\makeatletter
\def\gesgen{%
  \@ifnextchar^%
    {\@gesgen}
    {\@gesgen^{}}%
}
\def\@gesgen^#1{%
  \hat{\boldg}^{\text{ES-Gen} #1}%
}
\makeatother

\icmltitlerunning{ES-Single}

\begin{document}

\twocolumn[
\icmltitle{Low-Variance Gradient Estimation in \\Unrolled Computation Graphs with ES-Single}

\begin{icmlauthorlist}
\icmlauthor{Paul Vicol}{google}
\icmlauthor{Zico Kolter}{cmu,bosch}
\icmlauthor{Kevin Swersky}{google}
\end{icmlauthorlist}

\icmlaffiliation{google}{Google Brain}
\icmlaffiliation{cmu}{Carnegie Mellon University}
\icmlaffiliation{bosch}{Bosch Center for AI}

\icmlcorrespondingauthor{Paul Vicol}{paulvicol@google.com}
\icmlkeywords{evolution strategies, ES-Single, persistent evolution strategies, unrolled computation graphs, gradient estimation, low-variance}
\vskip 0.3in
]

\printAffiliationsAndNotice{}  % leave blank if no need to mention equal contribution

\begin{abstract}
We propose an evolution strategies-based algorithm for estimating gradients in unrolled computation graphs, called ES-Single. Similarly to the recently-proposed Persistent Evolution Strategies (PES), ES-Single is unbiased, and overcomes chaos arising from recursive function applications by smoothing the meta-loss landscape. ES-Single samples a single perturbation per particle, that is kept fixed over the course of an inner problem (e.g., perturbations are not re-sampled for each partial unroll). Compared to PES, ES-Single is simpler to implement and has lower variance: the variance of ES-Single is constant with respect to the number of truncated unrolls, removing a key barrier in applying ES to long inner problems using short truncations. We show that ES-Single is unbiased for quadratic inner problems, and demonstrate empirically that its variance can be substantially lower than that of PES. ES-Single consistently outperforms PES on a variety of tasks, including a synthetic benchmark task, hyperparameter optimization, training recurrent neural networks, and training learned optimizers.
\end{abstract}

\section{Introduction}

Many problems in machine learning involve computing gradients through unrolled computation graphs, including bilevel optimization (such as hyperparameter optimization~\cite{domke2012generic,maclaurin2015gradient,franceschi2017forward,shaban2019truncated} and meta-learning~\cite{bertinetto2018meta,finn2018learning,finn2018probabilistic}), RNN training~\cite{merity2017regularizing}, reinforcement learning~\cite{salimans2017evolution,mania2018simple}, and training learned optimizers~\cite{metz2019understanding,metz2018meta,metz2020tasks,metz2020using,wichrowska2017learned,andrychowicz2016learning,li2016learning,li2017learning}.
In each of these tasks, we wish to learn parameters that govern the evolution of a dynamical system, such that the states produced by the system satisfy some objective function.
For example, in hyperparameter optimization, we aim to tune hyperparameters (e.g., the learning rate or dropout coefficient), such that a model trained using these hyperparameters achieves low loss---in this case, the hyperparameters govern the evolution of the model parameters, that can be interpreted as the states of the dynamical system.
Classic approaches to computing gradients through unrolled computation graphs include reverse-mode~\cite{werbos1990backpropagation} and forward-mode~\cite{williams1990efficient,franceschi2017forward,tallec2017unbiased,menick2021practical} gradient accumulation.
\footnote{We note that another work, developed in parallel and independently, proposes a very similar method~\cite{li2023noise}.}

\begin{figure}
    \centering
    \includegraphics[width=0.85\linewidth]{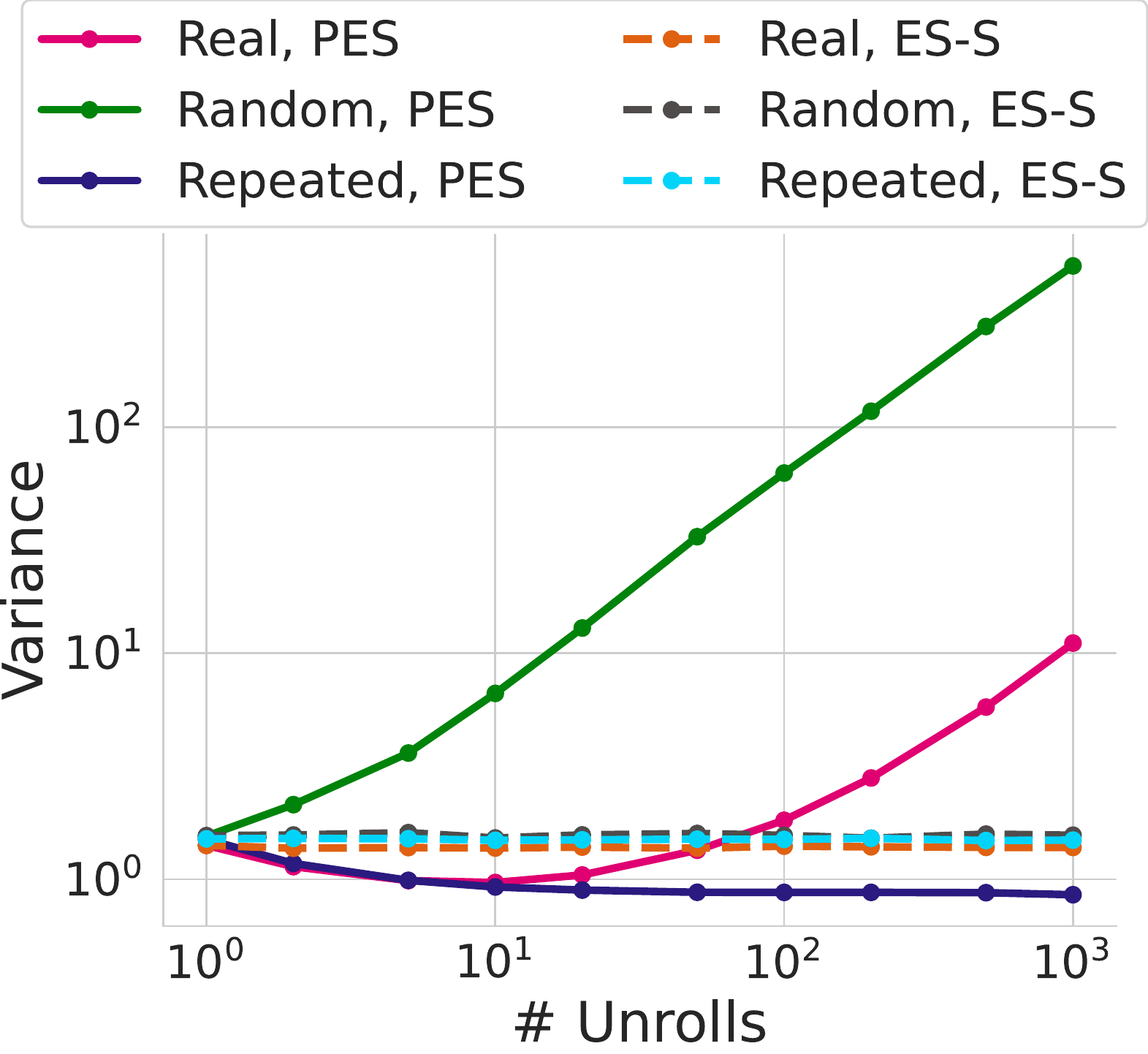}
    \vspace{-0.4cm}
    \caption{Comparing empirical variance measurements for PES (solid curves) and ES-Single (ES-S, dashed lines) on a small LSTM training task.
    Unlike PES, under all conditions the variance of ES-Single is constant as the number of partial unrolls increases.
    }
    \label{fig:variance}
    \vspace{-0.5cm}
\end{figure}

However, gradient-based methods face a fundamental obstacle: the loss landscape resulting from long unrolls is often chaotic and can exhibit near discontinuities~\cite{metz2019understanding,parmas2018pipps,parmas2019unified}, rendering gradients unsuitable~\cite{metz2021gradients}.
One approach to overcome this chaos is to consider Gaussian smoothing of the outer loss surface; the gradient of such a smoothed objective can be computed using evolution strategies (ES)~\cite{rechenberg1973}.
Applying vanilla ES to the full unrolled inner problem yields a useful gradient estimate, but leads to slow outer optimization; in contrast, applying ES to truncated unrolls leads to truncation bias, similarly to truncated BPTT.

\citet{vicol2021unbiased} proposed an ES-based algorithm called Persistent Evolution Strategies (PES), that yields unbiased gradient estimates from truncated unrolls, speeding up meta-optimization by allowing for more frequent outer parameter updates.
PES has a number of desirable characteristics, including unbiasedness and Gaussian smoothing of the outer loss landscape.
However, its variance increases with the number of truncated unrolls per full inner problem, potentially making it impractical to use short truncations for long-horizon problems (for example, truncations of length $K=1$ for problems where $T \geq 1000$).

In this paper, we propose an unbiased gradient estimator for unrolled computation graphs, called ES-Single, that re-uses the same outer parameter perturbations along each step of an unrolled trajectory.
ES-Single has constant variance with respect to the number of partial unrolls per inner problem.
Due to its low variance, ES-Single outperforms PES on a wide range of synthetic and real-world tasks.
In addition, ES-Single is simpler to implement than PES---as it does not require maintaining a perturbation accumulator per particle---and computationally cheaper---as it only samples perturbations at the start of each inner problem, rather than for each truncated unroll.
\paragraph{Contributions.}
\vspace{-0.3cm}
\begin{itemize}
    \item We propose an algorithm for ES-based, unbiased gradient estimation in unrolled computation graphs, called ES-Single. We motivate ES-Single by discussing its relationship to full-unroll ES and PES.
    \item We show that ES-Single can have substantially lower variance than PES, overcoming a key barrier for use in long-horizon inner problems, especially when using short truncated unrolls.
    \item We evaluate ES-Single on a diverse set of tasks, from synthetic problems designed to test unbiasedness, to hyperparameter optimization, RNN training, and meta-training learned optimizers. We found that ES-Single outperformed PES across all tasks we investigated.
\end{itemize}
\vspace{-0.1cm}
We provide JAX code for ES-Single in Appendix~\ref{app:code}, and a \href{https://colab.research.google.com/drive/1fgSzwaIXfJKbYTntEFfNbc2UuTXwEw0A?usp=sharing}{Colab notebook implementation here}.
\vspace{-0.1cm}
\begin{figure*}[t]
    \includegraphics[width=\linewidth]{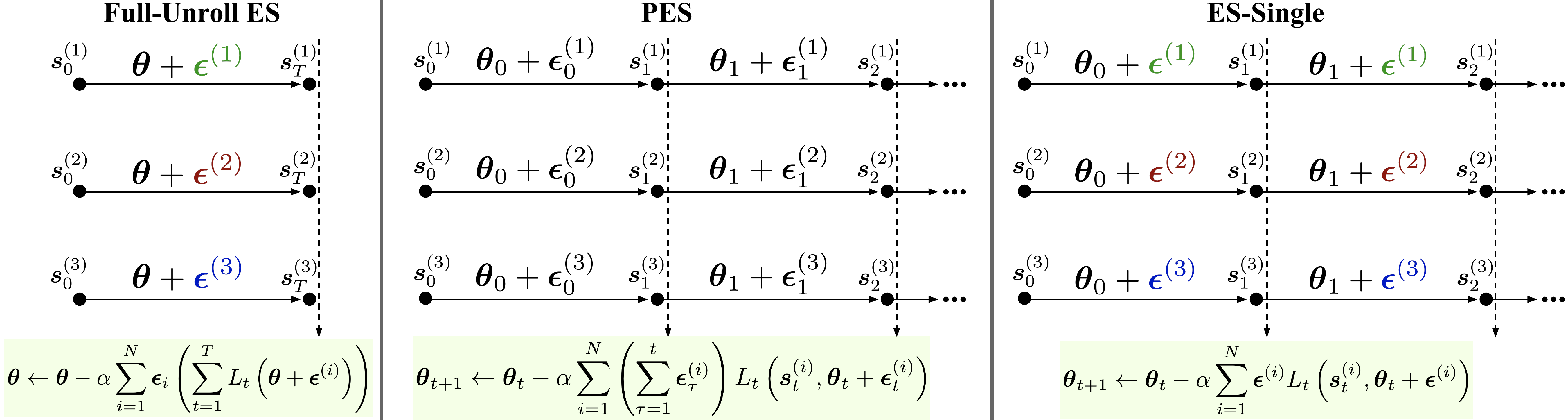}
    \vspace{-0.6cm}
    \caption{\textbf{Comparison of the computation graphs of full-unroll ES (left), PES (middle), and ES-Single (right).} Full-Unroll ES samples perturbation $\boldepsilon^{(i)}$ for particle $i$, and runs a full unroll from state $\bolds_0^{(i)}$ to $\bolds_T^{(i)}$ using perturbed parameters $\boldtheta + \boldepsilon^{(i)}$.
    Both PES and ES-Single split the computation graph into a sequence of partial unrolls, but differ in how perturbations are sampled and how intermediate results are aggregated to update the outer parameters online: PES samples a new perturbation $\boldepsilon^{(i)}_t$ for particle $i$ in each unroll $t$, and sums the perturbations experienced by each particle up to the current point in the inner problem, $\sum_{\tau=1}^t \boldepsilon^{(i)}_\tau$; in contrast, ES-Single samples a single perturbation $\boldepsilon^{(i)}$ per particle at the start of each inner problem---keeping it fixed for the duration of the problem---and does not sum perturbations over time.
    ES-Single can be interpreted as inserting \textit{breakpoints into the full-unroll ES computation}, at which the intermediate losses are aggregated to form a gradient estimate used to update the outer parameters.
    The computation graph for vanilla truncated ES is provided in Appendix~\ref{app:truncated-es}.}
    \label{fig:es-single-diagrams}
    \vspace{-0.2cm}
\end{figure*}
\section{Background}
\label{sec:background}
We follow the problem setup of~\citet{vicol2021unbiased}, considering an unrolled computation graph with state $\bolds_t$ at time $t$, updated via a function $f$ parameterized by $\boldtheta$:
\begin{align}
    \bolds_t = f(\bolds_{t-1}, \boldx_t ; \boldtheta)
\end{align}
where $\boldx_t$ is an optional input at each step (e.g., data).
Many common tasks in machine learning are instances of this problem: for example, when training an RNN, $\bolds_t$ is the hidden state and $\boldtheta$ are the RNN parameters, while for hyperparameter optimization, $\bolds_t$ are the parameters of a neural network being optimized and $\boldtheta$ are hyperparameters such as the learning rate and momentum.
The objective function for optimizing $\boldtheta$ is the sum of per-timestep losses $L_t(\bolds_t ; \boldtheta)$:
\begin{align}
    L(\boldtheta) = \sum_{t=1}^T L_t(\bolds_t ; \boldtheta)
\end{align}
Appendix~\ref{app:notation} summarizes the notation used in this paper.

\paragraph{Chaos and Smoothing.}
Unrolling computation graphs can give rise to chaotic dynamics, which pose fundamental challenges for gradient-based methods~\cite{parmas2019unified}.
Chaos frequently arises in meta-optimization (for example hyperparameter optimization and training learned optimizers), due to nonlinear inner loss surfaces which have many local minima---small changes to the outer parameters may lead to different local minima, that have different meta-loss values.
To overcome chaos, one may consider optimizing a Gaussian-smoothed outer loss, $\tilde{L}(\boldtheta) = \mathbb{E}_{\tilde{\boldtheta} \sim \mathcal{N}(\boldtheta, \sigma^2 \boldI)} \left[ L(\tilde{\boldtheta}) \right]$.
Common approaches for computing the gradient of the smoothed objective $\tilde{L}(\boldtheta)$ include evolution strategies and the reparameterization gradient~\cite{ruiz2016generalized}.

\vspace{-0.2cm}
\paragraph{Vanilla Evolution Strategies.}
Evolution strategies~\cite{rechenberg1973,nesterov2017random} is a method for zeroth-order gradient estimation, that computes a stochastic finite-difference estimate of the gradient as follows:
\begin{align*}
    \ges
    &=
    \frac{1}{\sigma^2} \mathbb{E}_{\boldepsilon \sim \mathcal{N}(\boldzero, \sigma^2 \boldI)} \left[ \boldepsilon L(\boldtheta + \boldepsilon) \right] \\
    &\approx
    \frac{1}{\sigma^2 N} \sum_{i=1}^N \boldepsilon^{(i)} L(\boldtheta + \boldepsilon^{(i)})
\end{align*}
where $N$ is the number of Monte Carlo samples (also called \textit{particles}) used to estimate the expectation.
Antithetic sampling~\citep{mcbook} is a widely-used technique to reduce the variance of ES, that works by sampling pairs of positive and negative perturbations, $\ges = \mathbb{E}_{\boldepsilon \sim \mathcal{N}(\boldzero, \sigma^2 \boldI)} \left[ \boldepsilon (L(\boldtheta + \boldepsilon) - L(\boldtheta - \boldepsilon)) \right]$.
In practice, one typically uses a Monte Carlo estimate as follows:
\begin{equation}
    \label{eq:es-antithetic}
    \gesanti = \frac{1}{N \sigma^2} \sum_{i=1}^{N/2} \boldepsilon^{(i)} (L(\boldtheta + \boldepsilon^{(i)}) - L(\boldtheta - \boldepsilon^{(i)}))
\end{equation}
where $N$ is even, and $\boldepsilon^{(i)} \sim \mathcal{N}(\boldzero, \sigma^2 \boldI)$.
Unfortunately, applying vanilla ES to long inner problems leads to slow updates, while using partial unrolls leads to truncation bias~\cite{metz2019understanding}.

\paragraph{Persistent Evolution Strategies (PES).}
PES~\cite{vicol2021unbiased} is an ES-based approach for unbiased gradient estimation using partial unrolls of the inner problem.
The PES gradient estimator is defined as follows, where $\boldtheta_t$ denotes the application of the shared parameters $\boldtheta$ at step $t$ and the loss $L_t$ is written explicitly as a function of all applications of $\boldtheta$ up to the current time, rather than as a function of the state $\bolds_t$ that implicitly depends on past $\boldtheta$'s:
\vspace{-0.3cm}
\begin{align*}
    \gpes
    =
    \frac{1}{\sigma^2} \mathbb{E}_{\boldepsilon} \left[ \sum_{t=1}^T \left( \sum_{\tau=1}^t \boldepsilon_{\tau} \right) L_t(\boldtheta_1 + \boldepsilon_1, \dots, \boldtheta_t + \boldepsilon_t) \right]
\end{align*}
%\vspace{-0.2cm}
Here, the expectation is over an $T \times P$ matrix of per-timestep perturbations.
Intuitively, PES maintains a collection of particles, and applies a different outer parameter perturbation for each partial unroll of the inner problem.
The particles are not reset after each partial unroll, and the perturbations experienced by each particle are accumulated over the course of an inner problem.
The particle states and accumulators reset at the start of a new inner problem.
\citet{vicol2021unbiased} showed that the variance of PES depends on the covariance between gradients of each loss term $L_t$ with respect to per-timestep parameters $\boldtheta_{\tau}$.
They analyzed several scenarios with different covariance assumptions, and found that in a real-world scenario, the variance increases as the number of unrolls per inner problem increases (Figure~\ref{fig:variance}); thus, while PES works well for tasks with an intermediate number of unrolls (e.g., 10-100 unrolls), it struggles with longer tasks due to variance.

\begin{figure*}[t]
\vspace{-0.3cm}
\begin{minipage}[t]{0.48\textwidth}
\begin{algorithm}[H]
  \caption{Truncated Evolution Strategies (ES) applied to partial unrolls of a computation graph.}
  \label{alg:es-again}
\begin{algorithmic}
    \State \textbf{Input:} $\bolds_0$, initial state
    \State \hspace{2.7em} $K$, truncation length for partial unrolls
    \State \hspace{2.7em} $N$, number of particles
    \State \hspace{2.9em} $\sigma$, standard deviation of perturbations
    \State \hspace{2.9em} $\alpha$, learning rate for outer optimization
    \State Initialize $\bolds = \bolds_0$ ${\color{white}\bolds^{(i)} = \bolds_0}$
    \While {inner problem not finished}
        \State $\ges \gets \boldzero$
        \For{$i=1,\dots, N$}
            \State $\boldepsilon^{(i)} = 
                \left\{\begin{array}{lcl}
                	\text{draw from } \mathcal{N}(0, \sigma^2 \boldI) &  & i \text{ odd} \\
                	-\boldepsilon^{(i-1)} &  & i \text{ even}
                \end{array}\right.$
            \State $\hat{L}_K^{(i)} \gets \text{unroll}(\bolds, \boldtheta + \boldepsilon^{(i)}, K)$
            \State $\ges \gets \ges + \boldepsilon^{(i)} \hat{L}_K^{(i)}$
        \EndFor
        \State $\ges \gets \frac{1}{N \sigma^2} \ges$
        \State $\bolds \gets \text{unroll}(\bolds, \boldtheta, K)$
        \State $\boldtheta \gets \boldtheta - \alpha \ges$
    \EndWhile
\end{algorithmic}
\end{algorithm}
\end{minipage}
\hfill
\begin{minipage}[t]{0.48\textwidth}
\begin{algorithm}[H]
  \caption{ES with a single perturbation per particle re-applied in each truncated unroll (ES-Single).}
  \label{alg:es-single}
\begin{algorithmic}
    \State \textbf{Input:} $\bolds_0$, initial state
    \State \hspace{2.7em} $K$, truncation length for partial unrolls
    \State \hspace{2.7em} $N$, number of particles
    \State \hspace{2.9em} $\sigma$, standard deviation of perturbations
    \State \hspace{2.9em} $\alpha$, learning rate for outer optimization
    \State Initialize ${\color{dkred}\bolds^{(i)}} = \bolds_0$ for {\color{dkred} $i \in \{1, \dots, N\}$}
    \For{${\color{dkred}i=1,\dots,N}$}
        \State ${\color{dkred}\boldepsilon^{(i)} = 
                    \left\{\begin{array}{lcl}
                    	\text{draw from } \mathcal{N}(0, \sigma^2 \boldI) &  & i \text{ odd} \\
                    	-\boldepsilon^{(i-1)} &  & i \text{ even}
                    \end{array}\right.}$
    \EndFor
    \hspace{-0.5cm}\While {inner problem not finished}
        \State $\gessingle \gets \boldzero$
        \For{$i=1,\dots, N$}
            \State {\color{dkred} $\bolds^{(i)}$}, $\hat{L}_K^{(i)} \gets \text{unroll}({\color{dkred} \bolds^{(i)}}, \boldtheta + \boldepsilon^{(i)}, K)$
            \State $\gessingle \gets \gessingle + {\color{dkred} \boldepsilon^{(i)}} \hat{L}_K^{(i)}$
        \EndFor
        \State $\gessingle \gets \frac{1}{N \sigma^2} \gessingle$
        \State {\color{white}$s \gets \text{unroll}(\bolds, \theta, K)$}
        \State $\boldtheta \gets \boldtheta - \alpha \gessingle$
    \EndWhile
\end{algorithmic}
\end{algorithm}
\end{minipage}
\vspace{-0.2cm}
\caption{\textbf{A comparison of the vanilla truncated ES and ES-Single gradient estimators}, applied to partial unrolls of a computation graph.
The conditional statement for $\boldepsilon^{(i)}$ is used to implement antithetic sampling.
Differences between the two algorithms are {\color{dkred} highlighted in red.}
While ES samples different perturbations for each particle in each partial unroll, ES-Single samples one perturbation per particle before the inner problem starts, and re-applies the same perturbation in each partial unroll comprising the inner problem.
}
\label{fig:both-algorithms}
\vspace{-0.2cm}
\end{figure*}
\vspace{-0.2cm}
\section{ES-Single}

We propose an algorithm for ES-based gradient estimation in unrolled computation graphs, that is simpler to implement than PES, has low variance, and performs well on a variety of tasks.
To introduce this algorithm, we first revisit full-unroll ES, which computes $\frac{1}{\sigma^2} \mathbb{E}_{\boldepsilon}[ \boldepsilon L(\boldtheta + \boldepsilon) ] \approx \frac{1}{N \sigma^2} \sum_{i=1}^N \boldepsilon_i L(\boldtheta + \boldepsilon_i)$.
% This estimator is inspired by full-unroll ES: 
In full-unroll ES, we initialize $N$ particles, sample an outer parameter perturbation $\boldepsilon_i$ for each particle, unroll the full inner problem using the perturbed outer parameters $\boldtheta + \boldepsilon_i$, and aggregate the results to form the gradient estimate.
Full-unroll ES yields useful gradient estimates, but is impractical due to the high latency between parameter updates (which is especially problematic for tasks such as hyperparameter optimization, where the inner problem typically has length $T > 10^3$).
However, one can consider splitting the computation graph into a series of truncated unrolls, and using the intermediate results obtained after each unroll to update the outer parameters more frequently.
The ES-Single algorithm inserts breakpoints in the inner optimization, at which the intermediate results (e.g., the losses from the current unroll) are aggregated to form a gradient estimate that is used to update the outer parameters.
Mathematically, the gradient estimator for ES-Single is equivalent to the full-unroll ES gradient (shown here using antithetic sampling),
\begin{equation}
\gessingle = \frac{1}{\sigma^2} \mathbb{E}_{\boldepsilon} \left[ \boldepsilon (L(\boldtheta + \boldepsilon) - L(\boldtheta - \boldepsilon)) \right],
\end{equation} where $\boldepsilon \sim \mathcal{N}(\boldzero, \sigma^2 \boldI)$.
However, ES-Single differs from full ES \textit{algorithmically}: full ES treats the inner problem as a black box, ignoring its iterative nature; in contrast, ES-Single treats it as a \textit{gray-box}, which leverages this structure by constructing gradient estimates from intermediate progress, and updating the outer parameters online.

ES-Single samples perturbations $\boldepsilon^{(i)}$ for each particle once, at the start of an inner problem, which are then kept fixed for the entirety of the inner problem---the same perturbations are applied to the outer parameters at each partial unroll.
Because ES-Single is mathematically equivalent to full-unroll ES, it is unbiased by construction:
\begin{tcolorbox}[colback=tabblue!15,boxrule=0pt,colframe=white,coltext=black,arc=2pt,outer arc=0pt,valign=center]%,valign=center]
\begin{proposition}[ES-Single is unbiased]
Assume that $L(\boldtheta)$ is quadratic and $\nabla_{\boldtheta} L(\boldtheta)$ exists. Then, the ES-Single gradient estimator using antithetic sampling is unbiased, that is, $\text{bias}(\hat{\boldg}^{\text{ES-Single}}) = \mathbb{E}_{\boldepsilon} \left[ \hat{\boldg}^{\text{ES-Single}} \right] - \nabla_{\boldtheta} L(\boldtheta) = 0$.
\end{proposition}
\begin{proof}
The proof is provided in Appendix~\ref{app:unbiased-proof}.
\end{proof}
\end{tcolorbox}
The main difference between ES-Single and truncated ES is that ES-Single maintains separate states for each particle throughout the full inner problem (that are updated in parallel in each partial unroll), rather than collapsing the particles to update a single mean state $\bolds$ after each truncated unroll.
Also, ES-Single differs from PES in two key ways: 1) it uses the same perturbations over all partial unrolls of an inner problem, rather than sampling new perturbations for each partial unroll; and 2) it does not accumulate the perturbations over time, as done in PES.
Thus, ES-Single is simple to implement, and may be slightly cheaper per iteration if the cost of sampling perturbations is high (e.g., for high-dimensional parameters).

\paragraph{Stochastic Computation Graph.}
\begin{figure}[H]
    \hspace{0.8cm} \textbf{ES-Single} \hspace{3.4cm} \textbf{PES} \\[1em]
    \includegraphics[width=0.45\linewidth]{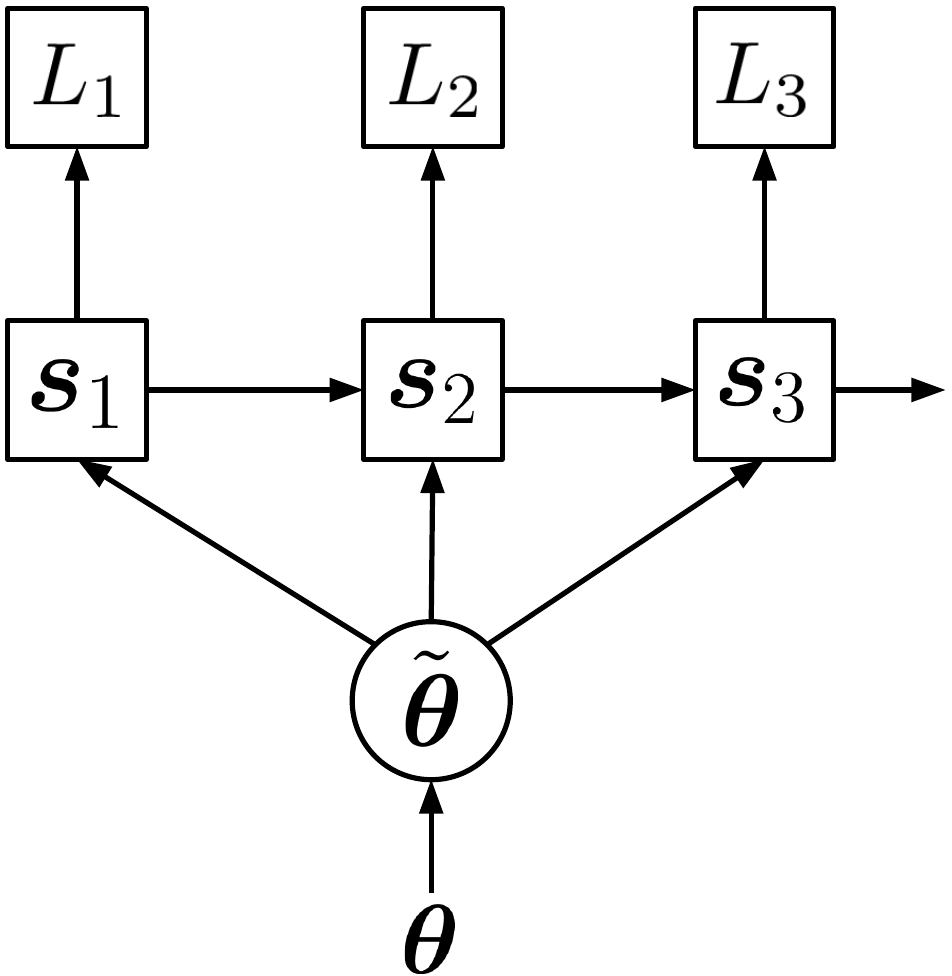}
    \hfill
    \includegraphics[width=0.45\linewidth]{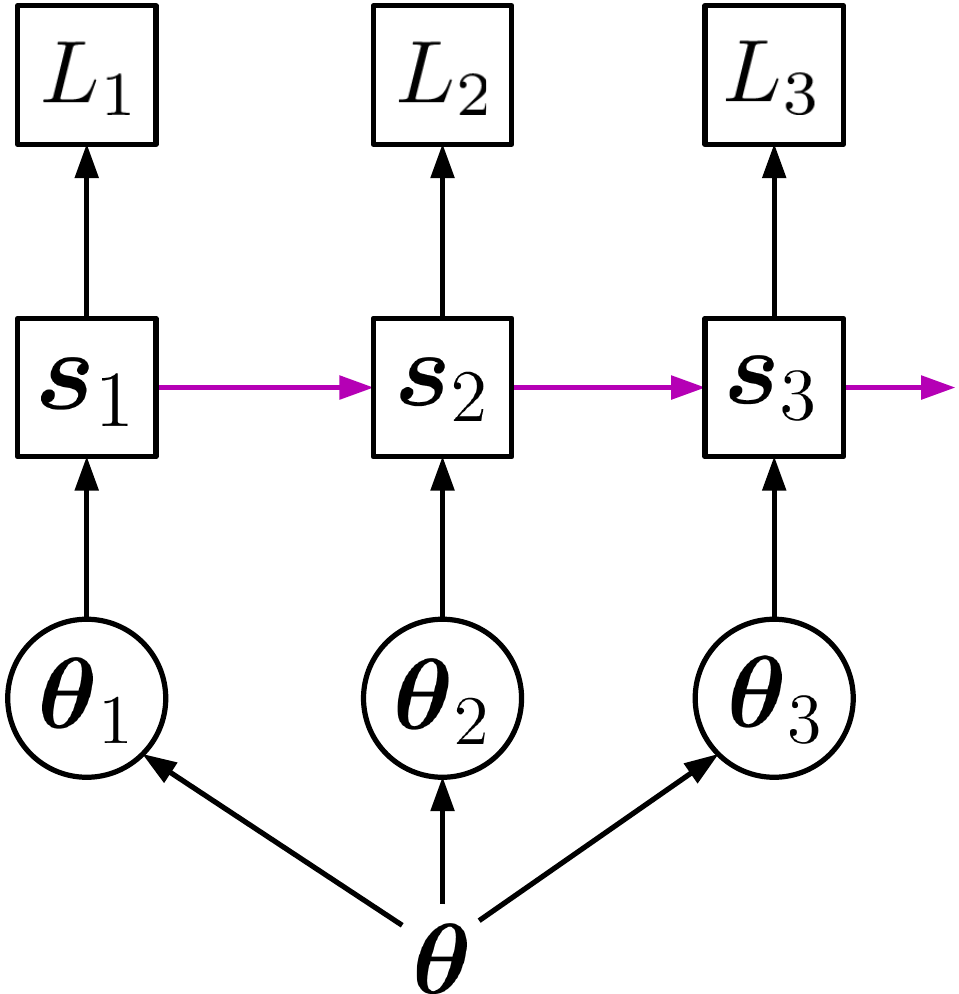}
    \vspace{-0.2cm}
    \caption{Stochastic computation graphs for ES-Single and PES, using the notation from~\citet{schulman2015gradient}. Vanilla truncated ES can be seen as removing the {\color{purple}recurrent connections} from PES.}
    \label{fig:stochastic-computation-graph}
    \vspace{-0.5cm}
\end{figure}

Figure~\ref{fig:stochastic-computation-graph} compares the stochastic computation graphs for ES-Single to those of PES and ES.
In these diagrams, squares represent deterministic nodes, which are functions of their parents; circles represent stochastic nodes (e.g., $\tilde{\boldtheta}$) which are distributed conditionally on their parents; and nodes not in squares or circles (e.g., $\boldtheta$) represent inputs.
ES-Single samples a single perturbed outer parameter $\tilde{\boldtheta} \sim \mathcal{N}(\boldtheta, \sigma^2 \boldI)$ that influences all the states $\bolds_t$ in the unroll, such that all losses $L_t$ are downstream of the node $\tilde{\boldtheta}$.
In contrast, PES perturbs the outer parameters independently for each partial unroll, $\boldtheta_t \sim \mathcal{N}(\boldtheta, \sigma^2 \boldI)$; in this case, the losses downstream of node $\boldtheta_t$ are $\{ L_\tau \}_{\tau=t}^T$.
In Appendix~\ref{app:stochasticcomp}, we provide derivations of each gradient estimator, leveraging Theorem 1 from~\citet{schulman2015gradient}, which gives a generic formula for an unbiased estimator of such graphs.
\vspace{-0.4cm}
\paragraph{Generalization of ES-Single and PES.}
One can also consider a generalization of both ES-Single and PES, that decouples the interval at which we update the perturbation accumulator from the interval at which we update the outer parameters.
In particular, one can introduce another hyperparameter, $\Omega$, that specifies the meta-update interval, while $K$ denotes the interval at which new perturbations are sampled and at which the perturbation accumulator is updated.
Many algorithms of interest can be obtained as special cases, by setting $K$ and $\Omega$ appropriately. Let $T$ be the length of a full inner problem.
Then, 1) if $K = \Omega = T$, we recover full-unroll ES; 2) if $K = \Omega$ and $K < T$, we recover PES; 3) if $K=T$ and $\Omega < T$, we recover ES-Single; and 4) if $K, \Omega < T$ and $\Omega < K$, we obtain a new estimator whose properties lie between the others.
The stochastic computation graph for this generalization, and the derivation of the resulting unbiased estimator, are provided in Appendix~\ref{app:generalization-es-pes}.

\begin{figure*}[t]
    \centering
    \begin{subfigure}[t]{0.36\linewidth}
        \includegraphics[width=\linewidth]{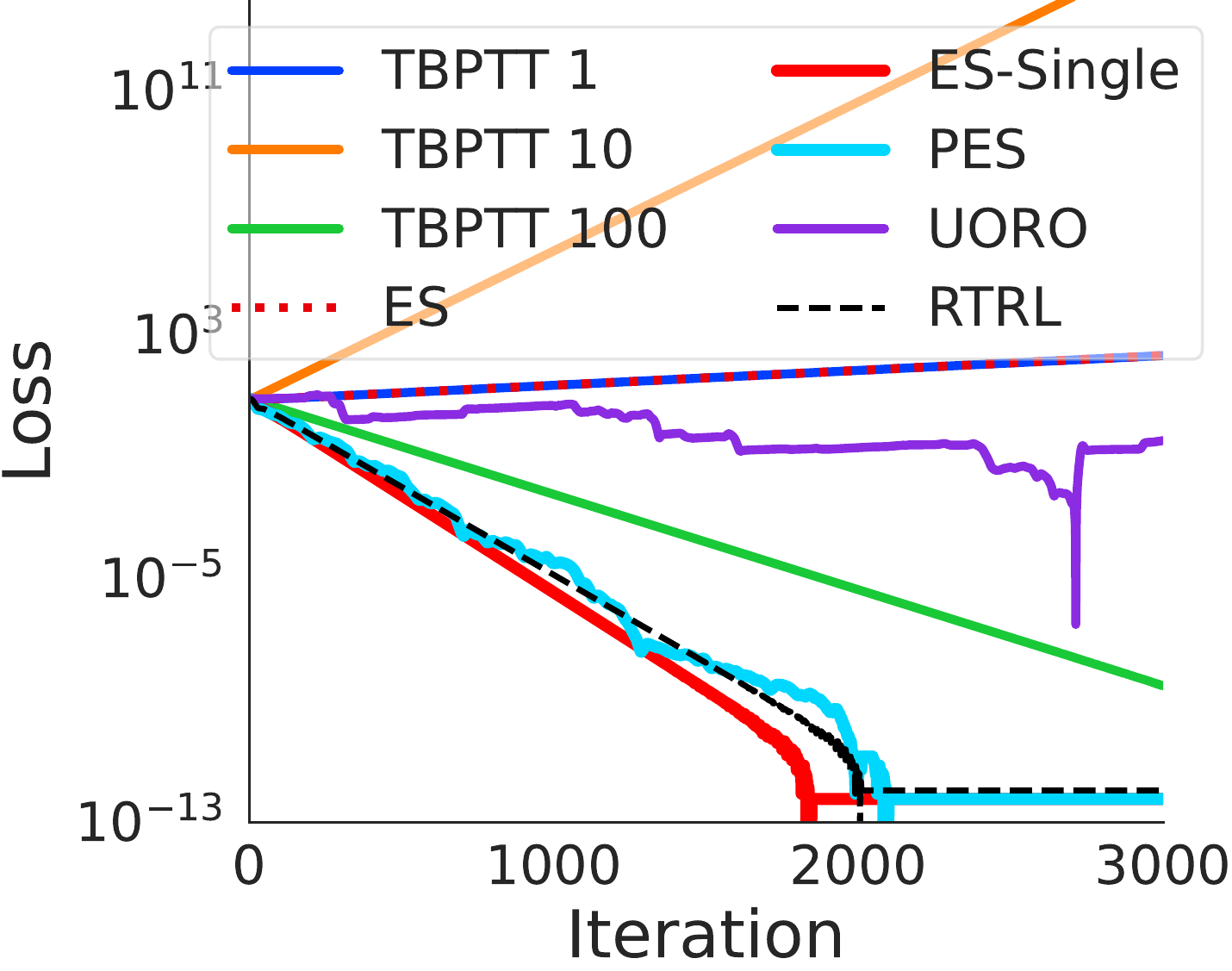}
        \caption{Comparing ES-Single to TBPTT, vanilla truncated ES, PES, UORO, and RTRL.}
        \label{fig:es-single-influence-comparison}
    \end{subfigure}
    \qquad
    \begin{subfigure}[t]{0.36\linewidth}
        \includegraphics[width=\linewidth]{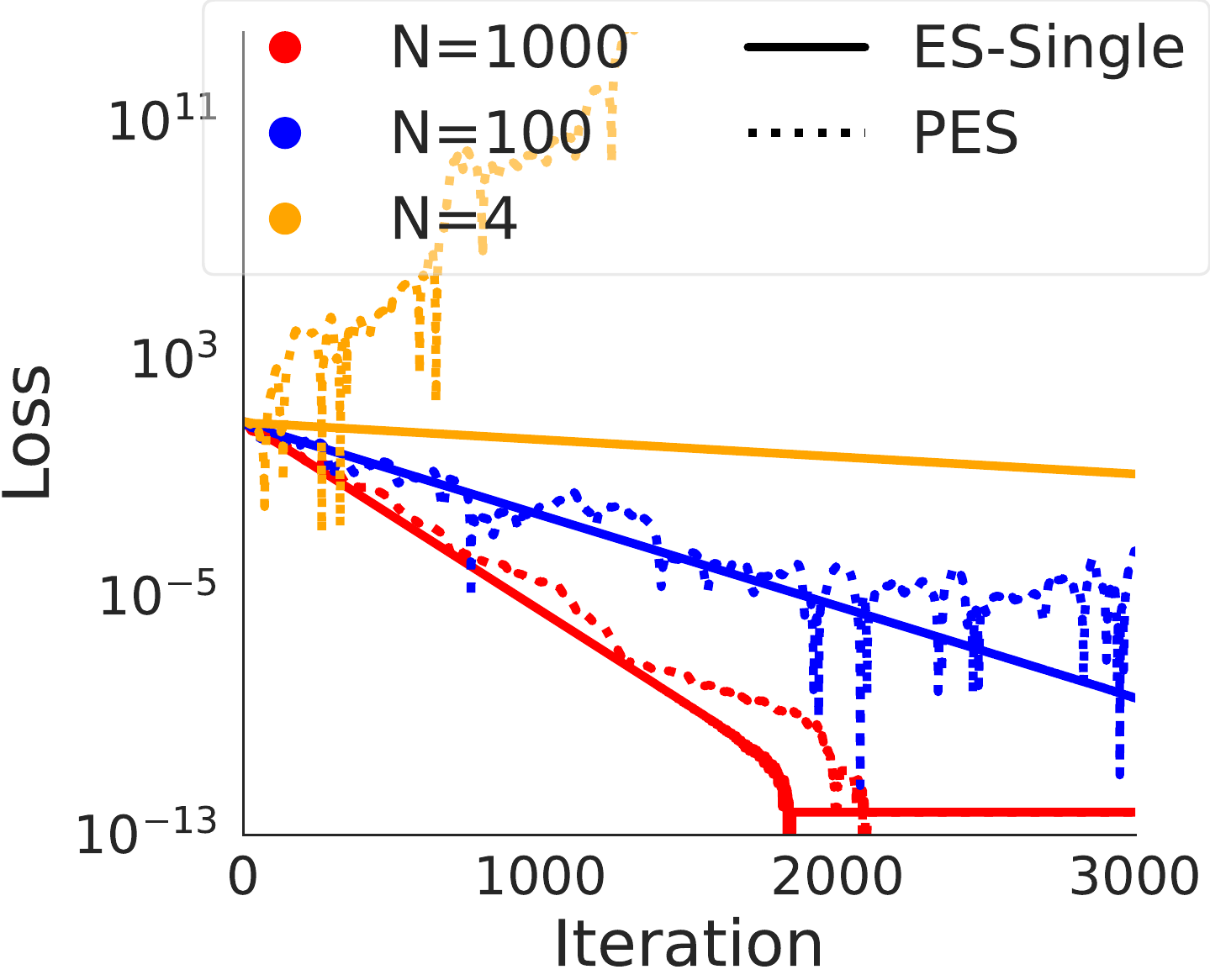}
        \caption{Ablation over the number of particles $N$ for PES and ES-Single.}
        \label{fig:es-single-influence-ablation}
    \end{subfigure}
    \caption{Evaluating ES-Single on the synthetic influence balancing task from ~\citet{tallec2017unbiased}.
    Note that TBPTT with truncation lengths 10 and 100 moves in the wrong direction, and truncated ES exactly matches the behavior of TBPTT with $K=1$.
    }
    \label{fig:es-single-influence}
    \vspace{-0.4cm}
\end{figure*}

\vspace{-0.1cm}
\subsection{Variance}
\label{sec:variance}
\vspace{-0.1cm}
In contrast to PES, the variance of the ES-Single estimator does not depend on the number of partial unrolls per inner problem.
We measured the empirical variance on the same task used by~\citet{vicol2021unbiased}:
we consider a tiny LSTM trained on the character-level Penn TreeBank dataset~\cite{marcus1993building}.
The full inner problem consists of a sequence of length $T=1000$, which we split into truncated unrolls of lengths $K \in \{ 1, 2, 5, 10, 20, 50, 100, 200, 500, 1000 \}$.
When measuring the variance of each estimator, we keep the parameters $\boldtheta$ fixed---that is, we do not update $\boldtheta$ after each partial unroll---and
accumulate the gradient for the full problem by summing the estimates over the truncated unrolls. This allows us to avoid any hysteresis effects.
We considered three different scenarios: 1) a sequence consisting of random characters, such that the gradients at each truncated unroll are i.i.d.; 2) a sequence consisting of a single repeated character, such that the gradients from each unroll are identical; and 3) a sequence of real data from the PTB dataset.
The results are shown in Figure~\ref{fig:variance}: we see that ES-Single has similar variance for all three scenarios, and in each case the variance is constant with respect to the number of unrolls, in contrast to PES.
Thus, ES-Single has substantially lower variance, especially when the inner problem is split into many unrolls; however, PES does have slightly lower variance for intermediate numbers of unrolls (e.g., 10-100 unrolls per inner problem).
Formally, ES-Single has the same variance characteristics as full-unroll ES.
\vspace{-0.2cm}
\begin{tcolorbox}[colback=tabblue!15,boxrule=0pt,colframe=white,coltext=black,arc=2pt,outer arc=0pt,valign=center]%,valign=center]
\begin{proposition}[ES-Single Variance]
The total variance of ES-Single using antithetic sampling is $\tr(\Var(\hat{\boldg}^{\text{ES-Single}})) = (P + 1) \| \nabla_{\boldtheta} L(\boldtheta) \|^2$, where $P$ is the dimensionality of $\boldtheta$.
\end{proposition}
\begin{proof}
The proof is provided in Appendix~\ref{app:variance-proof}.
\end{proof}
\end{tcolorbox}
\vspace{-0.2cm}
The total variance normalized by the squared gradient norm is $O(P)$, growing linearly in the number of outer parameters.
In contrast, the variance of PES includes terms in $T$, the number of unrolls per inner problem~\cite{vicol2021unbiased}; depending on the correlation between gradients at each partial unroll, its variance either decreases slightly with increasing $T$, or increases linearly with $T$.
In the realistic scenario using a true sequence from the PTB dataset, the variance of PES initially decreases slightly as the number of inner unrolls increases, after which it increases linearly.

\vspace{-0.1cm}
\subsection{Hysteresis}
\label{sec:hysteresis}
\vspace{-0.1cm}
Any method that makes updates to the outer parameters online during optimization of an inner problem will suffer from \textit{hysteresis}, including RTRL and its approximations (UORO, KF-RTRL, OK), PES, and ES-Single.
The impact of hysteresis on final performance is problem-dependent; one approach to help mitigate the effects of hysteresis is to use breakstep (as opposed to lockstep) training, described in Appendix~\ref{app:lockstep-breakstep}.

\vspace{-0.1cm}
\section{Experiments}
\label{sec:experiments}
\vspace{-0.1cm}
We evaluated ES-Single on several tasks from~\citet{vicol2021unbiased}, which include both toy problems and real-world tasks.
First, we show empirically that ES-Single is unbiased, via an influence balancing task that is designed such that truncated methods fail; then, we use ES-Single to optimize hyperparameters, to tune several mixed continuous and discrete hyperparameters for a FashionMNIST training task.
Finally, we consider two high-dimensional problems: 1) training an LSTM to copy sequences of increasing length (on which truncated methods fail); and 2) meta-training a learned optimizer.
Both of these tasks have thousands of outer parameters, which allows us to evaluate the scalability of ES-Single to real-world settings.
Overall, we show that ES-Single is unbiased, and consistently outperforms PES on all tasks, achieving lower meta-loss values in fewer meta-optimization steps.
Experimental details and additional results are provided in Appendix~\ref{app:exp-details}.

\vspace{-0.1cm}
\subsection{Synthetic Influence Balancing Task}
\vspace{-0.1cm}
\begin{figure}[h]
    \centering
    \includegraphics[width=0.95\linewidth]{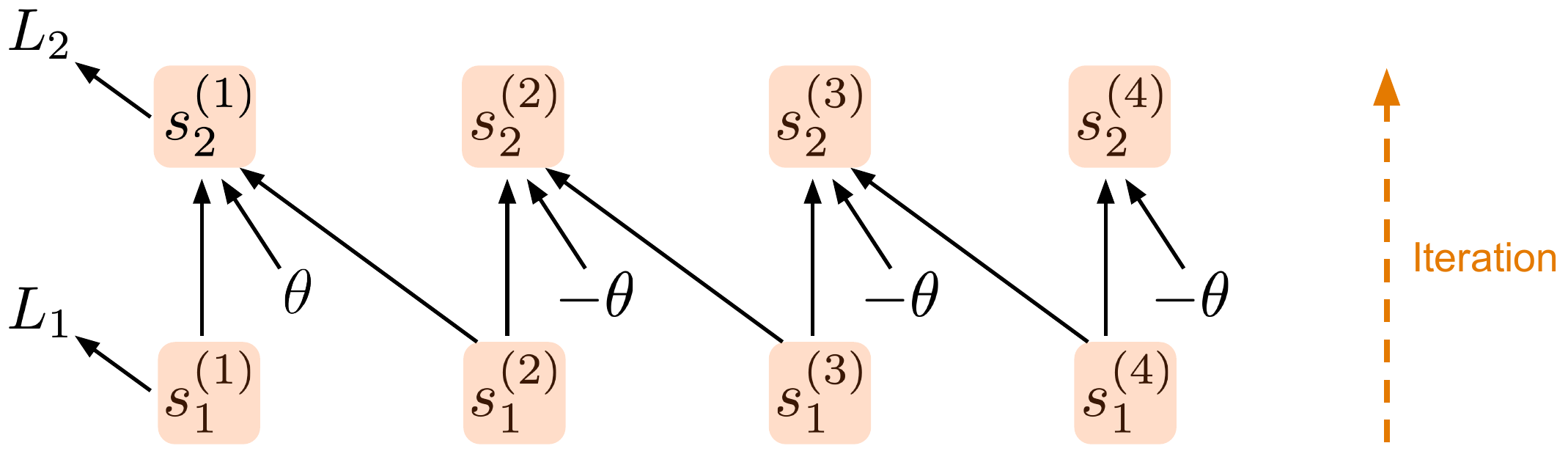}
    \vspace{-0.2cm}
    \caption{Illustration of the influence balancing task.}
    \label{fig:influence}
    \vspace{-0.2cm}
\end{figure}
First, we revisit the influence balancing task, originally introduced by ~\citet{tallec2017unbiased} and used in PES~\cite{vicol2021unbiased}.
This is a synthetic task with a scalar parameter $\theta \in \mathbb{R}$, designed such that $\theta$ has a negative influence in the short term but a positive influence in the long term.
We consider a linear dynamical system:
\begin{align}
    \bolds_{t+1} = \boldA \bolds_t + (\underbrace{\theta, \dots, \theta}_{p\text{ positive}}, \underbrace{-\theta, \dots, -\theta}_{n-p\text{ negative}})^\top
\end{align}
where $\boldA$ is an $n \times n$ matrix with $\boldA_{i,i} = 0.5$, $\boldA_{i,i+1}=0.5$, and 0 everywhere else.
The loss $L_t$ computes the squared error on the first index in the state vector $\bolds_t$; see Appendix~\ref{app:exp-details} for details.
This task is shown diagrammatically in Figure~\ref{fig:influence}.
As shown in Figure~\ref{fig:es-single-influence-comparison}, ES-Single outperforms PES and RTRL, and yields a smoother loss curve.
Note that for this task, the inner problem is infinite; thus, the perturbation accumulator for PES is never reset.
Because the variance of PES increases with the number of unrolls, PES requires a large number of particles ($N=1000$) to perform well.
If the number of particles is decreased, optimization becomes unstable, as shown in Figure~\ref{fig:es-single-influence-ablation}.
In contrast, because the variance of ES-Single does not depend on the number of unrolls, it can perform well on this task with substantially fewer particles, even $N=4$.

\subsection{Hyperparameter Optimization}
\label{sec:hyperparameter-optimization}

\paragraph{MNIST LR Schedule.}
Here, we used ES-Single to meta-learn a learning rate (LR) schedule used to train an MLP on MNIST.
Based on ~\cite{wu2018understanding}, we used a two-hidden-layer MLP with 100 units per layer, and tuned LR schedule parameterized by $\alpha_t = \frac{\theta_0}{\left( 1 + \frac{t}{Q} \right)^{\theta_1}}$, where $\theta_0$ is the initial LR, $\theta_1$ is the LR decay factor, and $Q=5000$ is a constant.
The results are shown in Figure~\ref{fig:mnist-mlp}.
We found that ES-Single performed similarly to PES, but had more stable convergence near the optimum, while PES at times drifted away from the optimum due to its high variance.
\begin{figure}[H]
    \vspace{-0.2cm}
    \centering
    \includegraphics[width=0.8\linewidth]{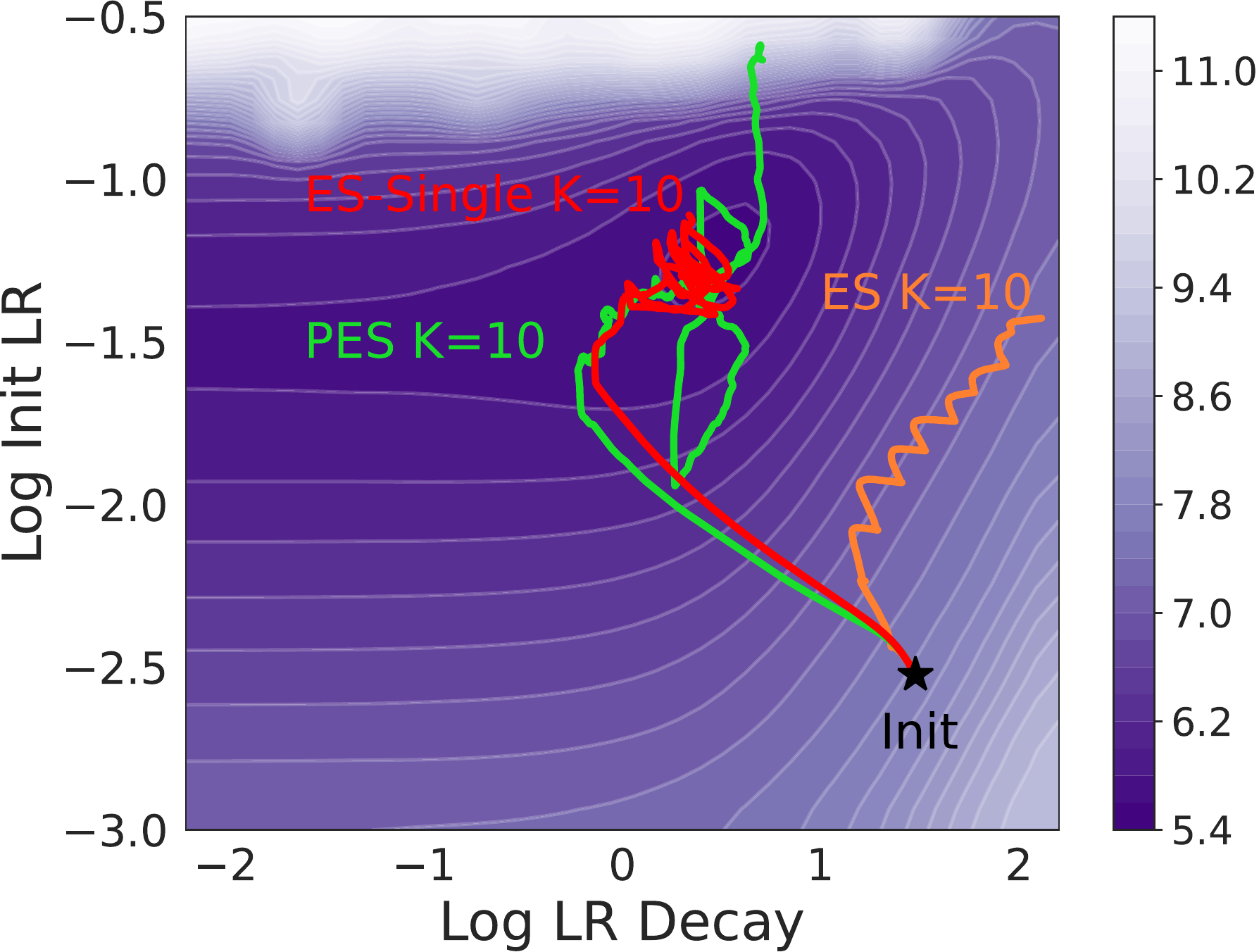}
    \vspace{-0.3cm}
    \caption{Meta-optimization of a learning rate schedule for an MNIST MLP, using truncation length $K=10$. Darker regions are better. Further experiments are provided in Appendix~\ref{app:exp-details}.}
    \label{fig:mnist-mlp}
    \vspace{-0.4cm}
\end{figure}

\paragraph{Telescoping Sums.}
If the desired meta-objective is the \textit{final} loss $L_T$ rather than the sum of losses $\sum_{t=0}^T$, then this can be handled gracefully in our framework by defining $p_t = L_t - L_{t-1}$, where we define $L_{-1} \equiv 0$ for notational simplicity.
Then, we can consider the sum of $p_t$, which yields a telescoping sum:
\begin{align}
    \sum_{t=0}^T p_t = (\cancel{L_0} - L_{-1})
    + \cdots +
    (L_T - \cancel{L_{T-1}}) = L_T
\end{align}
\begin{figure*}[t]
    \centering
    \begin{subfigure}[t]{0.38\linewidth}
        \includegraphics[width=\linewidth]{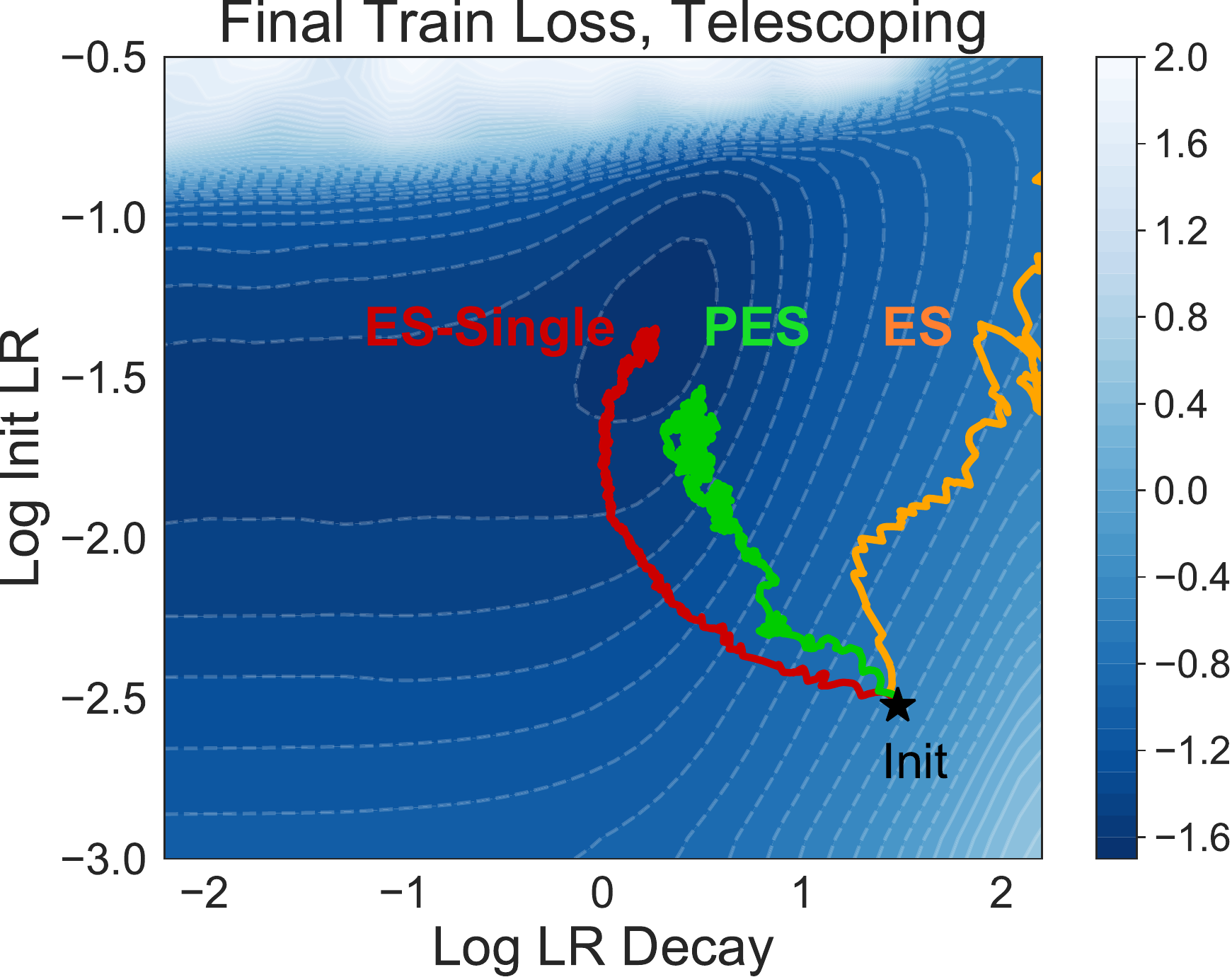}
        \caption{Meta-optimization  using telescoping sums.}
        \label{fig:telescoping-meta-trajectories}
    \end{subfigure}
    \qquad
    \begin{subfigure}[t]{0.38\linewidth}
        \includegraphics[width=\linewidth]{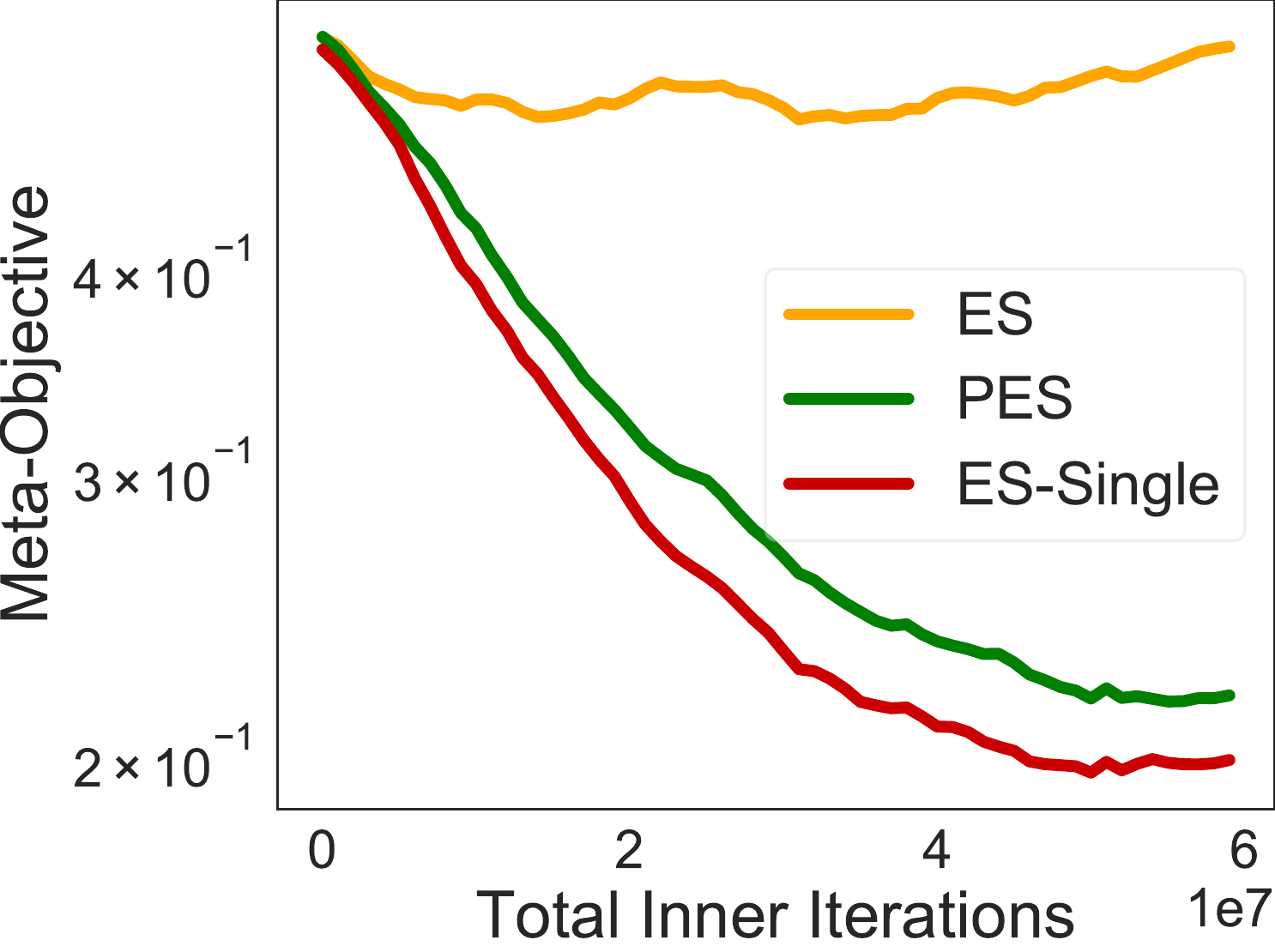}
        \caption{Meta-losses obtained by ES, PES, and ES-Single.}
        \label{fig:telescoping-meta-loss}
    \end{subfigure}
    \vspace{-0.2cm}
    \caption{Meta-optimizing a learning rate schedule for an MLP on FashionMNIST, using a telescoping sum to target the final training loss.}
    \label{fig:telescoping}
    \vspace{-0.2cm}
\end{figure*}
Figure~\ref{fig:telescoping} compares vanilla truncated ES, PES, and ES-Single on a task that tunes the learning rate and decay factor for training an MLP on FashionMNIST, targeting the final training loss.
We see that the meta-optimization trajectory of ES-Single was significantly smoother than that of PES, more closely followed the meta-loss contours, and had better stability near the optimum (Figure~\ref{fig:telescoping-meta-trajectories}).
As shown in Figure~\ref{fig:telescoping-meta-loss}, ES-Single converged more rapidly to the optimal meta-objective value than PES.

\paragraph{Tuning Many Hyperparameters.}
Here, we applied ES-Single to tune many hyperparameters simultaneously, to train a 5-hidden-layer MLP on FashionMNIST.
The meta-objective is the sum of validation losses over the inner problem.
We tuned 29 hyperparameters, including separate learning rates and momentum coefficients per parameter block (e.g., for each weight matrix and bias vector in the MLP), and the number of hidden units per layer (which is a discrete hyperparameter that takes values in the range 10-100).
We compared ES-Single to random search, vanilla ES, and PES.
The results are shown in Figure~\ref{fig:tuning-many-hparams}: we found that ES-Single substantially outperformed these baselines, and achieved lower meta-loss in fewer iterations compared to PES.
\begin{figure}[H]
    \centering
    \includegraphics[width=0.85\linewidth]{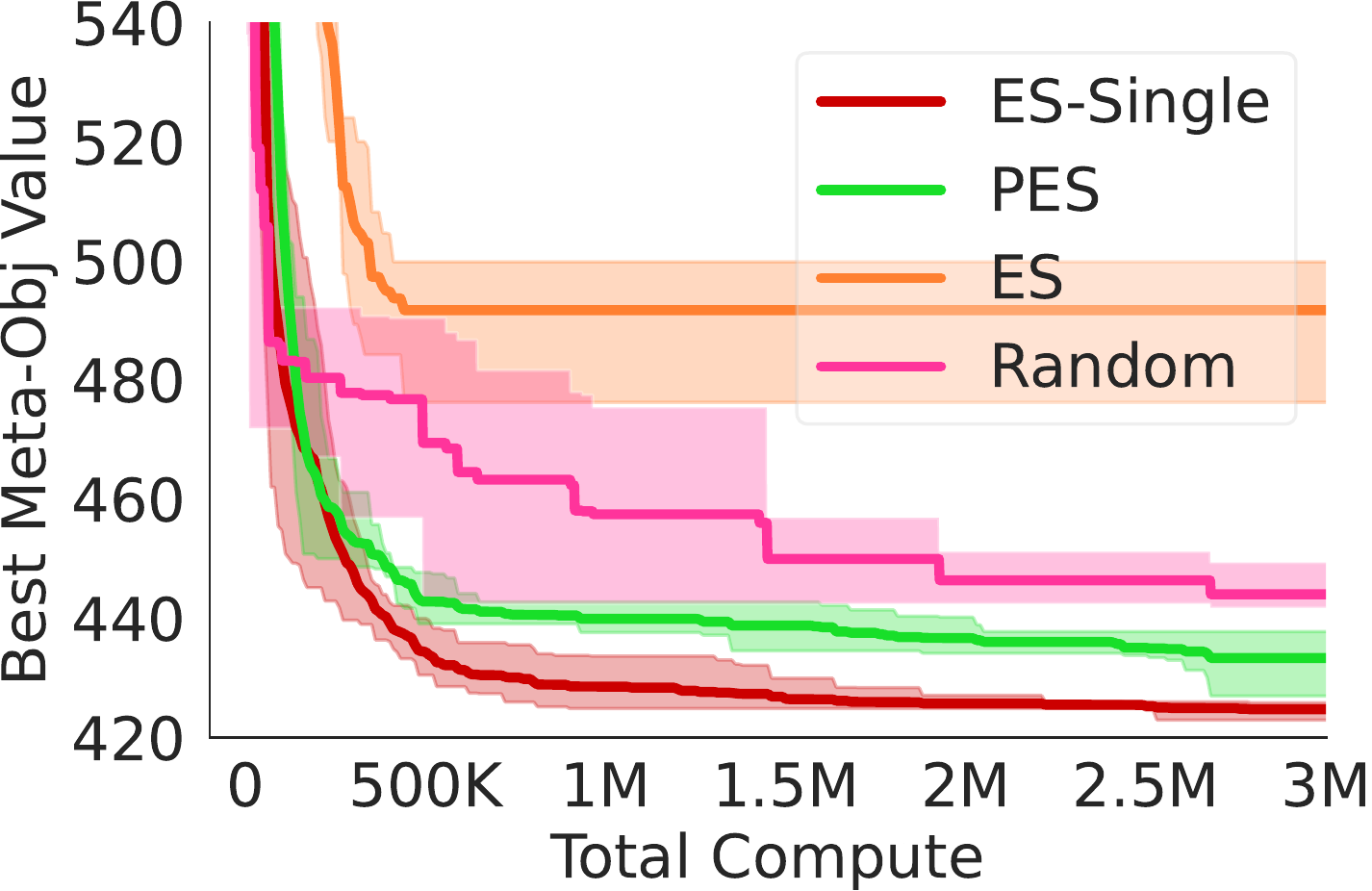}
    \vspace{-0.2cm}
    \caption{Tuning many hyperparameters for an MLP on FashionMNIST, targeting the sum of validation losses as the meta-objective.}
    \label{fig:tuning-many-hparams}
    \vspace{-0.3cm}
\end{figure}

\subsection{LSTM Copy Task}
Next, we used ES-Single to train an LSTM on the copy task introduced by~\citet{mujika2018approximating}, where the model must read a binary string of length $T$, and output the same string.
The challenge lies in learning long-term dependencies as $T$ increases.
Following~\citet{mujika2018approximating}, we use a curriculum starting with $T=1$, and increasing $T$ by 1 each time the exponential moving average of the cross-entropy loss (e.g., bits-per-character) drops below the threshold 0.15.
To ensure that the model does not overfit to a particular sequence length, we sample $T$ uniformly from $\{ T-5, \dots, T \}$ (or $T=1$ if the sampled value is negative).
We train a 1-layer LSTM with hidden state size 100, that has 42804 parameters, which we learn via ES-based methods, evaluating scalability.
PES and ES-Single were run using truncations of length $K=1$ for fully-online learning; for vanilla truncated ES, we used truncation lengths $K \in \{ 25, 50 \}$.
In Figure~\ref{fig:lstm-copy-task}, we show the maximum length $T$ that is successfully copied over the course of training using ES, PES, and ES-Single.
As expected, ES plateaus, as it intrinsically cannot model dependencies across longer horizons than its truncation length.
Both PES and ES-Single outperform ES, but ES-Single substantially outperforms PES, with $T$ increasing faster and reaching higher maximum values.
\begin{figure}[H]
    \vspace{-0.2cm}
    \centering
    \includegraphics[width=0.8\linewidth]{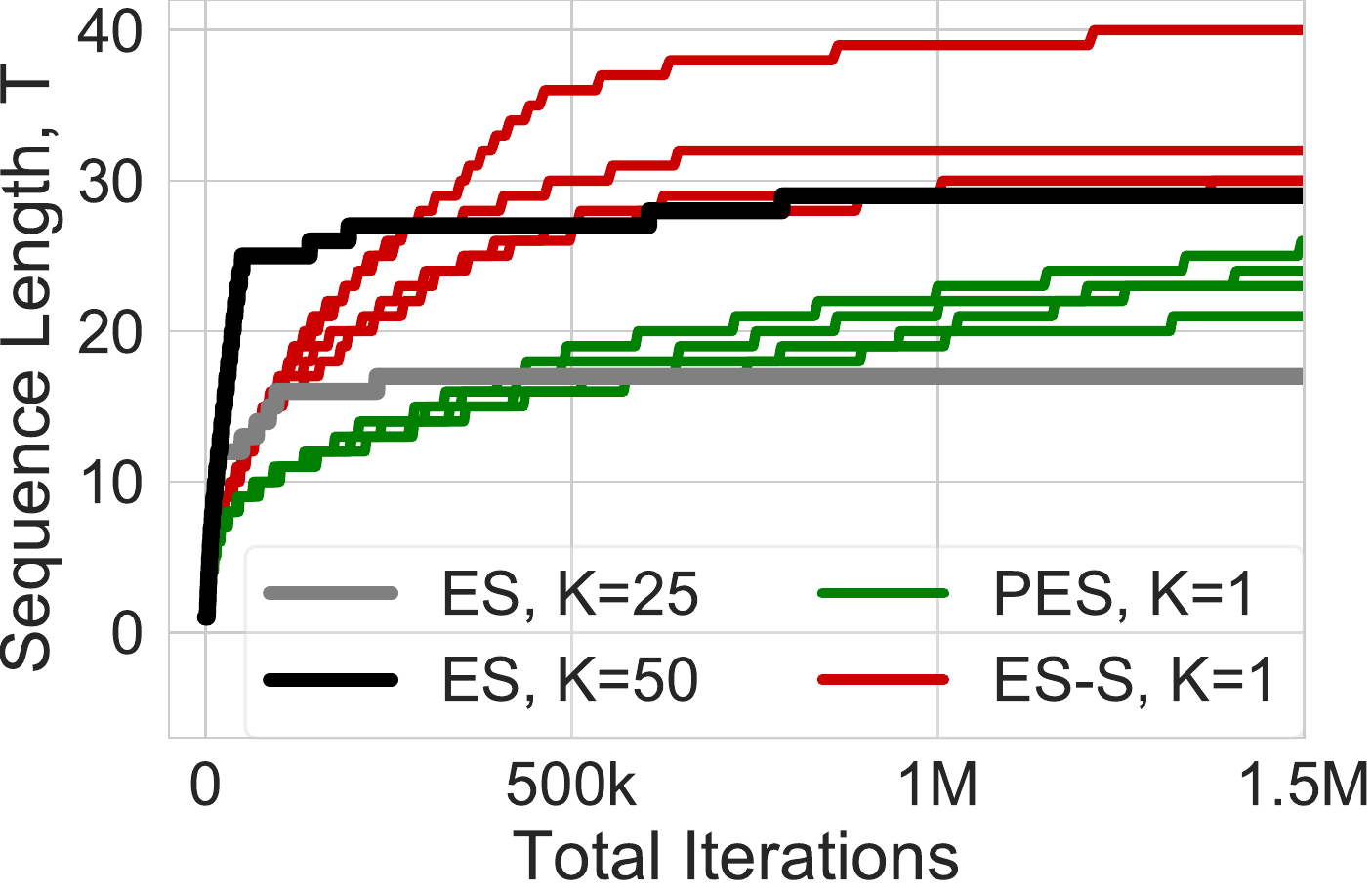}
    \vspace{-0.2cm}
    \caption{Maximum sequence length $T$ that is successfully copied in the copy task from~\citet{mujika2018approximating}.
    Curves of the same color use different random seeds.
    For the ES baselines (gray and black curves), we show only the best result to reduce clutter.
    Here, the x-axis represents the number of tokens ingested by each approach (e.g., data-time rather than compute).
    }
    \label{fig:lstm-copy-task}
    \vspace{-0.2cm}
\end{figure}

\vspace{-0.3cm}
\subsection{Learned Optimizer Training}
\label{sec:lopt}
Here, we used ES-Single to meta-optimize a learned optimizer using the LOLv2 architecture introduced by~\citet{metz2018meta}.
This optimizer is meta-trained to optimize a 2-hidden-layer MLP with 128 hidden units per layer, on FashionMNIST for $T=5000$ steps, using truncated unrolls of length $K=10$.
As the meta-objective, we targeted the mean training loss over the inner optimization trajectory.
In Figure~\ref{fig:learned-optimizer}, we show the meta-objective values obtained over the course of meta-training, using vanilla truncated ES, PES, and ES-Single.
ES fails due to truncation bias, while PES performs poorly due to high variance; in contrast, ES-Single performs much better on this long-horizon task with short truncations.
\begin{figure}[H]
    \centering
    \includegraphics[width=0.8\linewidth]{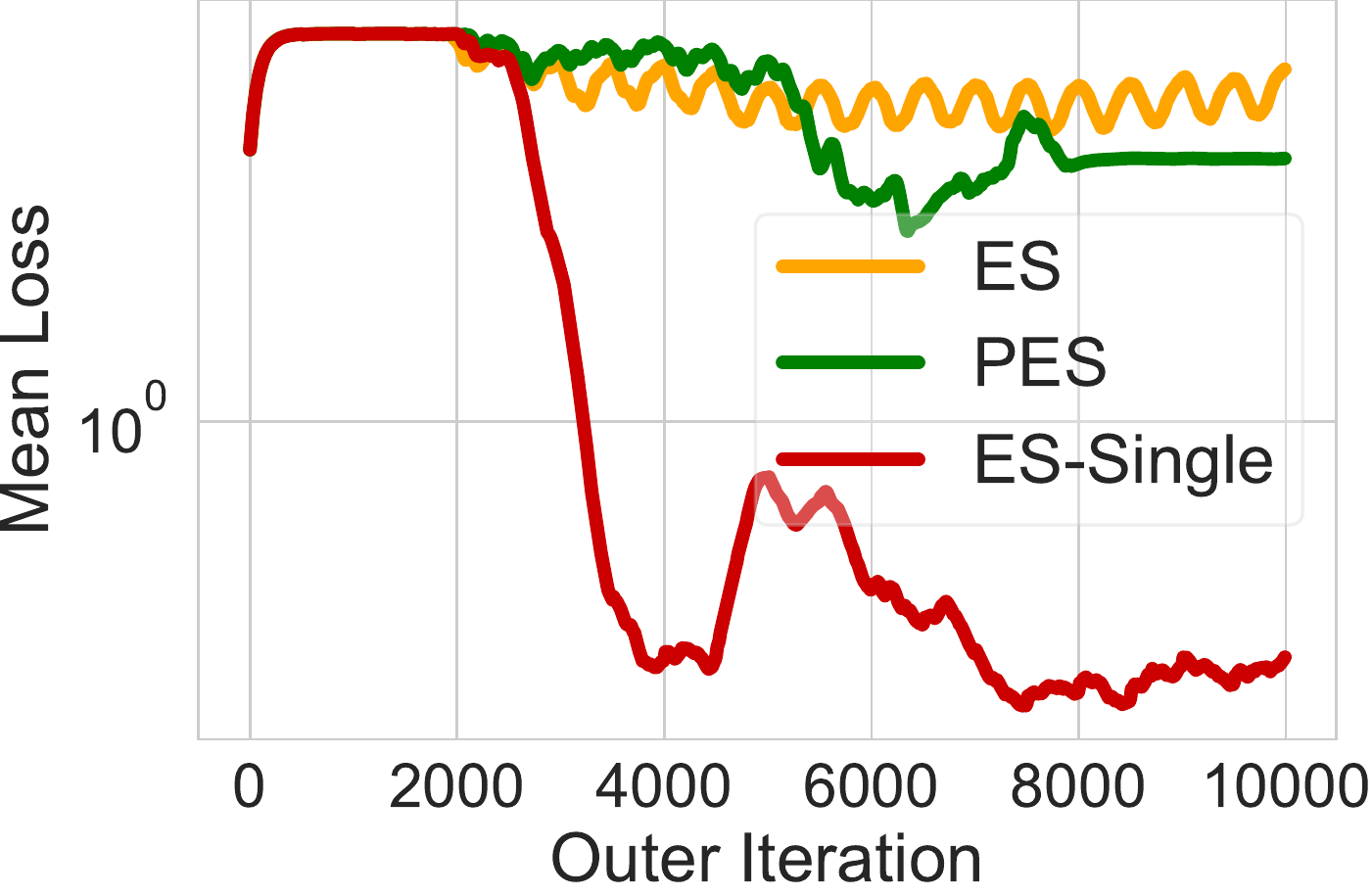}
    \vspace{-0.2cm}
    \caption{Meta-training a learned optimizer, targeting a two-layer MLP on FashionMNIST.}
    \label{fig:learned-optimizer}
    \vspace{-0.3cm}
\end{figure}

\section{Related Work}
\label{sec:related-work}
\vspace{-0.2cm}
We provide extended related work in Appendix~\ref{app:extended-related-work}.
\vspace{-0.3cm}
\paragraph{Gradient-Based Approaches.}
There are two families of gradient-based methods for computing gradients through unrolled computation, based on reverse-mode (e.g., backpropagation through time, BPTT) or forward-mode gradient accumulation (e.g., real-time recurrent learning, RTRL).
Backpropagating through full unrolled sequences is expensive, with compute and memory cost that scales linearly in the unroll length.
Gradient checkpointing~\cite{chen2016training} reduces the memory requirement to $O(\sqrt{T})$, at the cost of additional compute.
Truncated BPTT (TBPTT) operates on shorter sub-sequences of length $K \ll T$, substantially reducing cost, but introducing truncation bias that can lead to sub-optimal solutions~\cite{wu2018understanding}.
ART-BP~\cite{tallec2017unbiasing} uses randomly sampled truncation lengths, and introduces reweighting factors during backprop based on the sequence length to yield an unbiased gradient estimate of the total loss.

RTRL performs forward-mode gradient accumulation, by maintaining the recurrent Jacobian via the following update: $\frac{d \bolds_t}{d \boldtheta} = \frac{\partial \bolds_t}{\partial \bolds_{t-1}} \frac{d \bolds_{t-1}}{d \boldtheta} + \frac{\partial \bolds_t}{\partial \boldtheta}$.
RTRL allows for fully online learning of the outer parameters (e.g., with outer updates taken every $K=1$ steps), but is intractable for high-dimensional problems, as the recurrent Jacobian $\frac{d \bolds_t}{d \boldtheta}$ is $P \times P$ and thus too large to store in memory.
Several cheaper approximations to RTRL have been proposed, including: Unbiased Online Recurrent Optimization (UORO)~\cite{tallec2017unbiased}, maintains a rank-1 estimate of the recurrent Jacobian; KF -RTRL~\cite{mujika2018approximating} proposes a Kronecker factorization of the Jacobian, and the Optimal Kronecker Sum Approximation (OK)~\cite{benzing2019optimal} provides a lower-variance extension of KF-RTRL.
Unfortunately, these methods cannot optimize over chaotic loss landscapes, and are either high-variance, difficult to implement, or are only applicable to a restricted class of models (e.g., specific RNN architectures).
\citet{silver2021learning} propose a method for unbiased gradient estimation based on directional derivatives; being a gradient-based approach, this method requires a differentiable objective function.
\vspace{-0.3cm}
\paragraph{Chaos.}
Unrolled dynamical systems can lead to chaotic loss landscapes, for example in rigid-body physics, graphics, model-based control~\cite{parmas2018pipps}, fluid simulation~\cite{ni2017sensitivity,kochkov2021machine}, climate modeling~\cite{lea2000sensitivity,kohl2002adjoint}, and simulation of weather~\cite{bischof1996sensitivity} or nuclear fusion~\cite{mcgreivy2021optimized}. \citet{metz2021gradients} discuss this in depth, showing that while analytic gradients may be available in such systems, they are not necessarily useful due to high variance. In particular, the reparameterization gradient estimator~\cite{kingma2013auto} may have orders of magnitude larger variance than black-box ES estimates~\cite{parmas2018pipps,parmas2019unified,metz2019understanding,schwefel1977evolutionsstrategien,wierstra2014natural} or variational optimization~\cite{staines2012variational}.
\citet{metz2021gradients} provide an overview of scenarios in which chaos arises, and a taxonomy of approaches to either prevent chaos from arising (e.g., switching to a better-behaved system) or to optimize in the presence of chaos (e.g., using smoothing-based approaches, as we do here).
The high-level outline for ES-Single was first proposed in~\citet{vicol2023bilevel}.

\vspace{-0.3cm}
\section{Conclusion}
\label{sec:conclusion}
\vspace{-0.2cm}
We introduced an unbiased gradient estimator for unrolled computation graphs, called ES-Single. ES-Single inserts breakpoints into the computation graph for a full unroll, at which intermediate results are aggregated and used to form an ES-based gradient estimate, which is applied to update the outer parameters.
Crucially, compared to vanilla truncated ES and PES, ES-Single samples outer parameter perturbations once at the start of each inner problem, and re-applies the same perturbations in each partial unroll.
ES-Single is simpler to implement than PES, and has constant variance with respect to the number of partial unrolls per inner problem; this leads to substantially lower variance than PES in practice, and makes ES-Single well-suited for long-horizon tasks with short truncations.
We evaluated ES-Single on a diverse set of tasks, including a synthetic task to test for unbiasedness, hyperparameter optimization, RNN training, and training of learned optimizers.
On all tasks, it outperformed ES and PES.

\clearpage

\bibliography{references}

\begin{thebibliography}{75}
\providecommand{\natexlab}[1]{#1}
\providecommand{\url}[1]{\texttt{#1}}
\expandafter\ifx\csname urlstyle\endcsname\relax
  \providecommand{\doi}[1]{doi: #1}\else
  \providecommand{\doi}{doi: \begingroup \urlstyle{rm}\Url}\fi

\bibitem[Andrychowicz et~al.(2016)Andrychowicz, Denil, Gomez, Hoffman, Pfau,
  Schaul, Shillingford, and De~Freitas]{andrychowicz2016learning}
Andrychowicz, M., Denil, M., Gomez, S., Hoffman, M.~W., Pfau, D., Schaul, T.,
  Shillingford, B., and De~Freitas, N.
\newblock Learning to learn by gradient descent by gradient descent.
\newblock In \emph{Advances in Neural Information Processing Systems
  (NeurIPS)}, pp.\  3981--3989, 2016.

\bibitem[Asuncion \& Newman(2007)Asuncion and Newman]{asuncion2007uci}
Asuncion, A. and Newman, D.
\newblock {UCI Machine Learning Repository}, 2007.

\bibitem[Baydin et~al.(2017)Baydin, Cornish, Rubio, Schmidt, and
  Wood]{baydin2017online}
Baydin, A.~G., Cornish, R., Rubio, D.~M., Schmidt, M., and Wood, F.
\newblock Online learning rate adaptation with hypergradient descent.
\newblock \emph{arXiv preprint arXiv:1703.04782}, 2017.

\bibitem[Bengio(2000)]{bengio2000gradient}
Bengio, Y.
\newblock Gradient-based optimization of hyperparameters.
\newblock \emph{Neural Computation}, 12\penalty0 (8):\penalty0 1889--1900,
  2000.

\bibitem[Benzing et~al.(2019)Benzing, Gauy, Mujika, Martinsson, and
  Steger]{benzing2019optimal}
Benzing, F., Gauy, M.~M., Mujika, A., Martinsson, A., and Steger, A.
\newblock {Optimal Kronecker-sum approximation of real time recurrent
  learning}.
\newblock In \emph{International Conference on Machine Learning}, pp.\
  604--613. PMLR, 2019.

\bibitem[Bergstra et~al.(2011)Bergstra, Bardenet, Bengio, and
  K{\'e}gl]{bergstra2011algorithms}
Bergstra, J.~S., Bardenet, R., Bengio, Y., and K{\'e}gl, B.
\newblock Algorithms for hyper-parameter optimization.
\newblock In \emph{Advances in Neural Information Processing Systems
  (NeurIPS)}, pp.\  2546--2554, 2011.

\bibitem[Bertinetto et~al.(2018)Bertinetto, Henriques, Torr, and
  Vedaldi]{bertinetto2018meta}
Bertinetto, L., Henriques, J.~F., Torr, P.~H., and Vedaldi, A.
\newblock Meta-learning with differentiable closed-form solvers.
\newblock \emph{arXiv preprint arXiv:1805.08136}, 2018.

\bibitem[Bischof et~al.(1996)Bischof, Pusch, and
  Knoesel]{bischof1996sensitivity}
Bischof, C.~H., Pusch, G.~D., and Knoesel, R.
\newblock {Sensitivity analysis of the MM5 weather model using automatic
  differentiation}.
\newblock \emph{Computers in Physics}, 10\penalty0 (6):\penalty0 605--612,
  1996.

\bibitem[Blondel et~al.(2021)Blondel, Berthet, Cuturi, Frostig, Hoyer,
  Llinares-L{\'o}pez, Pedregosa, and Vert]{blondel2021efficient}
Blondel, M., Berthet, Q., Cuturi, M., Frostig, R., Hoyer, S.,
  Llinares-L{\'o}pez, F., Pedregosa, F., and Vert, J.-P.
\newblock Efficient and modular implicit differentiation.
\newblock \emph{arXiv preprint arXiv:2105.15183}, 2021.

\bibitem[Chandra et~al.(2022)Chandra, Xie, Ragan-Kelley, and
  Meijer]{chandra2022gradient}
Chandra, K., Xie, A., Ragan-Kelley, J., and Meijer, E.
\newblock Gradient descent: The ultimate optimizer.
\newblock \emph{Advances in Neural Information Processing Systems},
  35:\penalty0 8214--8225, 2022.

\bibitem[Chen et~al.(2016)Chen, Xu, Zhang, and Guestrin]{chen2016training}
Chen, T., Xu, B., Zhang, C., and Guestrin, C.
\newblock Training deep nets with sublinear memory cost.
\newblock \emph{arXiv preprint arXiv:1604.06174}, 2016.

\bibitem[Domke(2012)]{domke2012generic}
Domke, J.
\newblock Generic methods for optimization-based modeling.
\newblock In \emph{Proceedings of Machine Learning Research}, pp.\  318--326,
  2012.

\bibitem[Finn et~al.(2018)Finn, Xu, and Levine]{finn2018probabilistic}
Finn, C., Xu, K., and Levine, S.
\newblock Probabilistic model-agnostic meta-learning.
\newblock \emph{Advances in Neural Information Processing Systems (NeurIPS)},
  31, 2018.

\bibitem[Finn(2018)]{finn2018learning}
Finn, C.~B.
\newblock \emph{Learning to learn with gradients}.
\newblock University of California, Berkeley, 2018.

\bibitem[Foo et~al.(2008)Foo, Do, and Ng]{foo2008efficient}
Foo, C.-S., Do, C.~B., and Ng, A.~Y.
\newblock Efficient multiple hyperparameter learning for log-linear models.
\newblock In \emph{Advances in Neural Information Processing Systems
  (NeurIPS)}, pp.\  377--384, 2008.

\bibitem[Franceschi et~al.(2017)Franceschi, Donini, Frasconi, and
  Pontil]{franceschi2017forward}
Franceschi, L., Donini, M., Frasconi, P., and Pontil, M.
\newblock Forward and reverse gradient-based hyperparameter optimization.
\newblock \emph{arXiv preprint arXiv:1703.01785}, 2017.

\bibitem[Jaderberg et~al.(2017)Jaderberg, Dalibard, Osindero, Czarnecki,
  Donahue, Razavi, Vinyals, Green, Dunning, Simonyan,
  et~al.]{jaderberg2017population}
Jaderberg, M., Dalibard, V., Osindero, S., Czarnecki, W.~M., Donahue, J.,
  Razavi, A., Vinyals, O., Green, T., Dunning, I., Simonyan, K., et~al.
\newblock Population-based training of neural networks.
\newblock \emph{arXiv preprint arXiv:1711.09846}, 2017.

\bibitem[Jamieson \& Talwalkar(2016)Jamieson and Talwalkar]{jamieson2016non}
Jamieson, K. and Talwalkar, A.
\newblock Non-stochastic best arm identification and hyperparameter
  optimization.
\newblock In \emph{International Conference on Artificial Intelligence and
  Statistics}, 2016.

\bibitem[Kingma \& Welling(2013)Kingma and Welling]{kingma2013auto}
Kingma, D.~P. and Welling, M.
\newblock {Auto-encoding variational Bayes}.
\newblock \emph{arXiv preprint arXiv:1312.6114}, 2013.

\bibitem[Kochkov et~al.(2021)Kochkov, Smith, Alieva, Wang, Brenner, and
  Hoyer]{kochkov2021machine}
Kochkov, D., Smith, J.~A., Alieva, A., Wang, Q., Brenner, M.~P., and Hoyer, S.
\newblock Machine learning--accelerated computational fluid dynamics.
\newblock \emph{Proceedings of the National Academy of Sciences}, 118\penalty0
  (21):\penalty0 e2101784118, 2021.

\bibitem[K{\"o}hl \& Willebrand(2002)K{\"o}hl and Willebrand]{kohl2002adjoint}
K{\"o}hl, A. and Willebrand, J.
\newblock An adjoint method for the assimilation of statistical characteristics
  into eddy-resolving ocean models.
\newblock \emph{Tellus A: Dynamic Meteorology and Oceanography}, 54\penalty0
  (4):\penalty0 406--425, 2002.

\bibitem[Larsen et~al.(1996)Larsen, Hansen, Svarer, and
  Ohlsson]{larsen1996design}
Larsen, J., Hansen, L.~K., Svarer, C., and Ohlsson, M.
\newblock {Design and regularization of neural networks: The optimal use of a
  validation set}.
\newblock In \emph{IEEE Signal Processing Society Workshop}, pp.\  62--71,
  1996.

\bibitem[Lea et~al.(2000)Lea, Allen, and Haine]{lea2000sensitivity}
Lea, D.~J., Allen, M.~R., and Haine, T.~W.
\newblock Sensitivity analysis of the climate of a chaotic system.
\newblock \emph{Tellus A: Dynamic Meteorology and Oceanography}, 52\penalty0
  (5):\penalty0 523--532, 2000.

\bibitem[Li \& Malik(2016)Li and Malik]{li2016learning}
Li, K. and Malik, J.
\newblock Learning to optimize.
\newblock \emph{arXiv preprint arXiv:1606.01885}, 2016.

\bibitem[Li \& Malik(2017)Li and Malik]{li2017learning}
Li, K. and Malik, J.
\newblock Learning to optimize neural nets.
\newblock \emph{arXiv preprint arXiv:1703.00441}, 2017.

\bibitem[Li et~al.(2017)Li, Jamieson, DeSalvo, Rostamizadeh, and
  Talwalkar]{li2017hyperband}
Li, L., Jamieson, K., DeSalvo, G., Rostamizadeh, A., and Talwalkar, A.
\newblock {Hyperband: A novel bandit-based approach to hyperparameter
  optimization}.
\newblock \emph{The Journal of Machine Learning Research}, 18\penalty0
  (1):\penalty0 6765--6816, 2017.

\bibitem[Li et~al.(2023)Li, Harrison, Sohl-Dickstein, and Metz]{li2023noise}
Li, O., Harrison, J., Sohl-Dickstein, J., and Metz, L.
\newblock Noise reuse in online evolution strategies.
\newblock In \emph{International Conference on Machine Learning}, 2023.
\newblock Under review.

\bibitem[Lorraine \& Duvenaud(2018)Lorraine and
  Duvenaud]{lorraine2018stochastic}
Lorraine, J. and Duvenaud, D.
\newblock Stochastic hyperparameter optimization through hypernetworks.
\newblock \emph{arXiv preprint arXiv:1802.09419}, 2018.

\bibitem[Lorraine et~al.(2020)Lorraine, Vicol, and
  Duvenaud]{lorraine2020optimizing}
Lorraine, J., Vicol, P., and Duvenaud, D.
\newblock Optimizing millions of hyperparameters by implicit differentiation.
\newblock In \emph{International Conference on Artificial Intelligence and
  Statistics (AISTATS)}, pp.\  1540--1552, 2020.

\bibitem[Luketina et~al.(2016)Luketina, Berglund, Greff, and
  Raiko]{luketina2016scalable}
Luketina, J., Berglund, M., Greff, K., and Raiko, T.
\newblock Scalable gradient-based tuning of continuous regularization
  hyperparameters.
\newblock In \emph{International Conference on Machine Learning (ICML)}, pp.\
  2952--2960, 2016.

\bibitem[MacKay et~al.(2019)MacKay, Vicol, Lorraine, Duvenaud, and
  Grosse]{mackay2019self}
MacKay, M., Vicol, P., Lorraine, J., Duvenaud, D., and Grosse, R.
\newblock {Self-Tuning Networks: Bilevel optimization of hyperparameters using
  structured best-response functions}.
\newblock In \emph{International Conference on Learning Representations
  (ICLR)}, 2019.

\bibitem[Maclaurin et~al.(2015)Maclaurin, Duvenaud, and
  Adams]{maclaurin2015gradient}
Maclaurin, D., Duvenaud, D., and Adams, R.
\newblock Gradient-based hyperparameter optimization through reversible
  learning.
\newblock In \emph{International Conference on Machine Learning (ICML)}, pp.\
  2113--2122, 2015.

\bibitem[Maheswaranathan et~al.(2019)Maheswaranathan, Metz, Tucker, Choi, and
  Sohl-Dickstein]{maheswaranathan2019guided}
Maheswaranathan, N., Metz, L., Tucker, G., Choi, D., and Sohl-Dickstein, J.
\newblock Guided evolutionary strategies: Augmenting random search with
  surrogate gradients.
\newblock In \emph{International Conference on Machine Learning}, pp.\
  4264--4273. PMLR, 2019.

\bibitem[Mania et~al.(2018)Mania, Guy, and Recht]{mania2018simple}
Mania, H., Guy, A., and Recht, B.
\newblock Simple random search provides a competitive approach to reinforcement
  learning.
\newblock \emph{arXiv preprint arXiv:1803.07055}, 2018.

\bibitem[Marcus et~al.(1993)Marcus, Marcinkiewicz, and
  Santorini]{marcus1993building}
Marcus, M.~P., Marcinkiewicz, M.~A., and Santorini, B.
\newblock {Building a large annotated corpus of English: The Penn Treebank}.
\newblock \emph{Computational Linguistics}, 19\penalty0 (2):\penalty0 313--330,
  1993.

\bibitem[McGreivy et~al.(2021)McGreivy, Hudson, and Zhu]{mcgreivy2021optimized}
McGreivy, N., Hudson, S.~R., and Zhu, C.
\newblock Optimized finite-build stellarator coils using automatic
  differentiation.
\newblock \emph{Nuclear Fusion}, 61\penalty0 (2):\penalty0 026020, 2021.

\bibitem[Menick et~al.(2021)Menick, Elsen, Evci, Osindero, Simonyan, and
  Graves]{menick2021practical}
Menick, J., Elsen, E., Evci, U., Osindero, S., Simonyan, K., and Graves, A.
\newblock Practical real time recurrent learning with a sparse approximation.
\newblock In \emph{International Conference on Learning Representations}, 2021.
\newblock URL \url{https://openreview.net/forum?id=q3KSThy2GwB}.

\bibitem[Merity et~al.(2018)Merity, Keskar, and Socher]{merity2017regularizing}
Merity, S., Keskar, N.~S., and Socher, R.
\newblock {Regularizing and optimizing LSTM language models}.
\newblock In \emph{International Conference on Learning Representations
  (ICLR)}, 2018.

\bibitem[Metz et~al.(2018)Metz, Maheswaranathan, Cheung, and
  Sohl-Dickstein]{metz2018meta}
Metz, L., Maheswaranathan, N., Cheung, B., and Sohl-Dickstein, J.
\newblock Meta-learning update rules for unsupervised representation learning.
\newblock \emph{arXiv preprint arXiv:1804.00222}, 2018.

\bibitem[Metz et~al.(2019)Metz, Maheswaranathan, Nixon, Freeman, and
  Sohl-Dickstein]{metz2019understanding}
Metz, L., Maheswaranathan, N., Nixon, J., Freeman, D., and Sohl-Dickstein, J.
\newblock Understanding and correcting pathologies in the training of learned
  optimizers.
\newblock In \emph{International Conference on Machine Learning (ICML)}, pp.\
  4556--4565, 2019.

\bibitem[Metz et~al.(2020{\natexlab{a}})Metz, Maheswaranathan, Freeman, Poole,
  and Sohl-Dickstein]{metz2020tasks}
Metz, L., Maheswaranathan, N., Freeman, C.~D., Poole, B., and Sohl-Dickstein,
  J.
\newblock {Tasks, stability, architecture, and compute: Training more effective
  learned optimizers, and using them to train themselves}.
\newblock \emph{arXiv preprint arXiv:2009.11243}, 2020{\natexlab{a}}.

\bibitem[Metz et~al.(2020{\natexlab{b}})Metz, Maheswaranathan, Sun, Freeman,
  Poole, and Sohl-Dickstein]{metz2020using}
Metz, L., Maheswaranathan, N., Sun, R., Freeman, C.~D., Poole, B., and
  Sohl-Dickstein, J.
\newblock Using a thousand optimization tasks to learn hyperparameter search
  strategies.
\newblock \emph{arXiv preprint arXiv:2002.11887}, 2020{\natexlab{b}}.

\bibitem[Metz et~al.(2021)Metz, Freeman, Schoenholz, and
  Kachman]{metz2021gradients}
Metz, L., Freeman, C.~D., Schoenholz, S.~S., and Kachman, T.
\newblock Gradients are not all you need.
\newblock \emph{arXiv preprint arXiv:2111.05803}, 2021.

\bibitem[Micaelli \& Storkey(2020)Micaelli and Storkey]{micaelli2020non}
Micaelli, P. and Storkey, A.
\newblock Non-greedy gradient-based hyperparameter optimization over long
  horizons.
\newblock \emph{arXiv preprint arXiv:2007.07869}, 2020.

\bibitem[Mujika et~al.(2018)Mujika, Meier, and Steger]{mujika2018approximating}
Mujika, A., Meier, F., and Steger, A.
\newblock {Approximating real-time recurrent learning with random Kronecker
  factors}.
\newblock In \emph{Advances in Neural Information Processing Systems
  (NeurIPS)}, pp.\  6594--6603, 2018.

\bibitem[Nesterov \& Spokoiny(2017)Nesterov and Spokoiny]{nesterov2017random}
Nesterov, Y. and Spokoiny, V.
\newblock Random gradient-free minimization of convex functions.
\newblock \emph{Foundations of Computational Mathematics}, 17\penalty0
  (2):\penalty0 527--566, 2017.

\bibitem[Ni \& Wang(2017)Ni and Wang]{ni2017sensitivity}
Ni, A. and Wang, Q.
\newblock {Sensitivity analysis on chaotic dynamical systems by Non-Intrusive
  Least Squares Shadowing (NILSS)}.
\newblock \emph{Journal of Computational Physics}, 347:\penalty0 56--77, 2017.

\bibitem[Owen(2013)]{mcbook}
Owen, A.~B.
\newblock \emph{Monte Carlo Theory, Methods and Examples}.
\newblock 2013.

\bibitem[Parmas \& Sugiyama(2019)Parmas and Sugiyama]{parmas2019unified}
Parmas, P. and Sugiyama, M.
\newblock A unified view of likelihood ratio and reparameterization gradients
  and an optimal importance sampling scheme.
\newblock \emph{arXiv preprint arXiv:1910.06419}, 2019.

\bibitem[Parmas et~al.(2018)Parmas, Rasmussen, Peters, and
  Doya]{parmas2018pipps}
Parmas, P., Rasmussen, C.~E., Peters, J., and Doya, K.
\newblock {PIPPS: Flexible model-based policy search robust to the curse of
  chaos}.
\newblock In \emph{International Conference on Machine Learning (ICML)}, pp.\
  4062--4071, 2018.

\bibitem[Pedregosa(2016)]{pedregosa2016hyperparameter}
Pedregosa, F.
\newblock Hyperparameter optimization with approximate gradient.
\newblock In \emph{International Conference on Machine Learning (ICML)}, pp.\
  737--746, 2016.

\bibitem[Rechenberg(1973)]{rechenberg1973}
Rechenberg, I.
\newblock \emph{{Evolutionsstrategie: Optimierung technischer Systeme nach
  Prinzipien der biologischen Evolution}}.
\newblock Stuttgart: Frommann-Holzboog, 1973.

\bibitem[Ruiz et~al.(2016)Ruiz, AUEB, Blei, et~al.]{ruiz2016generalized}
Ruiz, F.~R., AUEB, T.~R., Blei, D., et~al.
\newblock The generalized reparameterization gradient.
\newblock \emph{Advances in Neural Information Processing Systems}, 2016.

\bibitem[Rumelhart et~al.(1985)Rumelhart, Hinton, and
  Williams]{rumelhart1985learning}
Rumelhart, D.~E., Hinton, G.~E., and Williams, R.~J.
\newblock Learning internal representations by error propagation.
\newblock Technical report, California University San Diego, La Jolla Institute
  for Cognitive Science, 1985.

\bibitem[Salimans et~al.(2017)Salimans, Ho, Chen, Sidor, and
  Sutskever]{salimans2017evolution}
Salimans, T., Ho, J., Chen, X., Sidor, S., and Sutskever, I.
\newblock Evolution strategies as a scalable alternative to reinforcement
  learning.
\newblock \emph{arXiv preprint arXiv:1703.03864}, 2017.

\bibitem[Schulman et~al.(2015)Schulman, Heess, Weber, and
  Abbeel]{schulman2015gradient}
Schulman, J., Heess, N., Weber, T., and Abbeel, P.
\newblock Gradient estimation using stochastic computation graphs.
\newblock \emph{Advances in Neural Information Processing Systems}, 2015.

\bibitem[Schwefel \& Schwefel(1977)Schwefel and
  Schwefel]{schwefel1977evolutionsstrategien}
Schwefel, H.-P. and Schwefel, H.-P.
\newblock \emph{{Evolutionsstrategien f{\"u}r die numerische Optimierung}}.
\newblock Springer, 1977.

\bibitem[Shaban et~al.(2019)Shaban, Cheng, Hatch, and
  Boots]{shaban2019truncated}
Shaban, A., Cheng, C.-A., Hatch, N., and Boots, B.
\newblock Truncated back-propagation for bilevel optimization.
\newblock In \emph{International Conference on Artificial Intelligence and
  Statistics (AISTATS)}, pp.\  1723--1732, 2019.

\bibitem[Silver et~al.(2021)Silver, Goyal, Danihelka, Hessel, and van
  Hasselt]{silver2021learning}
Silver, D., Goyal, A., Danihelka, I., Hessel, M., and van Hasselt, H.
\newblock Learning by directional gradient descent.
\newblock In \emph{International Conference on Learning Representations}, 2021.

\bibitem[Snoek et~al.(2012)Snoek, Larochelle, and Adams]{snoek2012practical}
Snoek, J., Larochelle, H., and Adams, R.~P.
\newblock {Practical Bayesian optimization of machine learning algorithms}.
\newblock In \emph{Advances in Neural Information Processing Systems
  (NeurIPS)}, pp.\  2951--2959, 2012.

\bibitem[Snoek et~al.(2015)Snoek, Rippel, Swersky, Kiros, Satish, Sundaram,
  Patwary, Prabhat, and Adams]{snoek2015scalable}
Snoek, J., Rippel, O., Swersky, K., Kiros, R., Satish, N., Sundaram, N.,
  Patwary, M., Prabhat, M., and Adams, R.
\newblock {Scalable Bayesian optimization using deep neural networks}.
\newblock In \emph{International Conference on Machine Learning (ICML)}, pp.\
  2171--2180, 2015.

\bibitem[Srivastava et~al.(2014)Srivastava, Hinton, Krizhevsky, Sutskever, and
  Salakhutdinov]{srivastava2014dropout}
Srivastava, N., Hinton, G., Krizhevsky, A., Sutskever, I., and Salakhutdinov,
  R.
\newblock {Dropout: A simple way to prevent neural networks from overfitting}.
\newblock \emph{The Journal of Machine Learning Research}, 15\penalty0
  (1):\penalty0 1929--1958, 2014.

\bibitem[Staines \& Barber(2012)Staines and Barber]{staines2012variational}
Staines, J. and Barber, D.
\newblock Variational optimization.
\newblock \emph{arXiv preprint arXiv:1212.4507}, 2012.

\bibitem[Swersky et~al.(2014)Swersky, Snoek, and Adams]{swersky2014freeze}
Swersky, K., Snoek, J., and Adams, R.~P.
\newblock {Freeze-thaw Bayesian optimization}.
\newblock \emph{arXiv preprint arXiv:1406.3896}, 2014.

\bibitem[Tallec \& Ollivier(2017{\natexlab{a}})Tallec and
  Ollivier]{tallec2017unbiased}
Tallec, C. and Ollivier, Y.
\newblock Unbiased online recurrent optimization.
\newblock \emph{arXiv preprint arXiv:1702.05043}, 2017{\natexlab{a}}.

\bibitem[Tallec \& Ollivier(2017{\natexlab{b}})Tallec and
  Ollivier]{tallec2017unbiasing}
Tallec, C. and Ollivier, Y.
\newblock Unbiasing truncated backpropagation through time.
\newblock \emph{arXiv preprint arXiv:1705.08209}, 2017{\natexlab{b}}.

\bibitem[Vicol et~al.(2021)Vicol, Metz, and Sohl-Dickstein]{vicol2021unbiased}
Vicol, P., Metz, L., and Sohl-Dickstein, J.
\newblock Unbiased gradient estimation in unrolled computation graphs with
  persistent evolution strategies.
\newblock In \emph{International Conference on Machine Learning (ICML)}, pp.\
  10553--10563, 2021.

\bibitem[Vicol et~al.(2022)Vicol, Lorraine, Pedregosa, Duvenaud, and
  Grosse]{vicol2022implicit}
Vicol, P., Lorraine, J., Pedregosa, F., Duvenaud, D., and Grosse, R.
\newblock {On Implicit Bias in Overparameterized Bilevel Optimization}.
\newblock In \emph{International Conference on Machine Learning (ICML)}, 2022.

\bibitem[Vicol(2023)]{vicol2023bilevel}
Vicol, P.~A.
\newblock \emph{On Bilevel Optimization without Full Unrolls: Methods and
  Applications}.
\newblock PhD thesis, University of Toronto (Canada), 2023.

\bibitem[Werbos(1990)]{werbos1990backpropagation}
Werbos, P.~J.
\newblock {Backpropagation through time: What it does and how to do it}.
\newblock \emph{Proceedings of the IEEE}, 78\penalty0 (10):\penalty0
  1550--1560, 1990.

\bibitem[Wichrowska et~al.(2017)Wichrowska, Maheswaranathan, Hoffman,
  Colmenarejo, Denil, de~Freitas, and Sohl-Dickstein]{wichrowska2017learned}
Wichrowska, O., Maheswaranathan, N., Hoffman, M.~W., Colmenarejo, S.~G., Denil,
  M., de~Freitas, N., and Sohl-Dickstein, J.
\newblock Learned optimizers that scale and generalize.
\newblock \emph{arXiv preprint arXiv:1703.04813}, 2017.

\bibitem[Wierstra et~al.(2014)Wierstra, Schaul, Glasmachers, Sun, Peters, and
  Schmidhuber]{wierstra2014natural}
Wierstra, D., Schaul, T., Glasmachers, T., Sun, Y., Peters, J., and
  Schmidhuber, J.
\newblock Natural evolution strategies.
\newblock \emph{The Journal of Machine Learning Research}, 15\penalty0
  (1):\penalty0 949--980, 2014.

\bibitem[Williams \& Peng(1990)Williams and Peng]{williams1990efficient}
Williams, R.~J. and Peng, J.
\newblock An efficient gradient-based algorithm for on-line training of
  recurrent network trajectories.
\newblock \emph{Neural Computation}, 2\penalty0 (4):\penalty0 490--501, 1990.

\bibitem[Williams \& Zipser(1989)Williams and Zipser]{williams1989learning}
Williams, R.~J. and Zipser, D.
\newblock A learning algorithm for continually running fully recurrent neural
  networks.
\newblock \emph{Neural Computation}, 1\penalty0 (2):\penalty0 270--280, 1989.

\bibitem[Wu et~al.(2018)Wu, Ren, Liao, and Grosse]{wu2018understanding}
Wu, Y., Ren, M., Liao, R., and Grosse, R.
\newblock Understanding short-horizon bias in stochastic meta-optimization.
\newblock \emph{arXiv preprint arXiv:1803.02021}, 2018.

\end{thebibliography}
\bibliographystyle{icml2023}

%%%%%%%%%%
% APPENDIX
%%%%%%%%%%
\newpage

\appendix
\onecolumn

\section*{Appendix}

This appendix is structured as follows:
\begin{itemize}
    \item In Section~\ref{app:notation}, we provide an overview of the notation used in this paper.
    \item In Section~\ref{app:extended-related-work}, we provide extended related work.
    \item In Section~\ref{app:exp-details}, we provide experimental details and additional results.
    \item In Section~\ref{app:proofs}, we provide proofs of all statements in the main text.
    \item In Section~\ref{app:stochasticcomp}, we present derivations of the ES-Single and vanilla ES gradient estimators using the framework of stochastic computation graphs.
    \item In Section~\ref{app:generalization-es-pes}, we derive a generalization of both ES-Single and PES. We provide its stochastic computation graph and resulting algorithm.
    \item In Section~\ref{app:bias-variance-generalized}, we derive the variance of a generalized estimator that combines a single perturbation (kept fixed over the course of an inner problem)---as in ES-Single---with independent perturbations sampled in each partial unroll---as in PES.
    \item In Section~\ref{app:code}, we provide a JAX implementation of ES-Single.
\end{itemize}

\clearpage

\section{Notation}
\label{app:notation}

Table~\ref{table:notation} summarizes the notation used in this paper.
\begin{table}[htbp]
\centering
\begin{tabular}{cc}
\toprule
\textbf{Symbol} & \textbf{Meaning} \\
\midrule
 ES          & Evolution strategies \\[3pt]
 PES         & Persistent evolution strategies \\[3pt]
 ES-Single   & Evolution strategies with a single perturbation re-used across unrolls \\[3pt]
 (T)BPTT   & (Truncated) backpropagation through time \\[3pt]
 RTRL      & Real time recurrent learning \\[3pt]
 UORO      & Unbiased online recurrent optimization \\[3pt]
 $T$        & The total sequence length / total unroll length of the inner problem \\[3pt]
 $K$          & The truncation length for subsequences / partial unrolls \\[3pt]
 $S$          & The dimensionality of the state of the unrolled system, $\text{dim}(\bolds)$  \\[3pt]
 $P$          & The dimensionality of the parameters of the unrolled system, $\text{dim}(\boldtheta)$  \\[3pt]
 $\boldtheta$ & The parameters of the unrolled system                    \\[3pt]
 $\boldtheta_t$ & The parameters of the unrolled system at time $t$, where $\boldtheta_t = \boldtheta, \forall t$   \\[3pt]
 $\bolds_t$        & The state of the unrolled system at time $t$             \\[3pt]
 $\boldx_t$        & The (optional) external input to the unrolled system at time $t$    \\[3pt]
 $f$            & The update function that evolves the unrolled system   \\[3pt]
 $N$          & The number of particles for ES and PES                   \\[3pt]
 $\sigma^2$   & The variance of the ES/PES perturbations                 \\[3pt]
 $\boldepsilon_t$ & A perturbation applied to the parameters $\boldtheta$ at timestep $t$\\[3pt]
%  $\boldepsilon$ & A matrix whose rows are perturbations at each timestep, $\boldepsilon = (\boldepsilon_1, \dots, \boldepsilon_T)^\top$ \\[3pt]
 $\boldxi_t$    & The sum of PES perturbations up to time $t$, $\boldxi_t = \boldepsilon_1 + \cdots + \boldepsilon_t$ \\[3pt]
%  $\boldg$        & The ES or PES gradient estimate (depending on context)     \\[3pt]
$\Theta$  & A matrix whose rows are per-timestep parameters $\boldtheta_1, \dots, \boldtheta_T$ \\[3pt]
 $L_t(\Theta)$   & The loss at timestep $t$, $L_t(\Theta) = L_t(\boldtheta_1, \dots, \boldtheta_t)$ \\[3pt]
 $L(\boldtheta)$, $L(\Theta)$   & The total loss, $L(\boldtheta) = L(\Theta) = \sum_{t=1}^T L_t(\Theta) = \sum_{t=1}^T L_t(\boldtheta_1, \dots, \boldtheta_t)$ \\[3pt]
 % $L(\boldtheta)$   & The total loss, $L(\boldtheta) = \sum_{t=1}^T L_t(\bolds_t, \boldtheta)$ \\[3pt]
 $\boldg_t$        & The true gradient at step $t$: $\nabla_{\boldtheta} L_t(\boldtheta)$     \\[3pt]
 $\ges$        & The vanilla ES gradient estimate (with Monte-Carlo sampling)     \\[3pt]
 $\gesanti$        & The vanilla ES gradient estimate, using antithetic sampling     \\[3pt]
 $\gpes$       & The PES gradient estimate (with Monte-Carlo sampling)     \\[3pt]
 $\gessingle$  & The ES-Single gradient estimate (with Monte Carlo sampling) \\[3pt]
%  $\gpesanti$   & The antithetic PES gradient estimate (with Monte-Carlo sampling)     \\[3pt]
 $\alpha$   & The learning rate for the parameters $\boldtheta$          \\[3pt]
 \multirow{3}{*}{$\text{unroll}(\bolds, \boldtheta, K)$} & A function that unrolls the system for $K$ steps \\
            & starting with state $\bolds$, using parameters $\boldtheta$. \\
            & Returns the updated state and loss resulting from the unroll \\
\bottomrule
\end{tabular}
\caption{\small \textbf{Table of notation, defining the terms we use in this paper.}}
\label{table:notation}
\end{table}

\clearpage

\section{Extended Related Work}
\label{app:extended-related-work}

\begin{figure}[H]
    \centering
    \includegraphics[width=\linewidth]{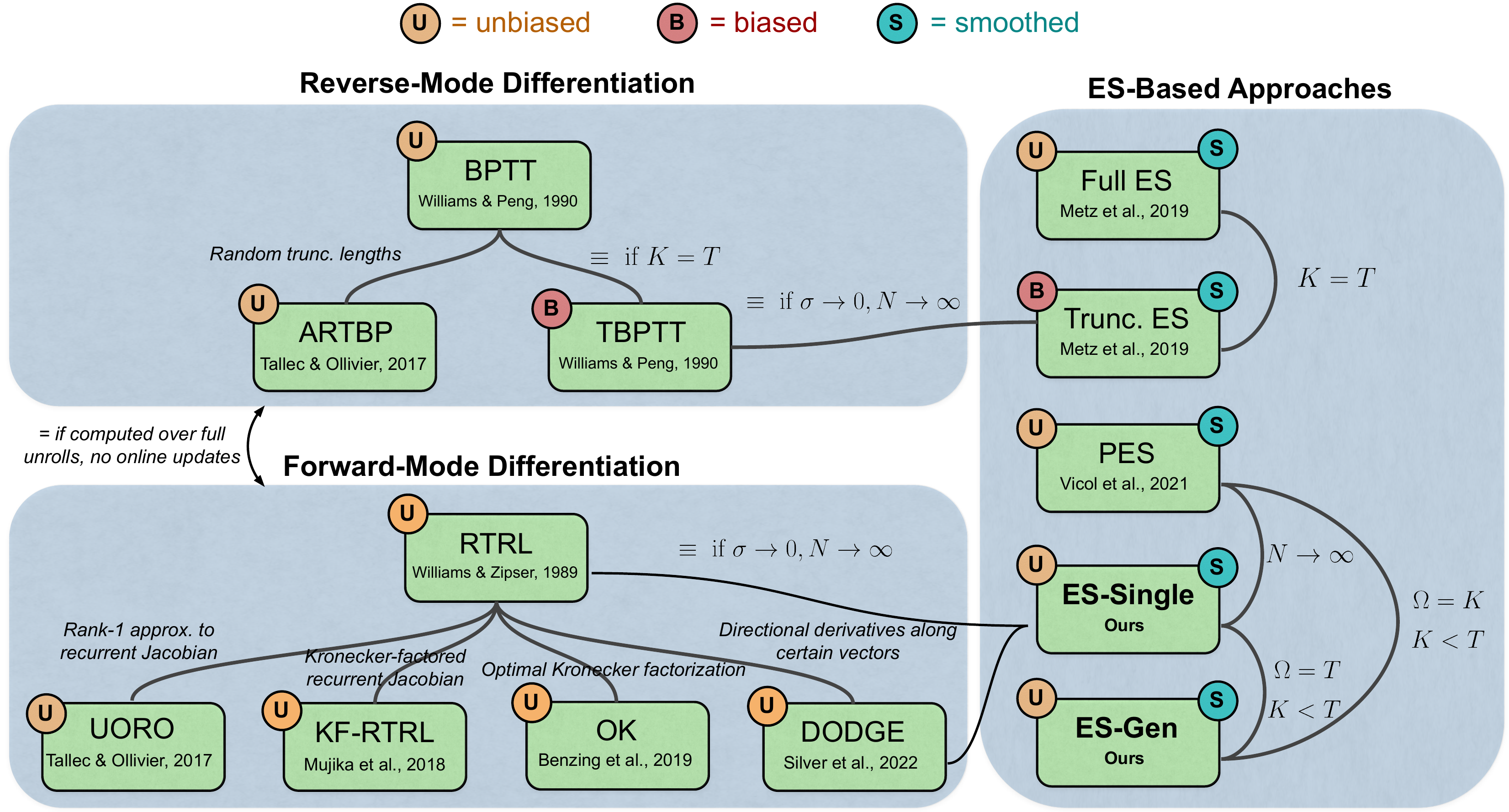}
    \vspace{-0.6cm}
    \caption{Connections between approaches for computing gradients in unrolled computation graphs, focusing on three categories of methods: 1) forward-mode differentiation, which includes RTRL~\cite{williams1989learning} and its approximations (UORO~\cite{tallec2017unbiased}, KF-RTRL~\cite{mujika2018approximating}, OK~\cite{benzing2019optimal}, DODGE~\cite{silver2021learning}); 2) reverse-mode differentiation, which includes backpropagation through time (BPTT), truncated BPTT, and ARTBP~\cite{tallec2017unbiasing}; and 3) evolution strategies (ES)-based approaches, which include full-unroll and truncated ES~\cite{metz2019understanding}, PES~\cite{vicol2021unbiased}, and the generalization we introduce in Section~\ref{app:generalization-es-pes}, which has as special cases PES and ES-Single.}
    \label{fig:es-connections}
\end{figure}

\paragraph{Approaches for Gradient Estimation.}
Figure~\ref{fig:es-connections} illustrates connections between forward-mode, reverse-mode, and evolution strategies-based approaches to gradient estimation in unrolled computation graphs.

\paragraph{Black-Box, Gray-Box, and Gradient-Based Approaches.}
Black-box approaches to meta-optimization include random search~\cite{bergstra2011algorithms}, Bayesian optimization~\cite{snoek2012practical,snoek2015scalable}, and full-unroll ES~\cite{metz2019understanding}.
Gray-box approaches make use of the iterative nature of the inner problem, to make faster progress than black-box methods; such approaches include Freeze-Thaw Bayesian optimization~\cite{swersky2014freeze}, Hyperband~\cite{li2017hyperband}, Successive Halving~\cite{jamieson2016non}, Population-Based Training~\cite{jaderberg2017population}, PES~\cite{vicol2021unbiased}, and ES-Single.
Gradient-based approaches either: 1) differentiate through inner unrolls~\cite{domke2012generic,maclaurin2015gradient,shaban2019truncated}; 2) leverage implicit differentiation~\cite{larsen1996design,bengio2000gradient,foo2008efficient,pedregosa2016hyperparameter,luketina2016scalable,vicol2022implicit,lorraine2020optimizing,blondel2021efficient}; or 3) leverage hypernetworks~\cite{lorraine2018stochastic,mackay2019self}.
There have also been attempts to use forward-mode gradient accumulation for hyperparameter optimization~\cite{franceschi2017forward}, which is only tractable when the hyperparameter dimensionality is very small (e.g., $< 10$).
Most gradient-based approaches perform online, joint optimization over the model parameters and hyperparameters; a notable exception is~\citet{micaelli2020non}, that performs offline updates after each full inner optimization run.
Black-box approaches typically do not scale well beyond $\sim 10$ hyperparameters. While gradient-based approaches are highly scalable, they often suffer from truncation bias, and are typically not applicable to discrete or stochastic hyperparameters (e.g., architectural hyperparameters such as the number of units per layer, or dropout rates).
ES-Single is applicable to a broad range of hyperparameters, including continuous, discrete, or stochastic (e.g., dropout~\cite{srivastava2014dropout}) hyperparameters. In addition, it can target non-differentiable meta-objectives, such as accuracy rather than loss.

\paragraph{Compute and Memory Cost.}
Table~\ref{table:computation-comparison} is an extension of Table 1 from~\citet{vicol2021unbiased}, including an additional row for ES-Single.
The compute cost of ES-Single is identical to that of PES.
Similarly to PES, ES-Single maintains the states of $N$ particles, with memory cost $NS$.
However, ES-Single does not need to store perturbation accumulators.
If the perturbations used by each particle (over the course of all unrolls in an inner problem) are sampled once at the start of the inner problem and stored in memory, then this would require $NP$ memory (similarly to the perturbation accumulators).
But the perturbations do not need to be stored this way, as they can be re-sampled using the same random seed in each partial unroll.
Thus, depending on the implementation, ES-Single has memory cost less than or equal to PES.

\begin{table}[H]
\centering
\label{table:comparison}
\caption{\textbf{Comparison of approaches for learning parameters in unrolled computation graphs}. $S$ is the size of the system state (e.g. the RNN hidden state dimension, or in the case of hyperparameter optimization the inner-problem's weight dimensionality and potentially the optimizer state; $P$ is the dimensionality of $\boldtheta$; $T$ is the total number of steps in a sequence/unroll; $K$ is the truncation length; and $N$ is the number of samples (also called \textit{particles}) used for the reparameterization gradient and in ES-based algorithms; $F$ and $B$ are the costs of a forward and backward pass, respectively; terms in {\color{dkred} red} denote computation/memory that can be split across parallel workers.
}
\footnotesize
\resizebox{\textwidth}{!}{%
\begin{tabular}{@{}cccccccc@{}}
\toprule
\thead{\textbf{Method}}             & \thead{\textbf{Compute}}    & \thead{\textbf{Memory}}    & \thead{\textbf{Parallel}}      & \thead{\textbf{Unbiased}} & \thead{\textbf{Optimize}\\\textbf{ Non-Diff.}} & \thead{\textbf{Smoothed}} \\ \midrule
BPTT{\footnotesize~\cite{rumelhart1985learning}}           & $T(F+B)$               & $TS$           & \xmark  & \cmark & \xmark & \xmark  \\
% BPTT Checkpoint{\footnotesize~\cite{chen2016training}}     & $T \log(T)n^2$         & $TS\log(T)n$               & \xmark  & \cmark & \xmark & \xmark  \\
TBPTT{\footnotesize~\cite{williams1990efficient}}          & $K(F+B)$               & $KS$                       & \xmark  & \xmark & \xmark & \xmark  \\
ARTBP{\footnotesize~\cite{tallec2017unbiasing}}            & $K(F+B)$               & $KS$                       & \xmark  & \cmark & \xmark & \xmark  \\
%RTRL{\footnotesize~\cite{williams1989learning}}            & $S^2 P$                & $\text{max}\{ SP, S^2 \}$  & \xmark  & \cmark & \xmark & \xmark  \\
RTRL{\footnotesize~\cite{williams1989learning}}            & $PS^2 + S(F+B)$                & $SP + S^2$  & \xmark  & \cmark & \xmark & \xmark  \\
% UORO{\footnotesize~\cite{tallec2017unbiased}}              & $PS + (F + B)$                  & $S + P$                      & \xmark  & \cmark & \xmark & \xmark  \\
UORO{\footnotesize~\cite{tallec2017unbiased}}              & $F + B + S^2 + P$                  & $S + P$                      & \xmark  & \cmark & \xmark & \xmark  \\
% KF-RTRL{\footnotesize~\cite{mujika2018approximating}}      & $r n^3$                & $r n^2$                    & \xmark  & \cmark & \xmark & \xmark  \\
% OK{\footnotesize~\cite{benzing2019optimal}}                & $r n^3$                & $r n^2$                    & \xmark  & \cmark & \xmark & \xmark  \\
Reparam.{\footnotesize~\cite{metz2019understanding}}       & ${\color{dkred}N} T(F+B)$               & ${\color{dkred}N} TS$                       & \cmark  & \cmark & \xmark & \cmark \\
ES{\footnotesize~\cite{rechenberg1973}}             & ${\color{dkred}N}TF$                  & ${\color{dkred}N}S$     & \cmark  & \cmark & \cmark & \cmark  \\
Trunc. ES{\footnotesize~\cite{metz2019understanding}}      & ${\color{dkred}N}KF$                  & ${\color{dkred}N}S$     & \cmark  & \xmark & \cmark & \cmark  \\
PES~\cite{vicol2021unbiased} & ${\color{dkred}N}$$KF$                  & ${\color{dkred}N}(S+P)$ & \cmark  & \cmark & \cmark & \cmark  \\
% PES + Analytic~\cite{vicol2021unbiased}                                                  & ${\color{dkred}N}$$KF + K(F+B)$                   & ${\color{dkred}N}(S+P) + (K+1)S$  & \cmark  & \cmark & \xmark & \cmark  \\
\midrule
\textbf{ES-Single (Ours)}  & ${\color{dkred}N}$$KF$                   & ${\color{dkred}N}(S+P)$  & \cmark  & \cmark & \cmark & \cmark  \\
\bottomrule
\end{tabular}}
\vspace{-0.6cm}
\label{table:computation-comparison}
\end{table}

\section{Experimental Details and Additional Results}
\label{app:exp-details}

In this section, we provide experimental details and additional results comparing ES-Single to truncated ES and PES.
For all approaches (vanilla ES, PES, and ES-Single), we use antithetic sampling.

\subsection{Truncated ES}
\label{app:truncated-es}

Figure~\ref{fig:truncated-es-graph} shows the computation graph for vanilla truncated ES, to illustrate how it differs from full-unroll ES, PES, and ES-Single as shown in Figure~\ref{fig:es-single-diagrams}.

\begin{figure}[H]
    \centering
    \includegraphics[width=0.45\linewidth]{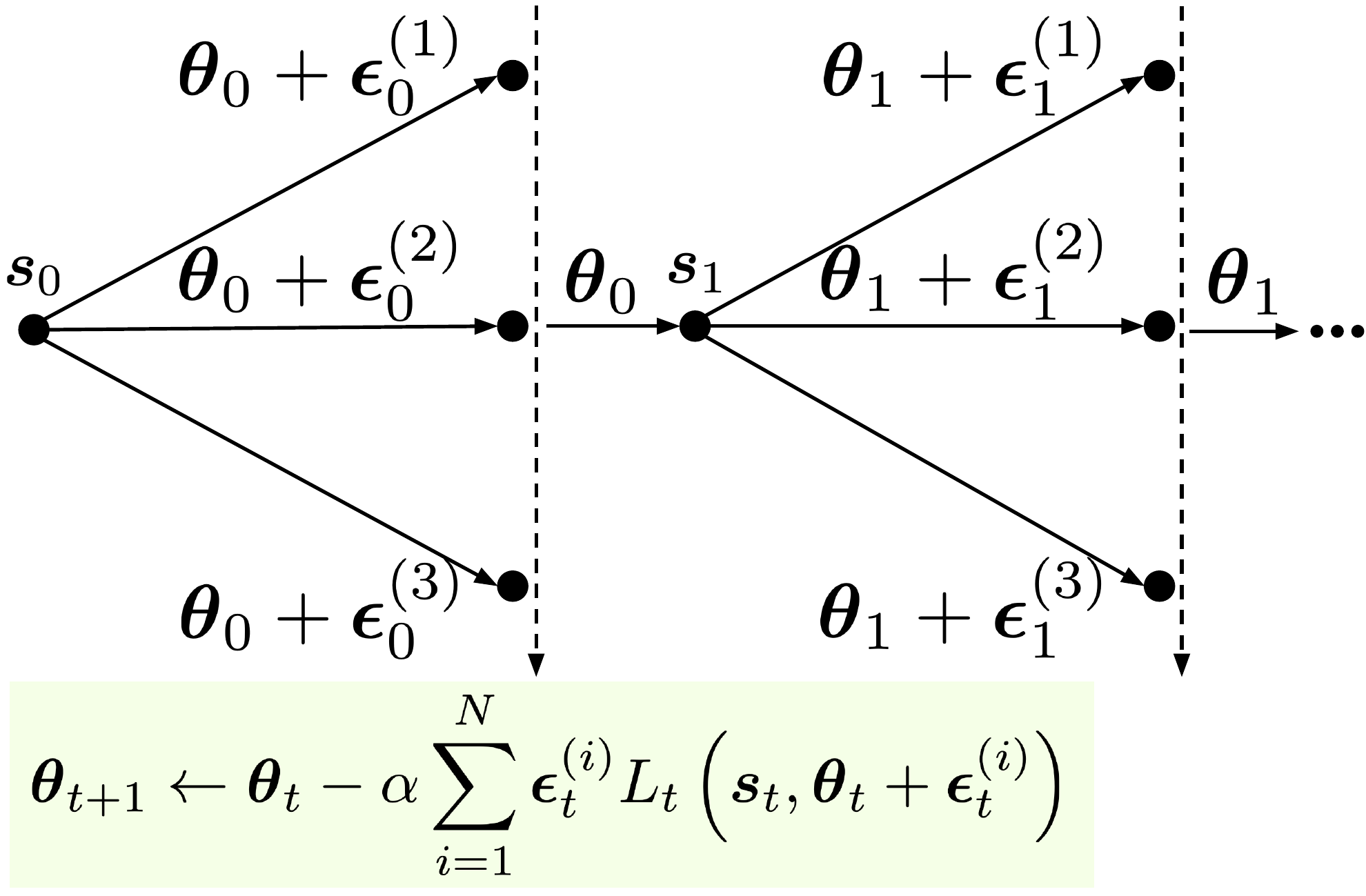}
    \caption{Computation graph for vanilla truncated ES.
    Note that truncated ES may be applied in two different ways.
    In the first approach, a single state $\bolds_t$ is maintained at time $t$, which serves as the common initialization for evaluating $N$ outer parameter perturbations $\{\boldtheta_t + \boldepsilon^{(i)}_t\}_{i=1}^N$. The losses obtained from these partial unrolls are aggregated to form a gradient estimate used to update $\boldtheta_t \to \boldtheta_{t+1}$. After each partial unroll, the states resulting from the $N$ perturbations are discarded, and the single state $\bolds_t$ is unrolled using the mean parameters $\boldtheta_t$, yielding the new initialization $\bolds_{t+1}$ for the subsequent unroll.
    In the second approach, separate states are maintained for each particle over the course of meta-optimization, and different random perturbations are used to unroll those states in each partial unroll. Both approaches suffer from truncation bias.
    }
    \label{fig:truncated-es-graph}
\end{figure}

\paragraph{Hyperparameter Optimization for UCI Regression.}
\begin{figure}[H]
    \centering
    \begin{subfigure}[t]{0.4\linewidth}
        \includegraphics[width=\linewidth]{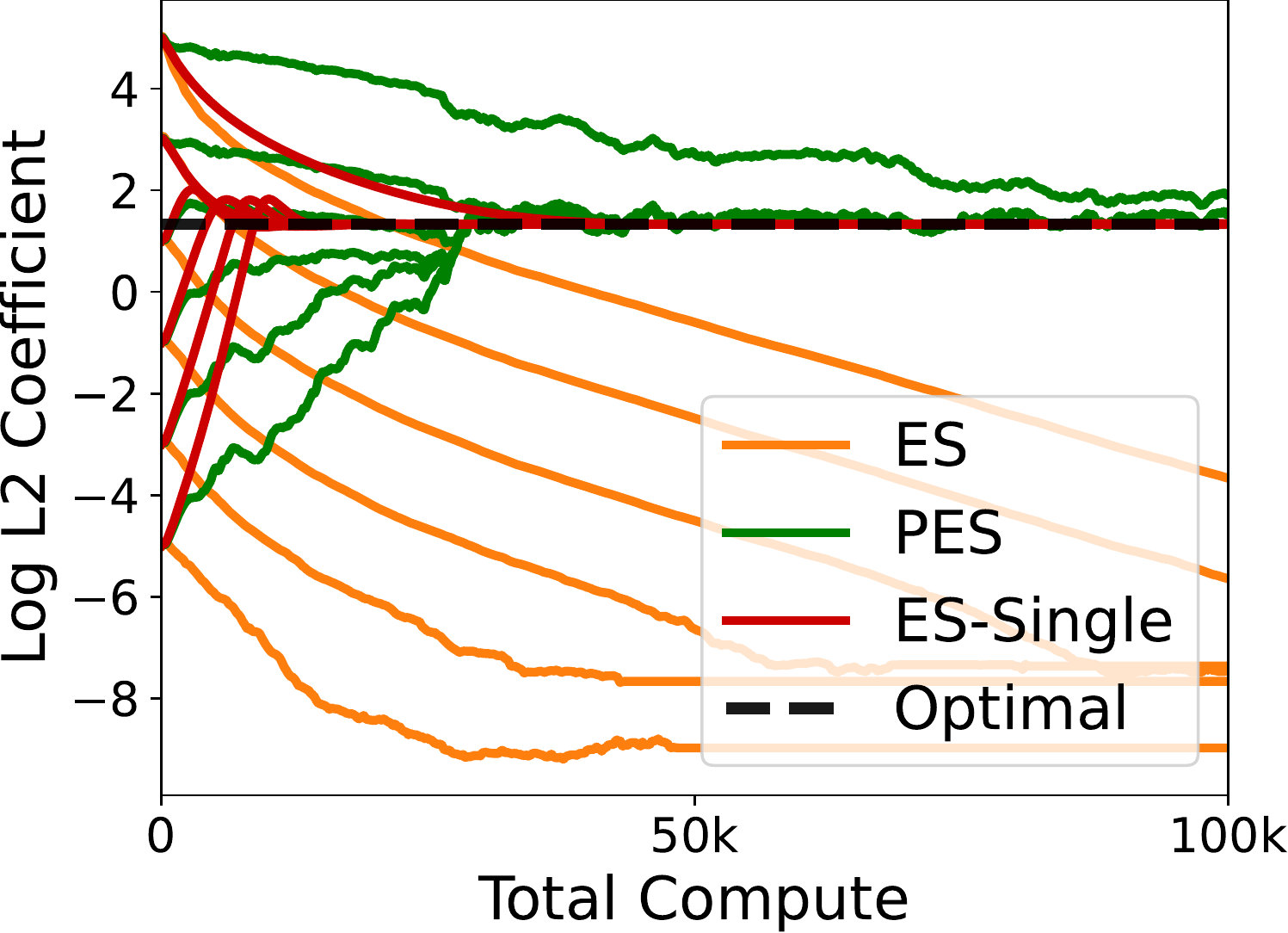}
        \caption{Meta-optimization trajectories for ES, PES, and ES-Single, starting from different initializations, $\{-5, -3, -1, 1, 3, 5 \}$ in log-space.}
        \label{fig:uci-meta-trajectories}
    \end{subfigure}
    \qquad
    \begin{subfigure}[t]{0.4\linewidth}
        \includegraphics[width=\linewidth]{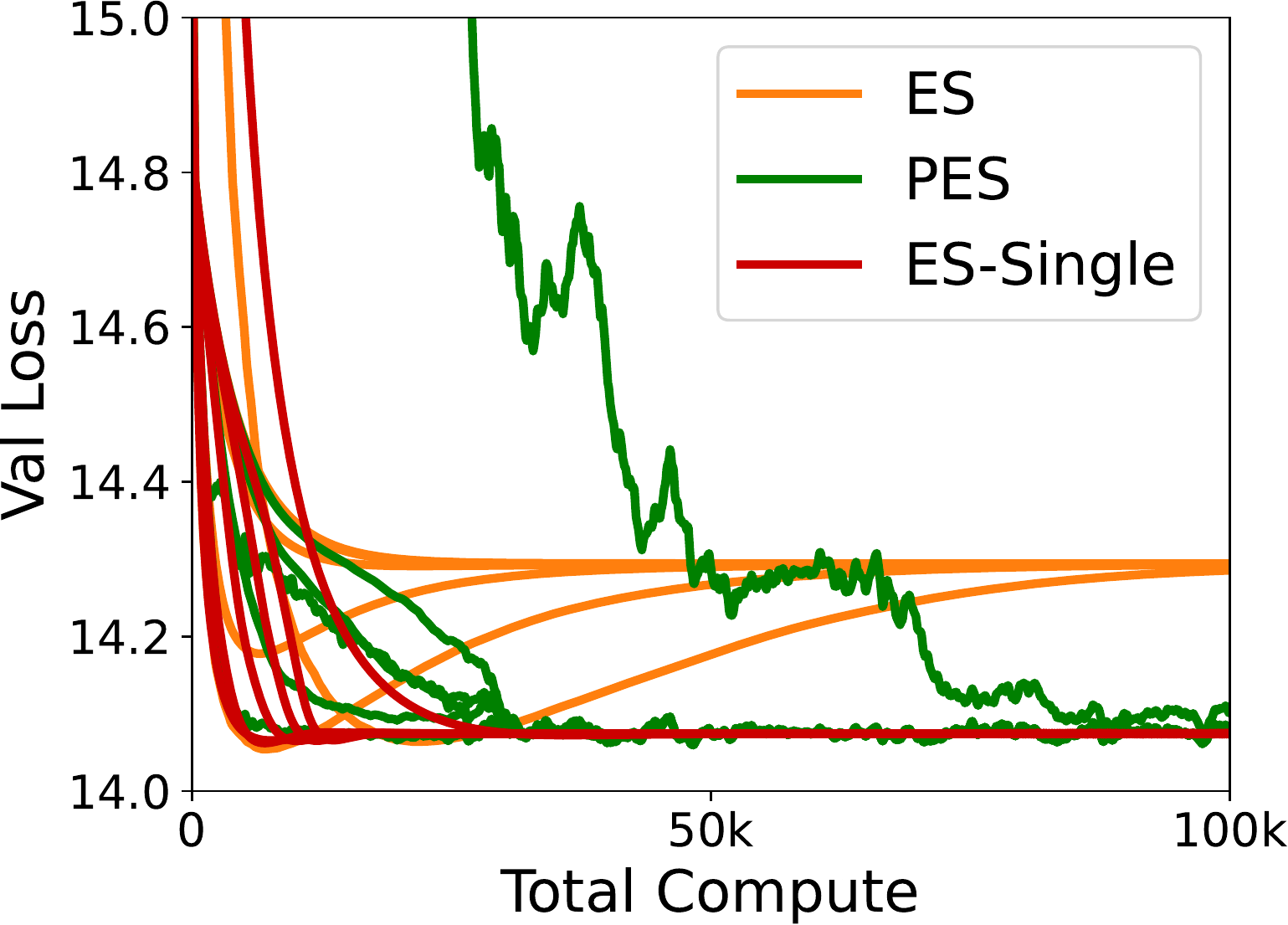}
        \caption{Validation losses attained by each method over the course of meta-optimization.}
        \label{fig:uci-val-loss}
    \end{subfigure}
    \caption{Comparing meta-optimization trajectories and validation losses obtained by ES, PES, and ES-Single when tuning a global $L_2$ regularization coefficient for linear regression on the UCI Yacht dataset.}
    \label{fig:uci}
\end{figure}
We also revisited the UCI linear regression task used in~\citet{vicol2021unbiased}, which demonstrates that truncation bias can also affect regularization hyperparameters (not only optimization hyperparameters).
In this task, we tune a global $L_2$ regularization coefficient for linear regression on the UCI Yacht dataset~\cite{asuncion2007uci}; the training set for this dataset is small, and thus strong regularization is necessary to obtain good validation performance.
In Figure~\ref{fig:uci-meta-trajectories}, we plot the optimal log $L_2$ coefficient obtained via a fine-grained grid search (dashed black line), and compare the meta-optimization trajectories of ES, PES, and ES-Single.
In Figure~\ref{fig:uci-val-loss}, we show the corresponding validation losses attained by each method.
In this task, the inner problem has an infinite horizon; it is never reset.
We found that ES-Single converged to the optimal $L_2$ value more rapidly and stably than PES.
All methods used Adam with learning rate 0.003 for outer optimization, $\sigma=0.01$, and $N=4$ particles.

\paragraph{Influence Balancing.}

Written out, the dynamical system update $\bolds_{t+1} = \boldA \bolds_{t} + \boldtheta$ is:
\begin{align}
    \begin{bmatrix}
      s_{t+1}^{(1)} \\
      s_{t+1}^{(2)} \\
      \vdots \\
      s_{t+1}^{(n)}
    \end{bmatrix}
    =
    \begin{bmatrix}
    \frac{1}{2} & \frac{1}{2} & 0 & \cdots & 0 \\
    0 & \frac{1}{2} & \frac{1}{2} & \cdots & 0 \\
    \vdots & \vdots & \vdots & \ddots & \vdots \\
    0 & 0 & 0 & \cdots & \frac{1}{2}
    \end{bmatrix}
    \begin{bmatrix}
    s_t^{(1)} \\
    s_t^{(2)} \\
    \vdots \\
    s_t^{(n)}
    \end{bmatrix}
    +
    \begin{bmatrix} 
    \theta \\
    \vdots \\
    \theta \\
    -\theta \\
    \vdots \\
    -\theta
    \end{bmatrix}
\end{align}
The loss $L_t$ computes the squared error on the first index in the state vector $\bolds_t$:
\begin{align}
    \mathcal{L}(\boldtheta)
    =
    \sum_{t=1}^T L_t(\boldtheta)
    =
    \sum_{t=1}^T \frac{1}{2} \left( s_t^{(1)} - 1 \right)^2
\end{align}
In our experiments, we used a state $\bolds_t$ of dimension $n=23$, and used $p=10$ positive copies of the scalar parameter $\theta$ concatenated with $n-p=13$ negative copies.
We initialized the state to a vector of ones, $\bolds_0 = \boldone$, and we inittialized $\theta = 0.5$.

\paragraph{Toy 2D Regression.}

Here, we evaluated ES-Single on a synthetic 2D task introduced by~\citet{vicol2021unbiased}, which aims to learn a linearly-decaying learning rate schedule for a regression problem.
The inner problem is designed to have a single global optimum but many local optima, such that small changes in the learning rate schedule can lead to convergence to different local minima; this yields a chaotic meta-loss landscape, and makes the task challenging for gradient-based outer optimizers.
The learning rate at iteration $t$ is parameterized by $\alpha_t = (1 - \frac{t}{T}) e^{\theta_0} + \frac{t}{T} e^{\theta_1}$.

The inner problem involves optimizing parameters $\boldx = (x_0, x_1)$ to minimize the following objective function:
\begin{align}
f(x_0, x_1) = \sqrt{x_0^2 + 5} - \sqrt{5} + \sin^2(x_1) \exp(-5 x_0^2) + 0.25 | x_1 - 100 |
\end{align}
We used total inner problem length $T=100$ and truncations of length $K=10$. For all ES-based methods, we used $N=100$ particles. For vanilla truncated ES, we used perturbation scale $\sigma=1$, while for PES and ES-Single, we used perturbation scale $\sigma=0.3$. For all methods, we performed outer optimization using Adam with learning rate 0.01.
The results are shown in Figure~\ref{fig:toy-regression}.
We found that ES-Single performed similarly to PES, both finding the optimal region of the meta-loss landscape, while the gradient-based methods (TBPTT, UORO, and RTRL) failed due to chaos in the meta-loss, and while ES failed due to truncation bias.

\begin{figure}[H]
    \centering
    \begin{subfigure}[t]{0.45\linewidth}
        \includegraphics[width=\linewidth]{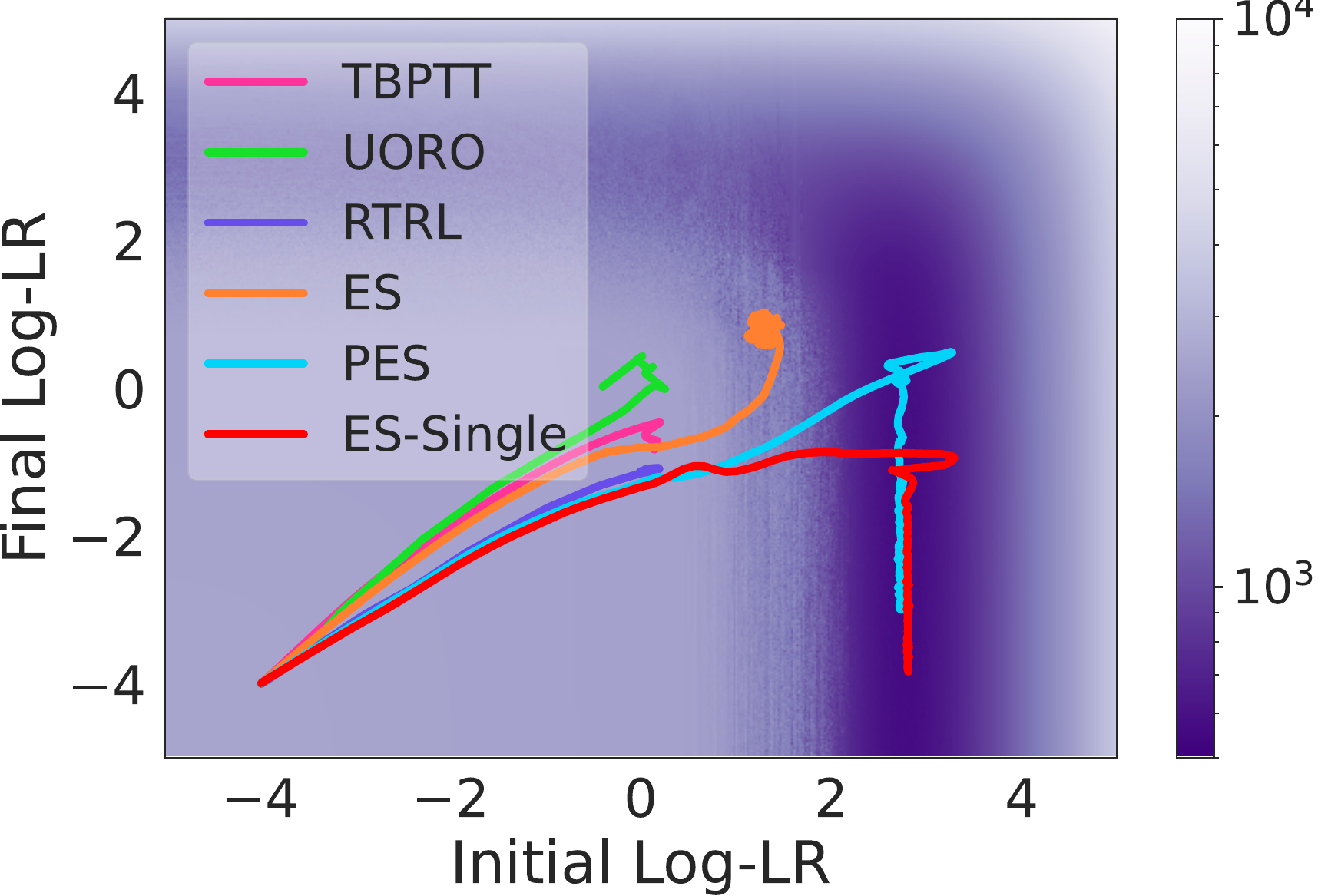}
        \caption{Meta-optimization trajectories.}
        \label{fig:toy-meta-opt}
    \end{subfigure}
    \hfill
    \begin{subfigure}[t]{0.45\linewidth}
        \includegraphics[width=\linewidth]{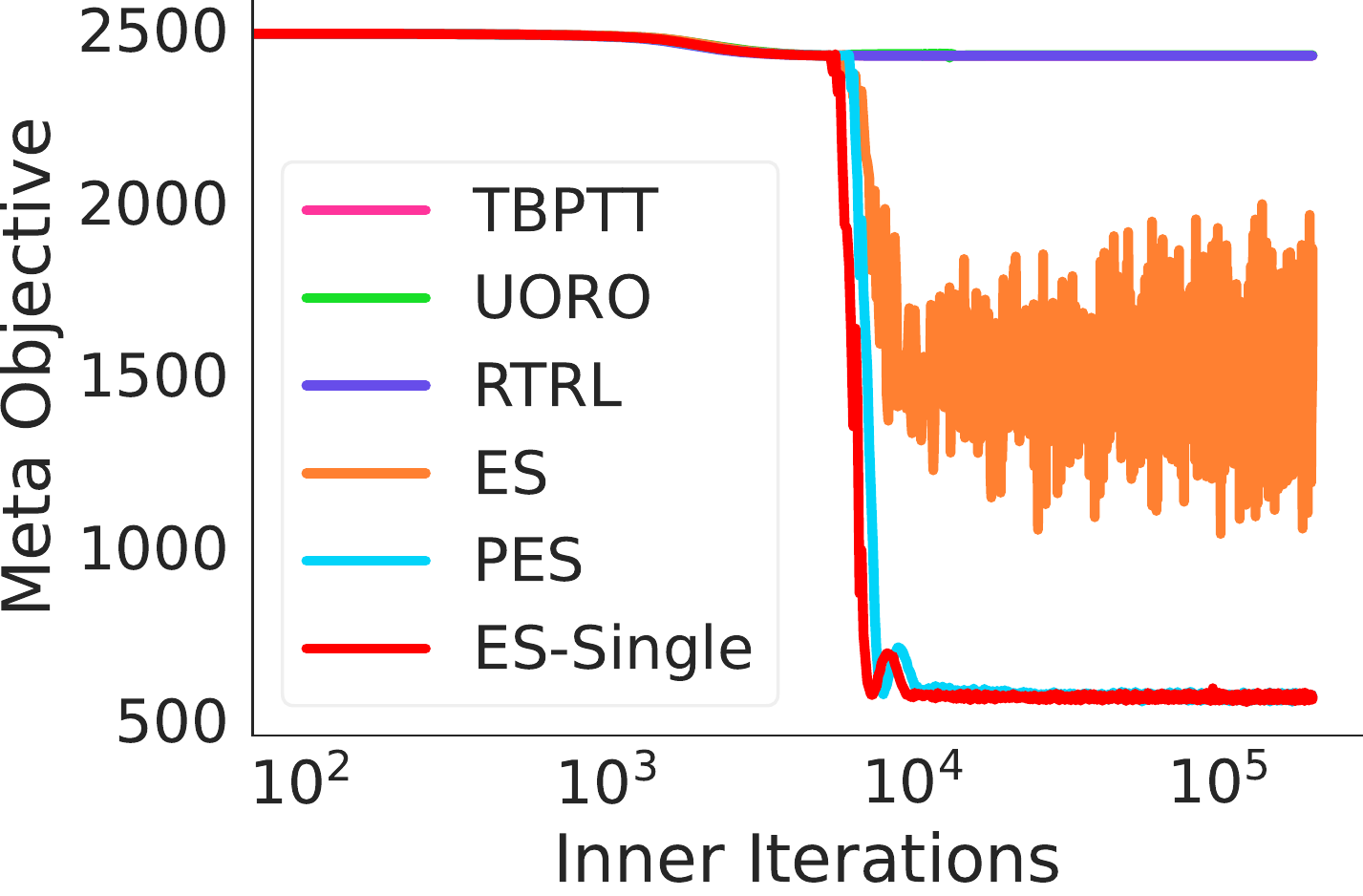}
        \caption{Meta-objective values attained by each algorithm.}
        \label{fig:toy-meta-loss}
    \end{subfigure}
    \caption{Toy regression problem, heatmap with meta-optimization trajectories overlaid, and meta-objective values over the course of training. Darker regions represent lower meta-objective values.}
    \label{fig:toy-regression}
\end{figure}

\paragraph{LSTM Copy Task.}

We train on minibatches of size 32, and feed the ES gradient estimates into Adam with default parameters $\beta_1 = 0.9, \beta_2=0.999$;
for each method, we performed a grid search over learning rates $\alpha \in \{0.01, 0.001, 0.0001\}$ and perturbation scales $\sigma \in \{ 0.1, 0.01, 0.001, 0.0001 \}$, choosing the best values based on final training performance.
We used $N=1000$ particles for each method.

As an additional result, in Figure~\ref{fig:lstm-copy-task-tbptt}, we compare PES and ES-Single to truncated backpropagation through time (denoted by TBP in the legend).
Similarly to truncated ES in Figure~\ref{fig:lstm-copy-task}, TBP also plateaus for each truncation length, as it is intrinsically limited with respect to the horizon that it can memorize.
\begin{figure}[H]
    \vspace{-0.2cm}
    \centering
    \includegraphics[width=0.4\linewidth]{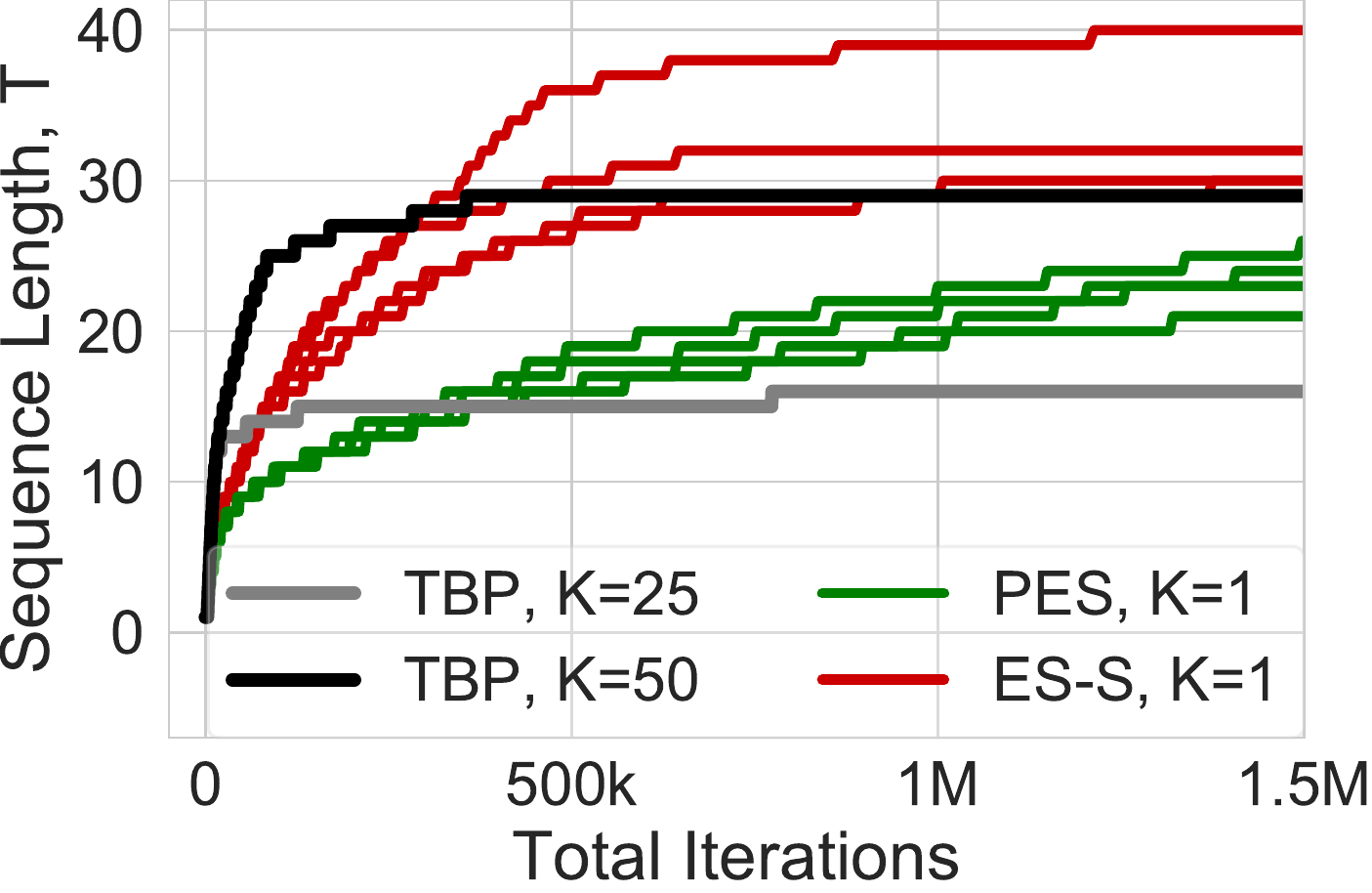}
    \vspace{-0.2cm}
    \caption{Maximum sequence length $T$ that is successfully copied in the copy task from~\citet{mujika2018approximating}.
    Curves of the same color use different random seeds.
    For the truncated backprop through time (TBP) baselines (gray and black curves), we show only the best result to reduce clutter.
    Here, the x-axis represents the number of tokens ingested by each approach (e.g., data-time rather than compute).
    }
    \label{fig:lstm-copy-task-tbptt}
\end{figure}

\paragraph{Meta-Learning MNIST LR Schedule.}
We used a two-hidden-layer MLP with 100 hidden units per layer and ReLU activations.
The learning rate schedule we meta-learn is applied to SGD with momentum, using a fixed momentum coefficient of 0.9.
The total inner problem length is $T=5000$, which is split into 500 partial unrolls of length $K=10$.
We used $N=1000$ particles and $\sigma=0.1$ for each estimator, and we used Adam with learning rate 1e-2 for outer optimization. 

When using PES, there is a trade-off between stability and convergence speed; using a large learning rate may yield fast progress, but may lead to unstable convergence, where meta-optimization diverges away from the optimum, as shown in Figure~\ref{fig:mnist-lr-large-lr}.
Reducing the learning rate may avoid such unstable behavior, but leads to much slower progress compared to ES-Single, as shown in Figure~\ref{fig:mnist-small-lr}.
In this experiment, ES-Single was more stable using larger learning rates than PES.

\begin{figure}[H]
    \centering
    \includegraphics[width=0.4\linewidth]{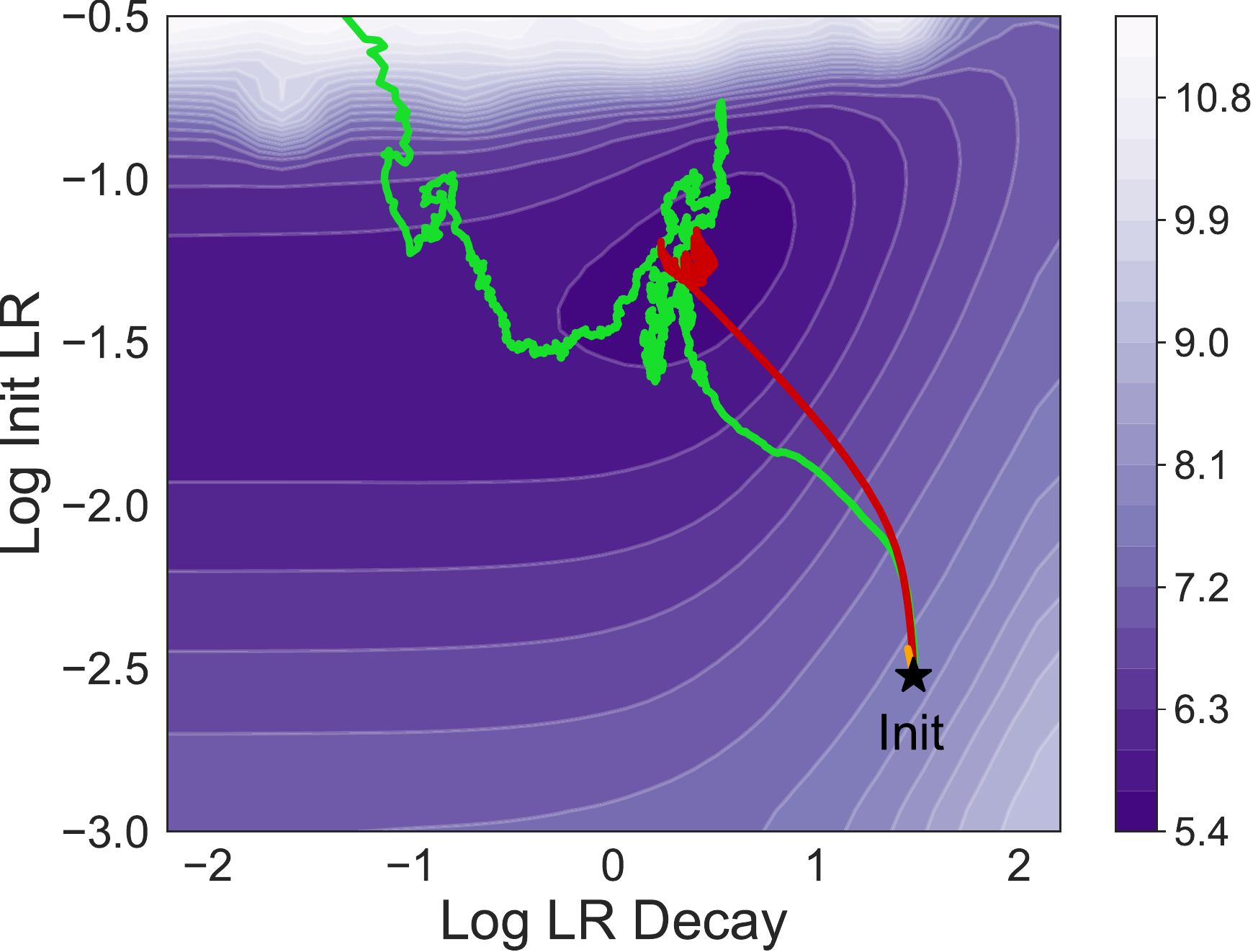}
    \caption{Meta-optimization trajectories of PES and ES-Single using truncations of length $K=1$ on an inner problem of length $T=5000$. With outer learning rate 1e-2, ES-Single performs well and converges stably to the optimum, while PES explodes due to variance. The red curve denotes ES-Single, while the green curve denotes PES.}
    \label{fig:mnist-lr-large-lr}
\end{figure}

\begin{figure}[H]
    \centering
    \includegraphics[width=0.4\linewidth]{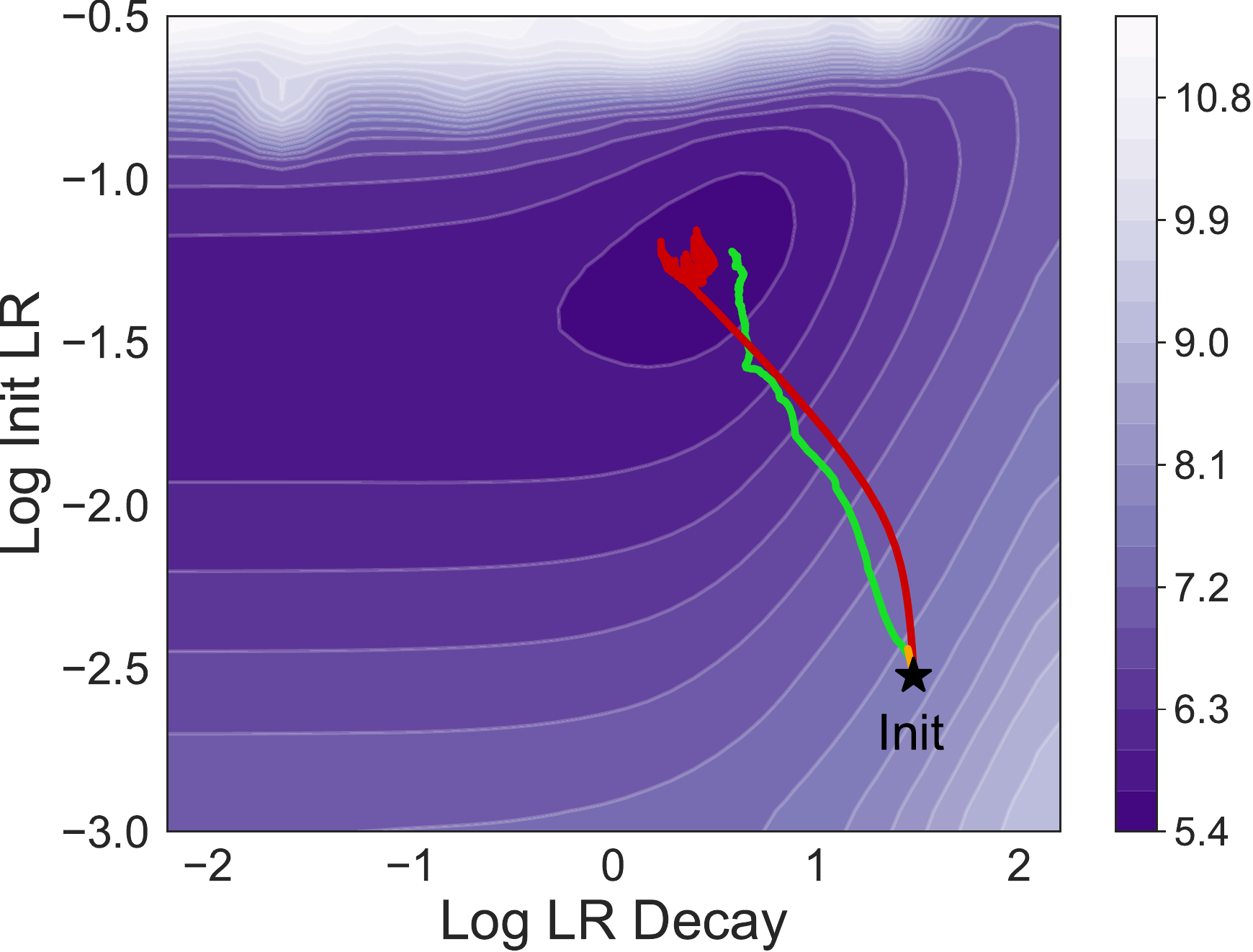}
    \quad
    \includegraphics[width=0.4\linewidth]{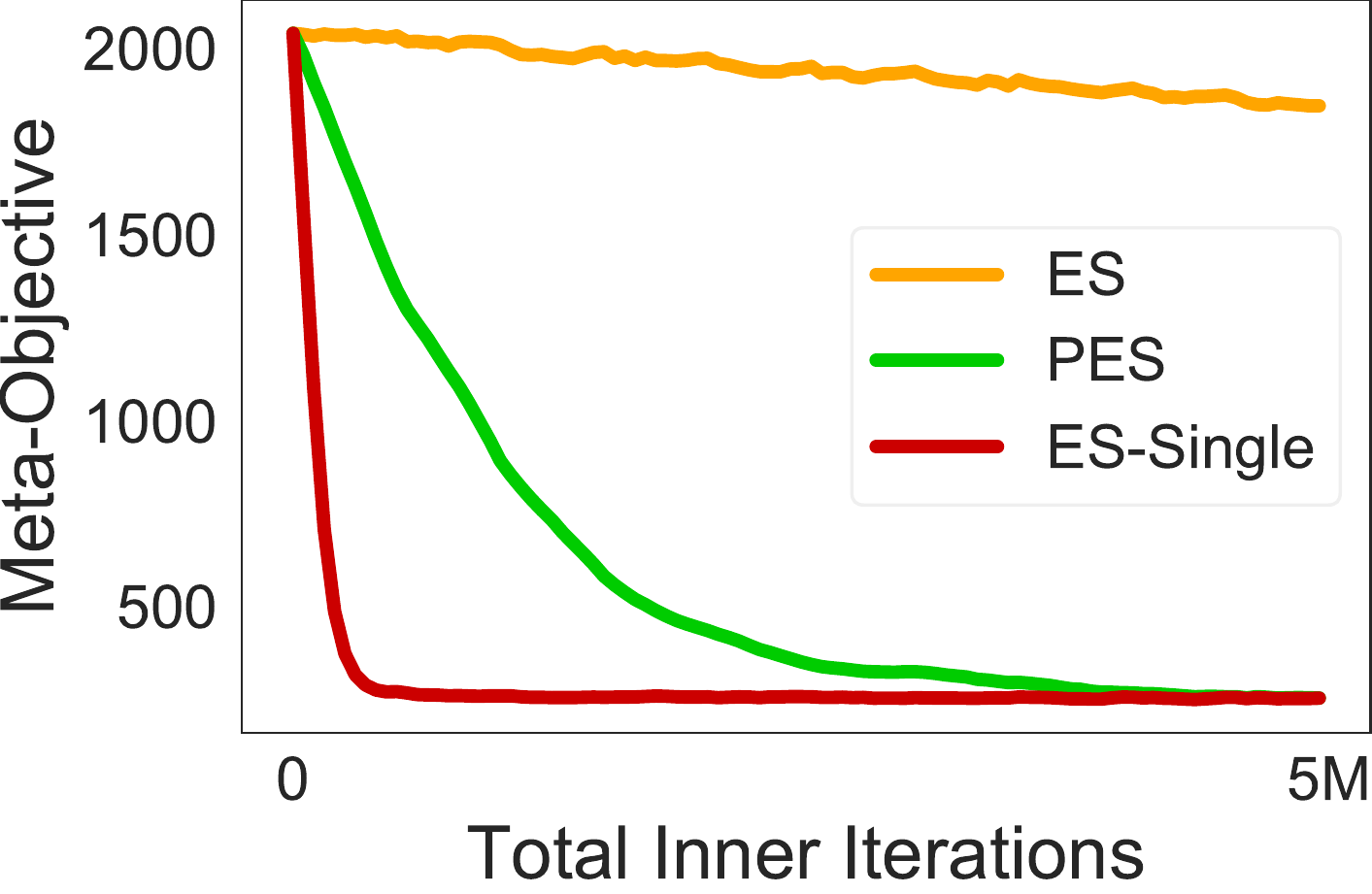}
    \caption{Meta-optimization trajectories of ES, PES, and ES-Single using truncations of length $K=1$ on an inner problem of length $T=5000$. The ES curve makes very little progress, so is obscured by the others. When using a smaller outer learning rate of 1e-3 for PES (to prevent it from exploding), it converges stably, but slowly.}
    \label{fig:mnist-small-lr}
\end{figure}

\paragraph{Tuning Many Hyperparameters.}
Following~\citet{vicol2021unbiased}, we trained an MLP with 5 hidden layers and ReLU activations on FashionMNIST, using the sum of validation losses along the inner optimization trajectory as the meta-objective. The total inner problem length was $T=1000$, and we used truncations of length $K=10$, yielding 100 partial unrolls per inner problem. For ES, PES, and ES-Single, we used perturbation scale $\sigma=0.3$ and $N=10$ particles. We used Adam as the outer optimizer, with learning rate 1e-2. The inner problem used SGD with momentum as the inner optimizer, and trained on minibatches of size 100. We tuned 29 hyperparameters in total, consisting of: a separate learning rate and momentum coefficient for each weight matrix and bias vector in the MLP (yielding 2 hyperparameters for each of the 6 weight matrices and 6 bias vectors, for 24 hyperparameters); and the number of hidden units per layer, which yields 5 additional hyperparameters.
To tune the number of hidden units per layer, we used a nested dropout scheme, where the hyperparameter specifies the fraction of the maximum number of units that should be used. We set the maximum number of hidden units to 100 in each layer. All hyperparameters are optimized in the unconstrained space; each one is mapped to a constrained space by a hyperparameter-specific transformation. For learning rates, we used an exponential mapping; for momentum coefficients, we used the logistic sigmoid (such that the momentum is constrained within $(0, 1)$); and for the number of hidden units per layer, we used the sigmoid to map the unconstrained parameterization to $(0, 1)$, which represents the structured dropout rate.

For ES, PES, and ES-Single, we initialized the hyperparameters randomly: the learning rates were initialized uniformly at random in log-space, in the range $(1e-4, 1e-2)$; the momentum coefficients were initialized uniformly at random in logit space corresponding to the sigmoid-transformed range $(0.01, 0.9)$; and the number of hidden units is initialized randomly in logit-space corresponding to the sigmoid-transformed range $(0.2, 0.8)$.
In Figure~\ref{fig:tuning-many-hparams}, we track the best meta-objective value obtained so far over the course of meta-optimization. We measure the performance as a function of total compute, which takes into account the total number of inner iterations performed (considering the number of particles used). We ran each method four times using different random seeds, and plot the mean performance as well as the min and max shown via shaded regions.

\paragraph{Telescoping Sums.}
We trained a 2-layer MLP with 100 hidden units per layer.
As a computationally tractable proxy for the loss on the full training set, we sample a minibatch of size 1000 at the start of each inner problem, which is kept fixed for the duration of the problem---for telescoping sums, we evaluate the loss on the same minibatch after each partial unroll, as opposed to sampling a different random minibatch in each step, as we do when targeting the sum of losses.

\paragraph{Meta-Gradient Comparison.}
Here, we illustrate how the difference in variance between ES-Single and PES manifests in a simple hyperparameter optimization task.
We tune a global learning rate used to train an MLP on MNIST.
Because the outer parameter is 1-dimensional, we can find the global optimum via a fine-grained grid search, and visualize the meta-optimization iterates and meta-gradients of each algorithm.
Figure~\ref{fig:global-lr-variance} compares the learning rates adapted by PES and ES-Single over the course of meta-optimization; the dashed vertical lines indicate the start of a new inner problem (training is performed in lock-step, where all particles progress through the inner problem at identical iterates).
We found that ES-Single converged stably towards the optimal solution, while PES was less stable due to variance. In Figure~\ref{fig:global-lr-gradients}, we show the gradient estimates produced by PES and ES-Single: we observe that the PES gradient exhibits increasingly large fluctuations over the course of each inner problem, while the ES-Single gradient is more stable.
For this task, we used total inner problem length $T=5000$, partial unrolls of length $K=10$, and $N=10$ particles. For each estimator, we performed a grid search over the outer learning rate and perturbation scale, choosing the best values based on convergence speed and final performance.

\begin{figure}
    \centering
    \begin{subfigure}[t]{0.35\linewidth}
        \includegraphics[width=\linewidth]{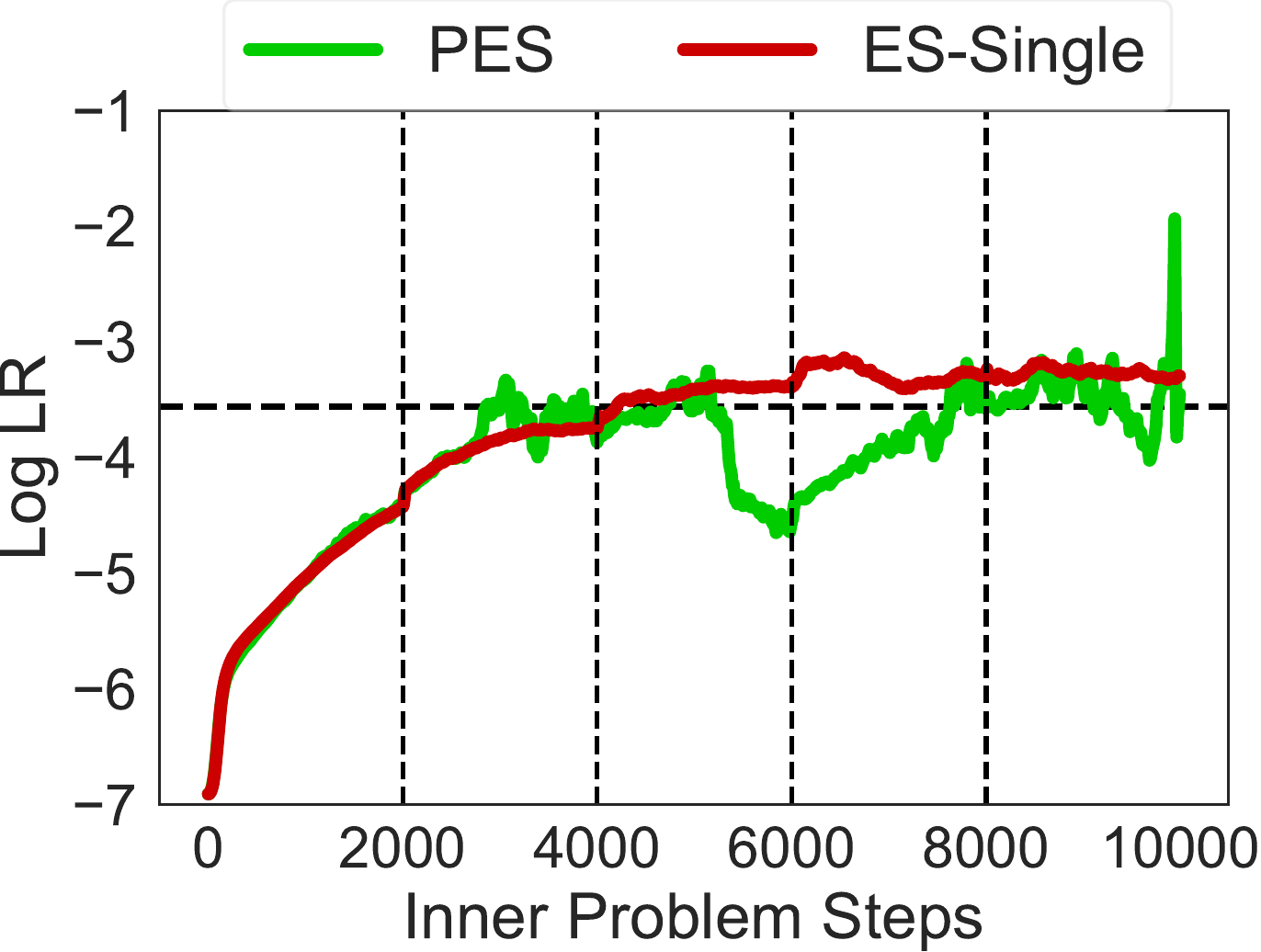}
        \caption{Adaptation of the log-learning rate.}
        \label{fig:global-lr-variance}
    \end{subfigure}
    \qquad
    \begin{subfigure}[t]{0.35\linewidth}
        \includegraphics[width=\linewidth]{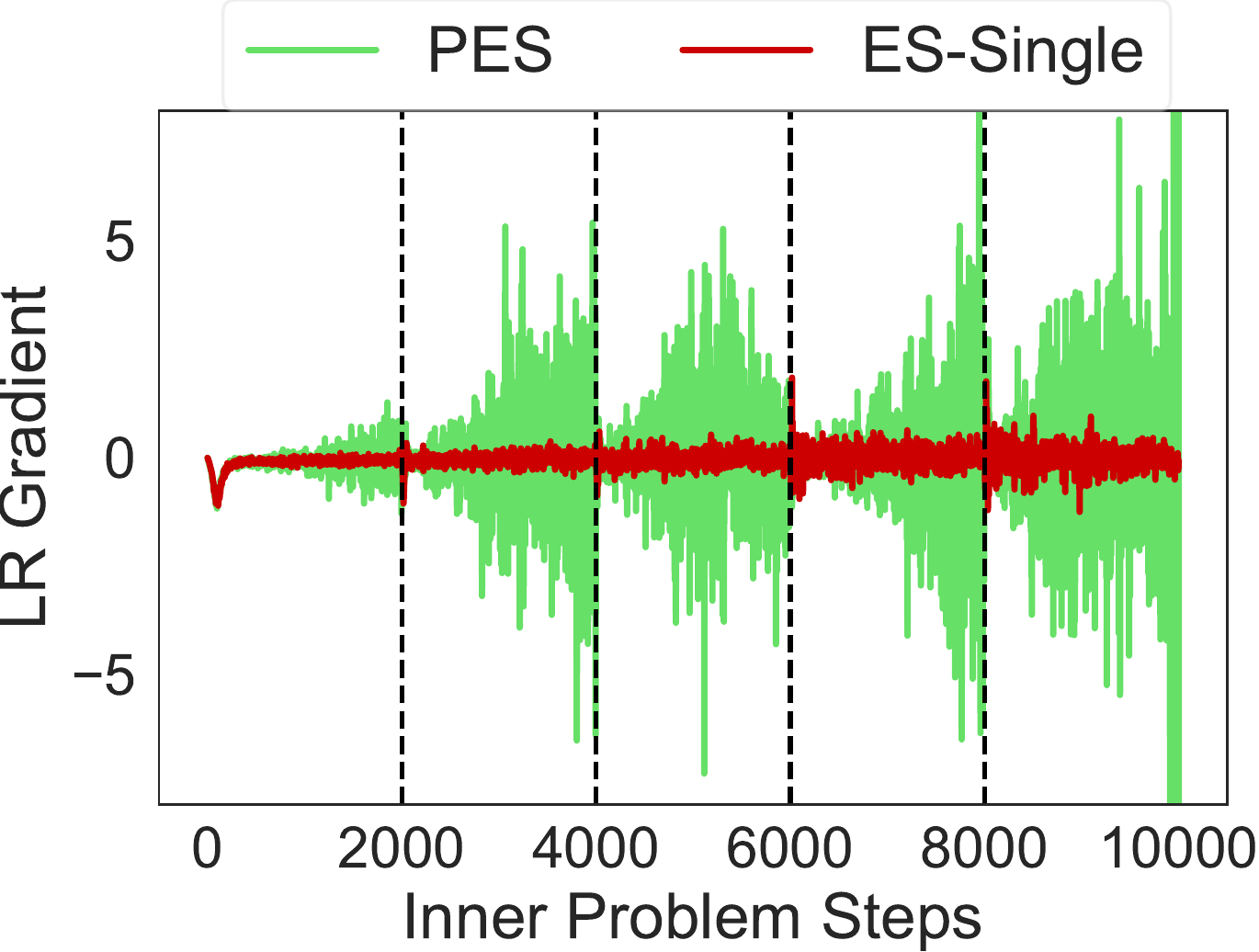}
        \caption{PES and ES-Single meta-gradients over the course of multiple inner problems.}
        \label{fig:global-lr-gradients}
    \end{subfigure}
    \vspace{-0.2cm}
    \caption{Comparing meta-optimization performance and meta-gradients from PES and ES-Single on a simple hyperparameter optimization task, where we aim to tune a global learning rate for training an MLP on MNIST.}
    \label{fig:pes-meta-gradient-comparison}
    % \vspace{-1.6cm}
\end{figure}

\clearpage

\subsection{Lockstep vs Breakstep Training}
\label{app:lockstep-breakstep}
\begin{wrapfigure}[14]{r}{0.4\linewidth}
    \vspace{-0.8cm}
    \centering
    \includegraphics[width=\linewidth]{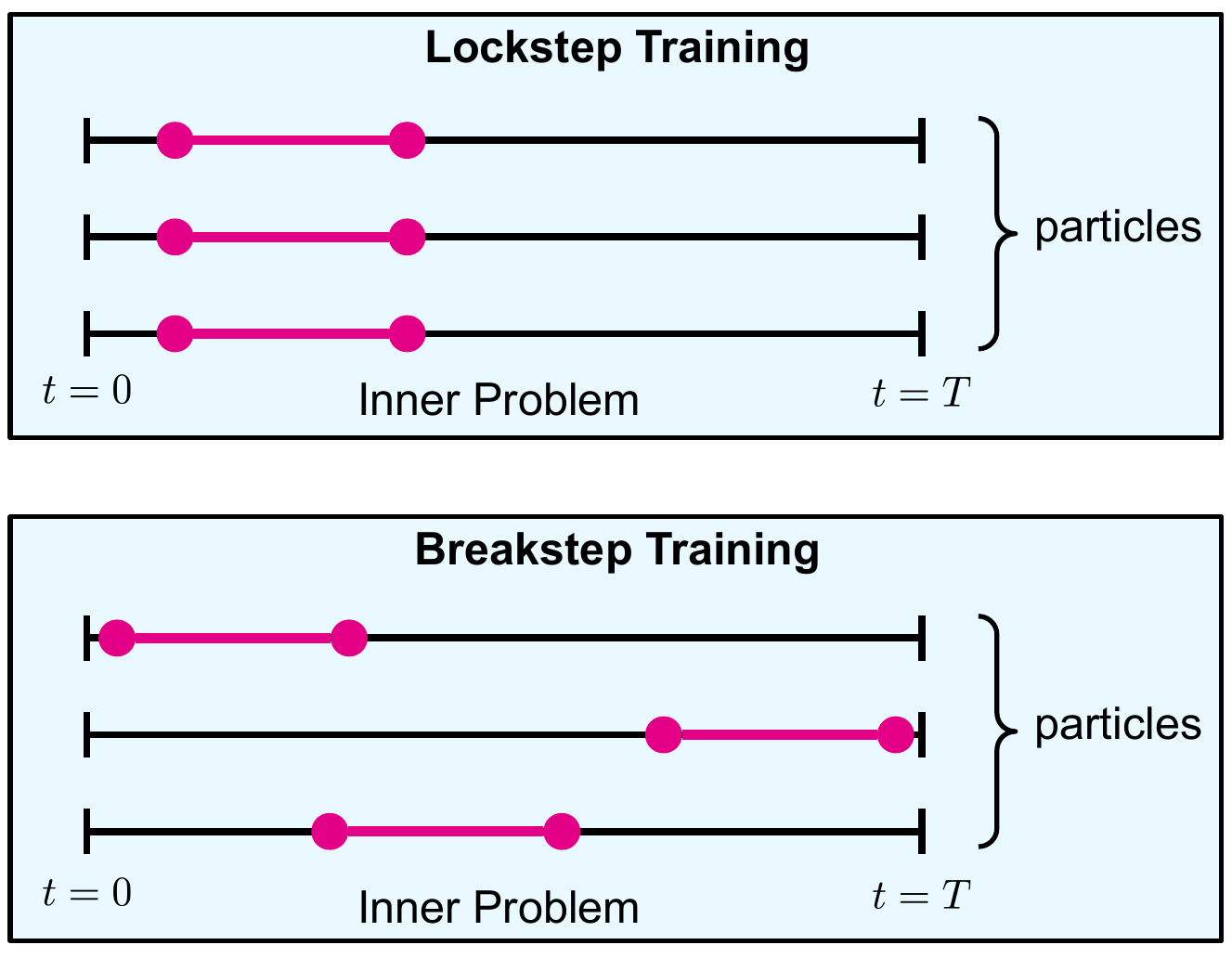}
    \vspace{-0.8cm}
    \caption{\small Conceptual illustration of lockstep and breakstep training, for methods that aggregate information across a collection of particles.}
    \label{fig:lockstep-breakstep}
\end{wrapfigure}
Vanilla truncated ES, PES, and ES-Single all aggregate information across a collection of particles.
In general, for such methods, there are two approaches for initializing and unrolling the particles: 1) \textit{lockstep training}, where all particles are initialized identically, at step $t=0$ of the inner problem, and progress through the inner optimization synchronously; or 2) \textit{breakstep} training, in which each particle pair (considering antithetic sampling) is initialized separately, potentially at an arbitrary starting step $t$ of the inner problem, and where particles progress through the inner problem asynchronously (e.g., one particle pair may be unrolled from $t=K$ to $t=2K$ while another pair is unrolled from $t=4K$ to $t=5K$). These two approaches are illustrated in Figure~\ref{fig:lockstep-breakstep}.
Often, both approaches work well; most of the experiments in this paper use lockstep training, except the learned optimizer experiment in Section~\ref{sec:lopt}, which uses breakstep training.

\subsection{Effect of Smoothing}
\begin{wrapfigure}[12]{r}{0.4\linewidth}
    \vspace{-0.6cm}
    \centering
    \includegraphics[width=\linewidth]{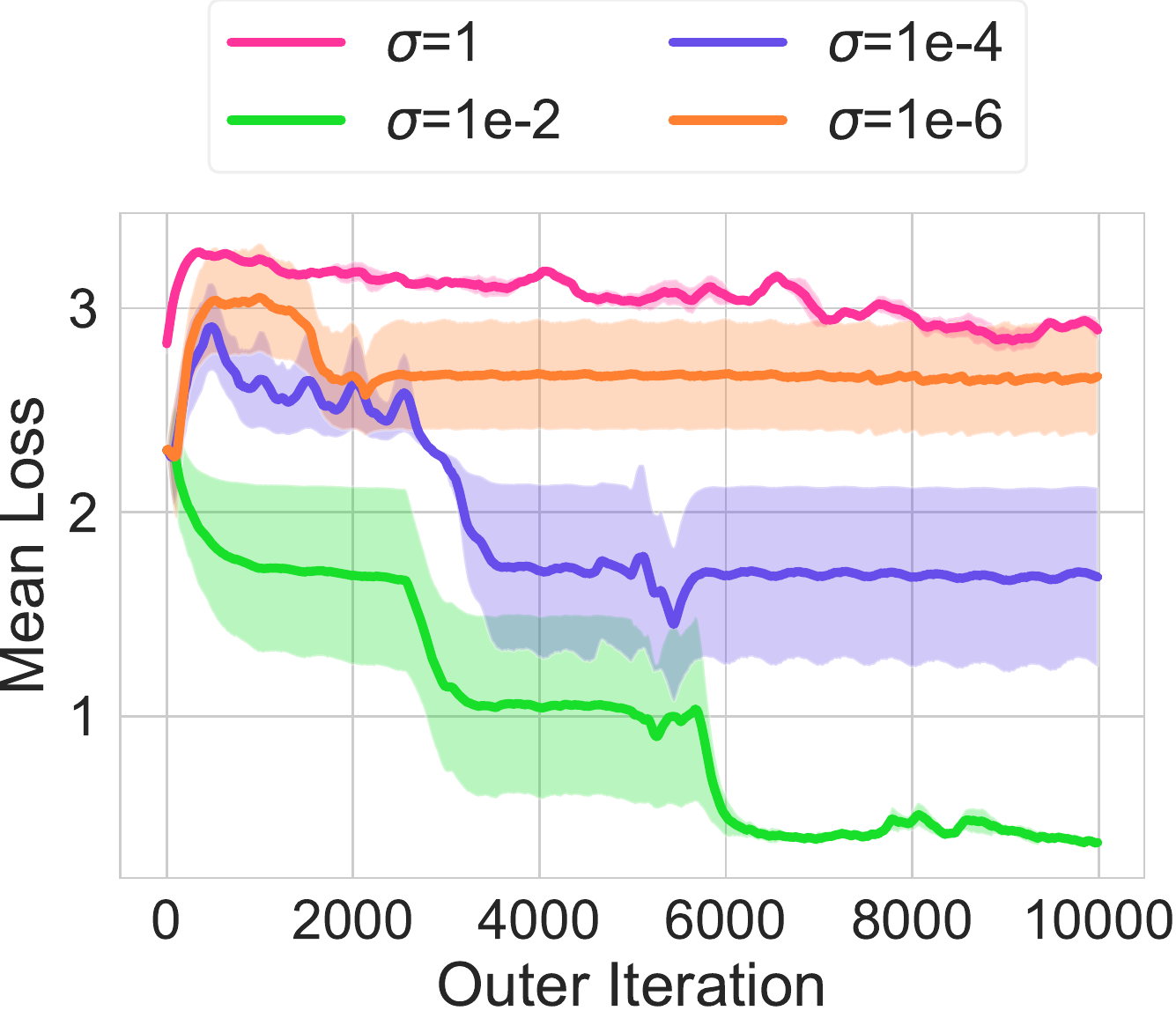}
    \vspace{-0.6cm}
    \caption{Ablation over the perturbation scale $\sigma$ used to train a learned optimizer, targeting an MLP on FashionMNIST.}
    \label{fig:perturbation-scale-lopt}
    % \vspace{-0.4cm}
\end{wrapfigure}
Meta-learning tasks such as hyperparameter optimization and learned optimizer training often lead to chaotic meta-loss landscapes, that are not amenable to optimization via gradient-based methods.
Here, we performed an ablation over the perturbation scale $\sigma$ (that controls the degree of smoothing) for ES-Single, used to optimize an MLP learned optimizer, similarly to Section 4.4 in the paper.
Small perturbation scales lead to behavior similar to gradient-based methods, which may get stuck in sub-optimal local minima in chaotic loss landscapes.
As shown in Figure~\ref{fig:perturbation-scale-lopt}, when the perturbation scale is too small, $\sigma =$ 1e-6, meta-optimization fails to make progress; in contrast, using an appropriate scale $\sigma =$ 1e-2 leads to stable convergence.

\vspace{2.5cm}

\subsection{Reinforcement Learning}
\begin{wrapfigure}[11]{r}{0.4\linewidth}
    \centering
    \vspace{-0.5cm}
    \includegraphics[width=\linewidth]{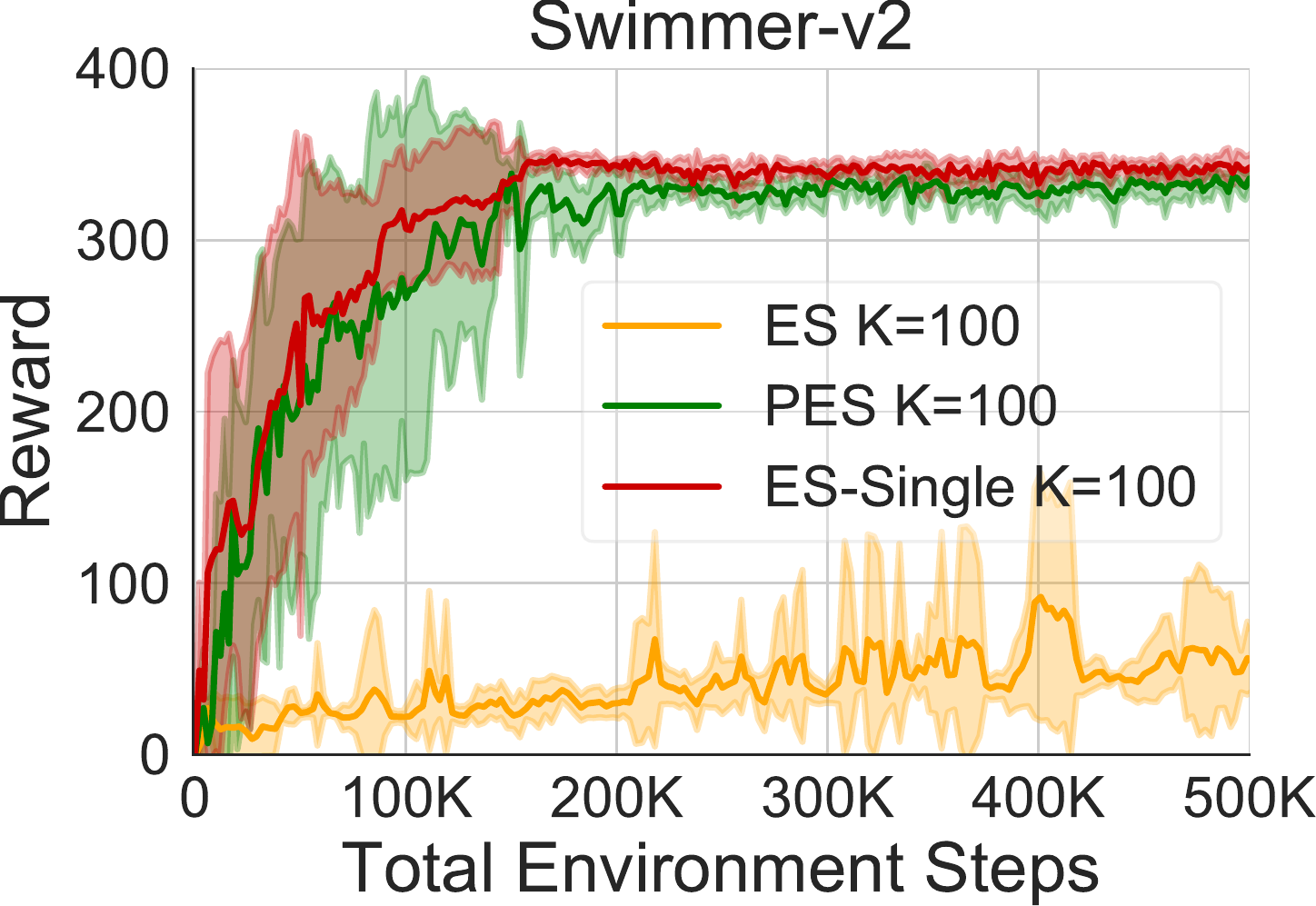}
    \vspace{-0.6cm}
    \caption{Learning a linear policy for the Swimmer MuJoCo environment using vanilla truncated ES, PES, and ES-Single. Here, $T=1000$ and $K=100$. Shaded regions denote standard deviations over 6 random seeds.}
    \label{fig:swimmer-es-single}
\end{wrapfigure}
While investigating truncated ES-based methods for RL is an area for future work, we provide a proof-of-concept experiment here.
We ran ES-Single on the continuous control task used in PES, which trains linear policy on the Swimmer MuJoCo environment.
We compared vanilla truncated ES, PES, and ES-Single, all using partial unrolls of length $K=100$.
The results are shown in Figure~\ref{fig:swimmer-es-single}, where the shaded regions denote the standard deviations over 6 random seeds.
We found that ES-Single slightly outperformed PES, with smaller standard deviation, and more stable convergence to the optimal episode return.

\vspace{2cm}

\subsection{CIFAR-10 Experiments}
\begin{figure}[t]
    \centering
    \begin{subfigure}[t]{0.44\linewidth}
        \includegraphics[width=\linewidth]{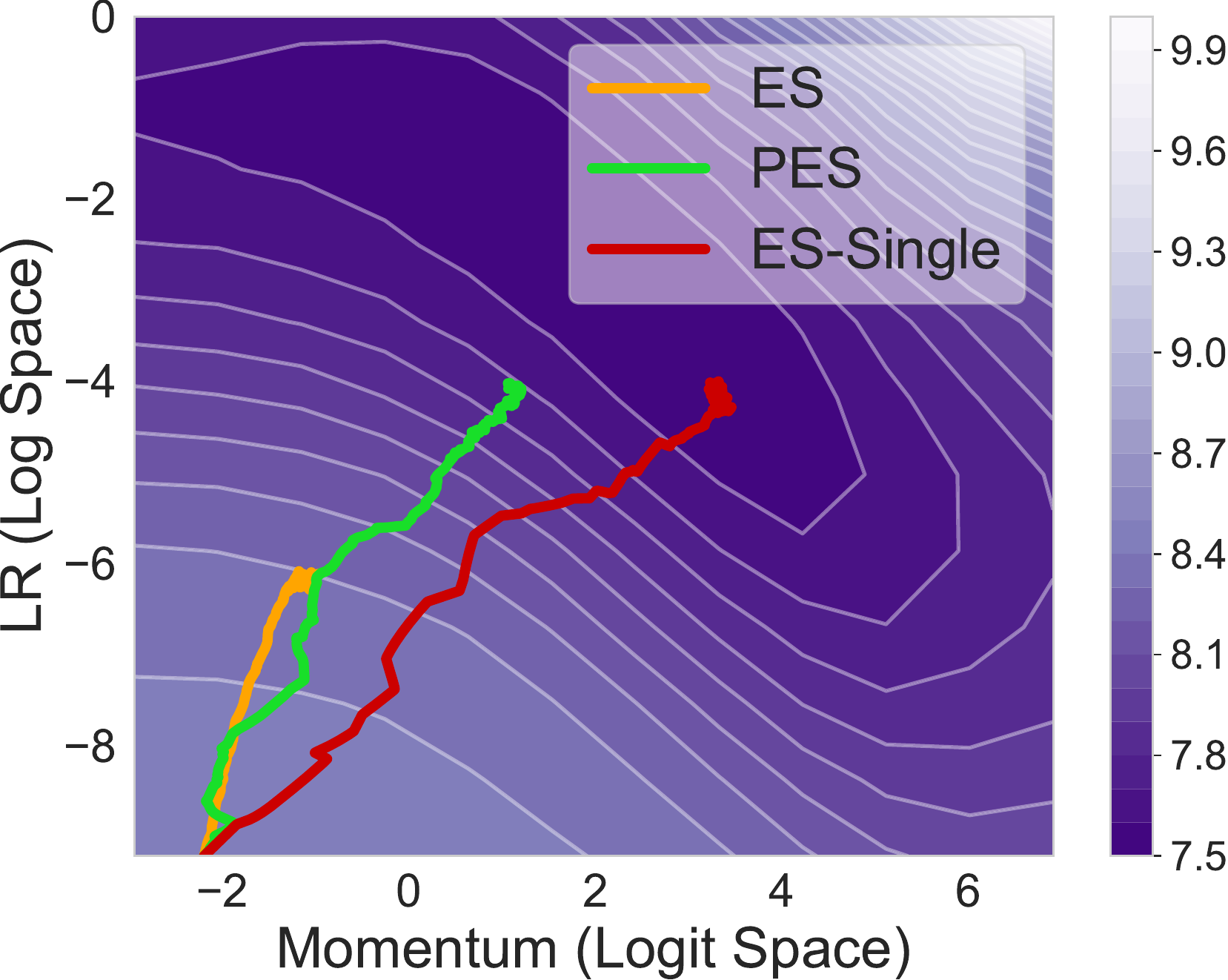}
        \caption{Meta-optimization trajectories.}
        \label{fig:}
    \end{subfigure}
    \qquad
    \begin{subfigure}[t]{0.46\linewidth}
        \includegraphics[width=\linewidth]{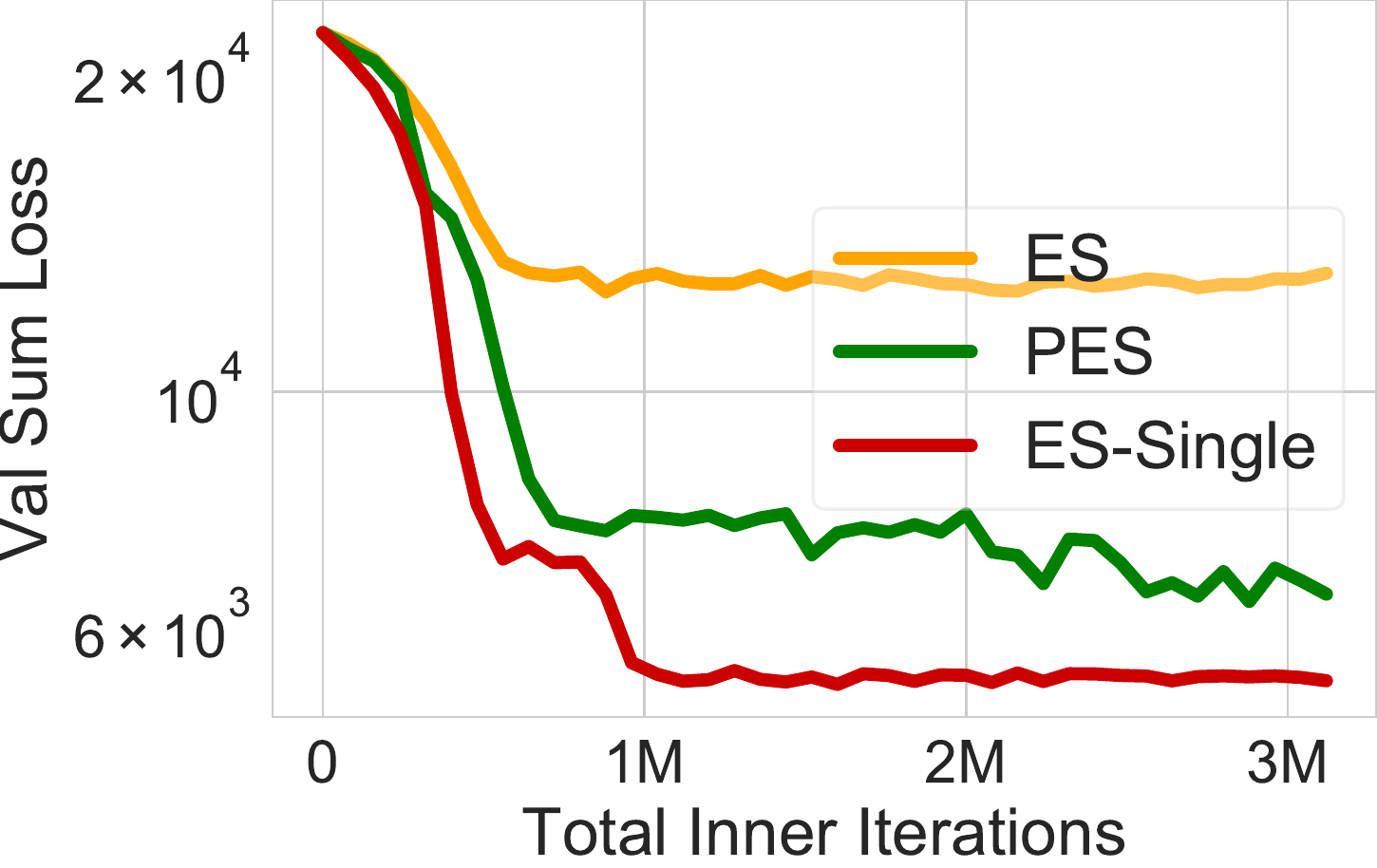}
        \caption{Meta-objective values.}
        \label{fig:}
    \end{subfigure}
    % \vspace{-0.2cm}
    \caption{Tuning the learning rate and momentum coefficient for SGDm, used to optimize a ResNet on CIFAR-10. Here, $T=5000$ and $K=20$.}
    \label{fig:es-momentum}
    \vspace{-0.4cm}
\end{figure}
\begin{wrapfigure}[15]{r}{0.4\linewidth}
    \centering
    \vspace{-0.6cm}
    \includegraphics[width=\linewidth]{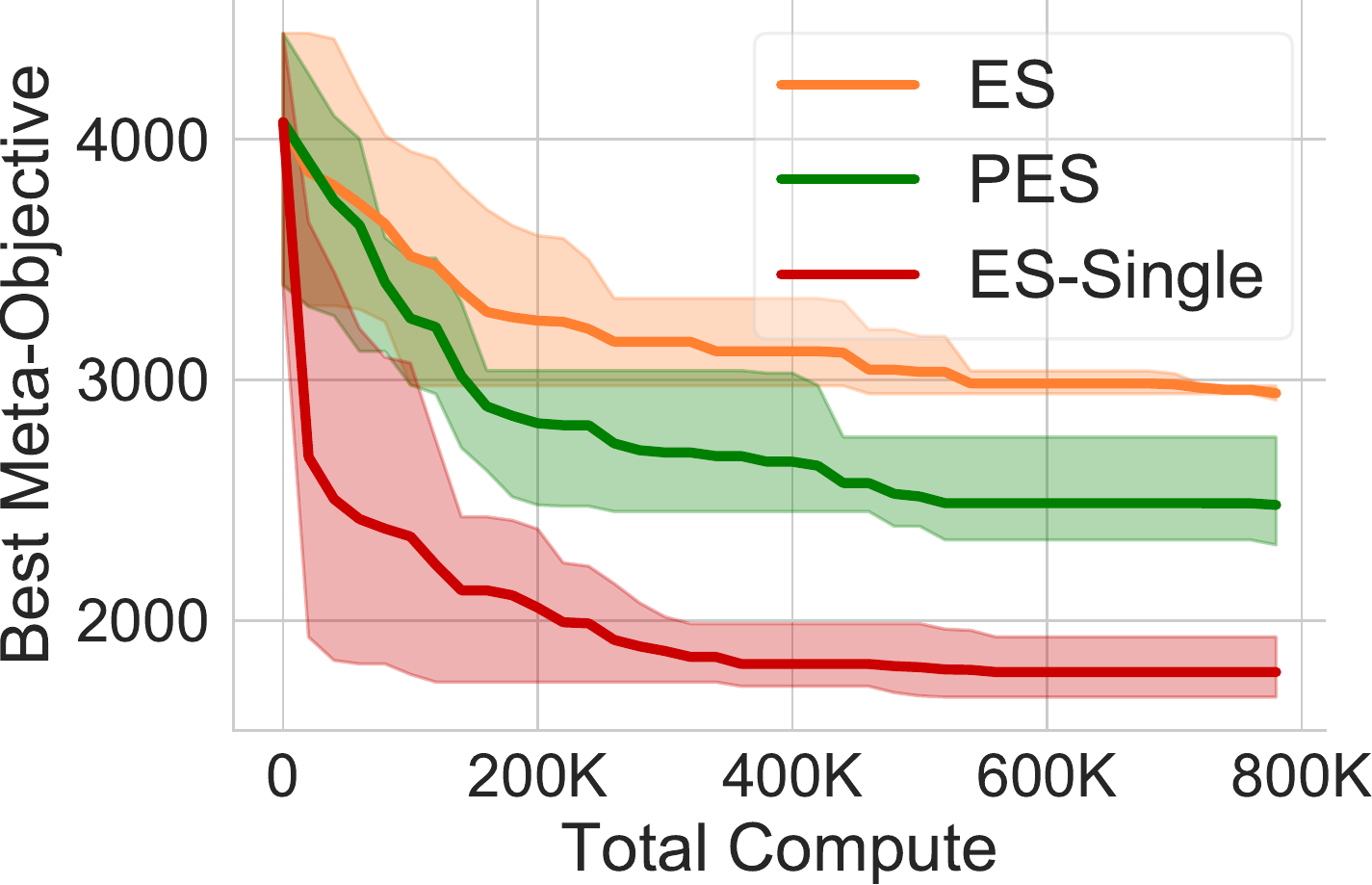}
    \vspace{-0.6cm}
    \caption{Tuning (continuous) per-parameter block learning rates and momentum coefficients, as well as the (discrete) number of channels per convolutional layer in a ResNet trained on CIFAR-10. The inner problem has length $T=2000$ and we use truncations of length $K=10$. Shaded regions denote the min/max performance over 3 random seeds.}
    \label{fig:cifar10-tuning-many-params}
    % \vspace{-0.3cm}
\end{wrapfigure}
Here, we applied ES-Single to two hyperparameter optimization tasks in which we train a ResNet on CIFAR-10.
We used a 1.6M parameter Myrtle.ai ResNet architecture.
In both cases, the meta-objective is the sum of validation losses over the inner problem.
First, we tuned the global learning rate and momentum coefficient for SGDm.
Here, the inner problem had length $T=5000$, and we used truncations of length $K=20$ for all approaches.
As shown in Figure~\ref{fig:es-momentum}, ES-Single converged to the optimal region substantially faster than PES.
Second, we tuned 24 continuous and discrete hyperparameters simultaneously.
In particular, we tuned per-parameter-block learning rates and momentum coefficients, as well as the number of channels per convolutional layer.
Here, the inner problem had length $T=2000$, and we used truncations of length $K=10$ for all approaches.
The results are shown in Figure~\ref{fig:cifar10-tuning-many-params}; we found that ES-Single substantially outperformed ES and PES, achieving lower meta-objective values using less total compute.

\subsection{Comparison to DODGE}
\label{sec:dodge-comparison}

Here, we provide an extended discussion on the similarities and differences between ES-Single and DODGE~\cite{silver2021learning}. Like ES-Single, DODGE is a method for computing gradients in unrolled computation graphs.
DODGE is an approximation of RTRL: rather than maintaining the expensive recurrent Jacobian matrix $\frac{d \bolds_t}{d \boldtheta}$ over time, it takes the directional derivative $(\nabla_{\boldtheta} L(\boldtheta) \cdot \boldu) \boldu$ along a specific vector $\boldu$. This reduces the space complexity of the algorithm, because it only needs to store and propagate a vector $\frac{d \bolds_t}{d \boldtheta} \boldu$:

\begin{align}
    \frac{d L(\boldtheta)}{d \boldtheta} \boldu
    &=
    \sum_{t=1}^T \frac{d L_t(\bolds_t, \boldtheta)}{d \boldtheta} \boldu \\
    \frac{d L_t(\bolds_t, \boldtheta)}{d \boldtheta} \boldu
    &=
    \frac{\partial L_t(\bolds_t, \boldtheta)}{\partial \boldtheta} \boldu + \frac{\partial L_t(\bolds_t, \boldtheta)}{\partial \bolds_t} \underbrace{\frac{d \bolds_t}{d \boldtheta} \boldu}_{\boldc_t} \\
    \boldc_t = \frac{d \bolds_t}{d \boldtheta} \boldu
    &=
    \frac{d f(\bolds_{t-1}, \boldtheta)}{d \boldtheta} \boldu
    =
    \frac{\partial f(\bolds_{t-1}, \boldtheta)}{\partial \boldtheta} \boldu + \frac{\partial f(\bolds_{t-1}, \boldtheta)}{\partial \bolds_{t-1}} \underbrace{\frac{d \bolds_{t-1}}{d \boldtheta} \boldu}_{\boldc_{t-1}}
\end{align}
Here, $\boldc_t$ is a \textit{carry term} that propagates information over the course of a full inner problem.
While DODGE and ES-Single quite distinct—ES-Single is gradient-free while DODGE is gradient-based—they share some conceptual similarities. In particular, both perform meta-optimization within a subspace spanned by a set of direction vectors.
In DODGE, a few possibilities were given for determining these directions, including drawing samples from an isotropic Gaussian similarly to ES-Single.
A critical difference is that ES-Single smooths the outer loss landscape, allowing for optimization over chaotic surfaces that arise in meta-optimization.

As shown in Figure~\ref{fig:dodge}, when applying DODGE to a 2D meta-optimization task, it gets stuck when it reaches a chaotic part of the landscape, similarly to the other gradient-based methods (TBPTT, RTRL, and UORO~\cite{tallec2017unbiased}).

\begin{figure}[t]
    \centering
    \begin{subfigure}[t]{0.4\linewidth}
        \includegraphics[width=\linewidth]{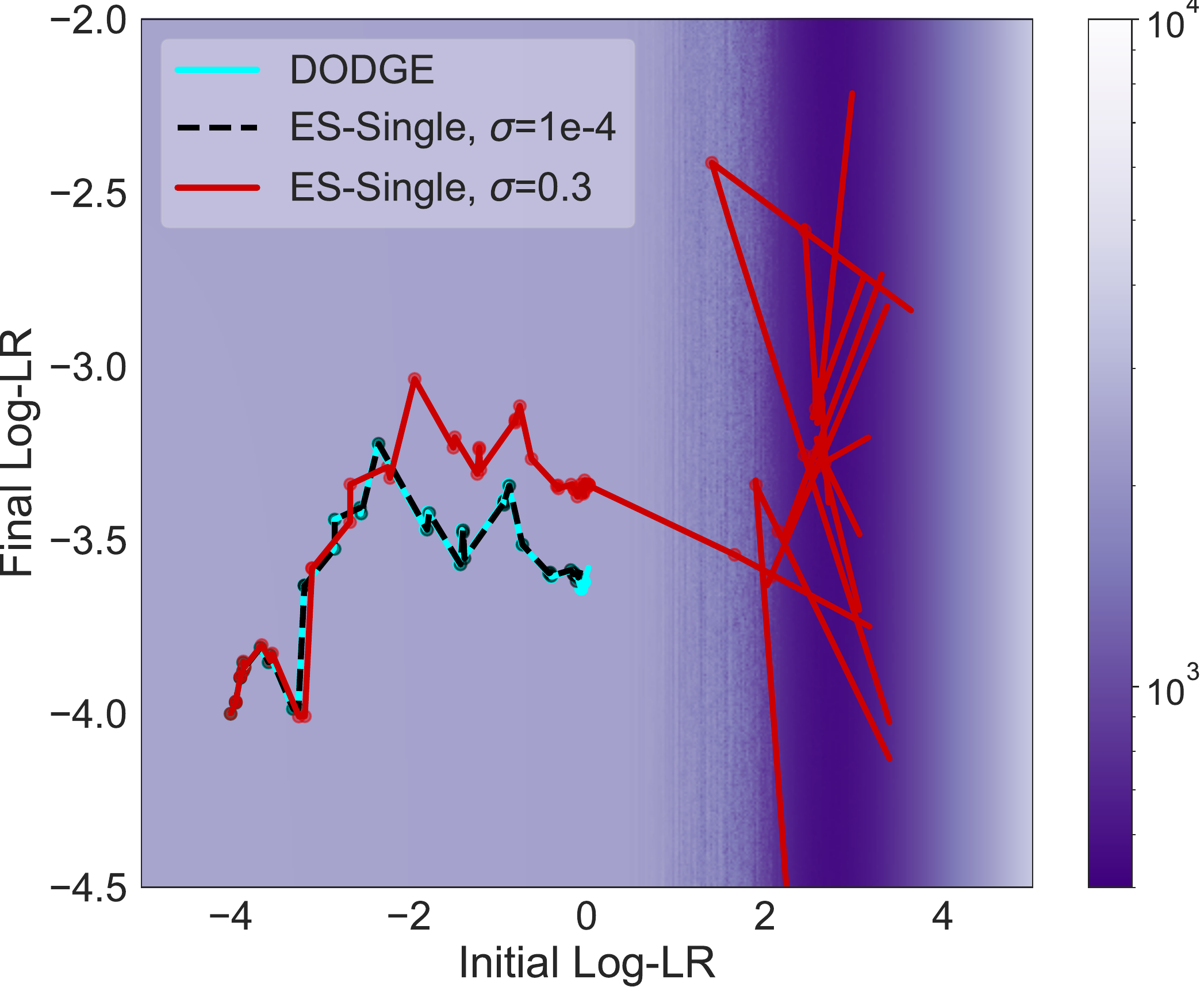}
        \caption{Using the same sequence of direction vectors for ES-Single and DODGE. With a small perturbation scale, ES-Single becomes nearly identical to DODGE, while with a large perturbation scale, it traverses a smoothed landscape.}
        \label{fig:}
    \end{subfigure}
    \qquad
    \begin{subfigure}[t]{0.4\linewidth}
        \includegraphics[width=\linewidth]{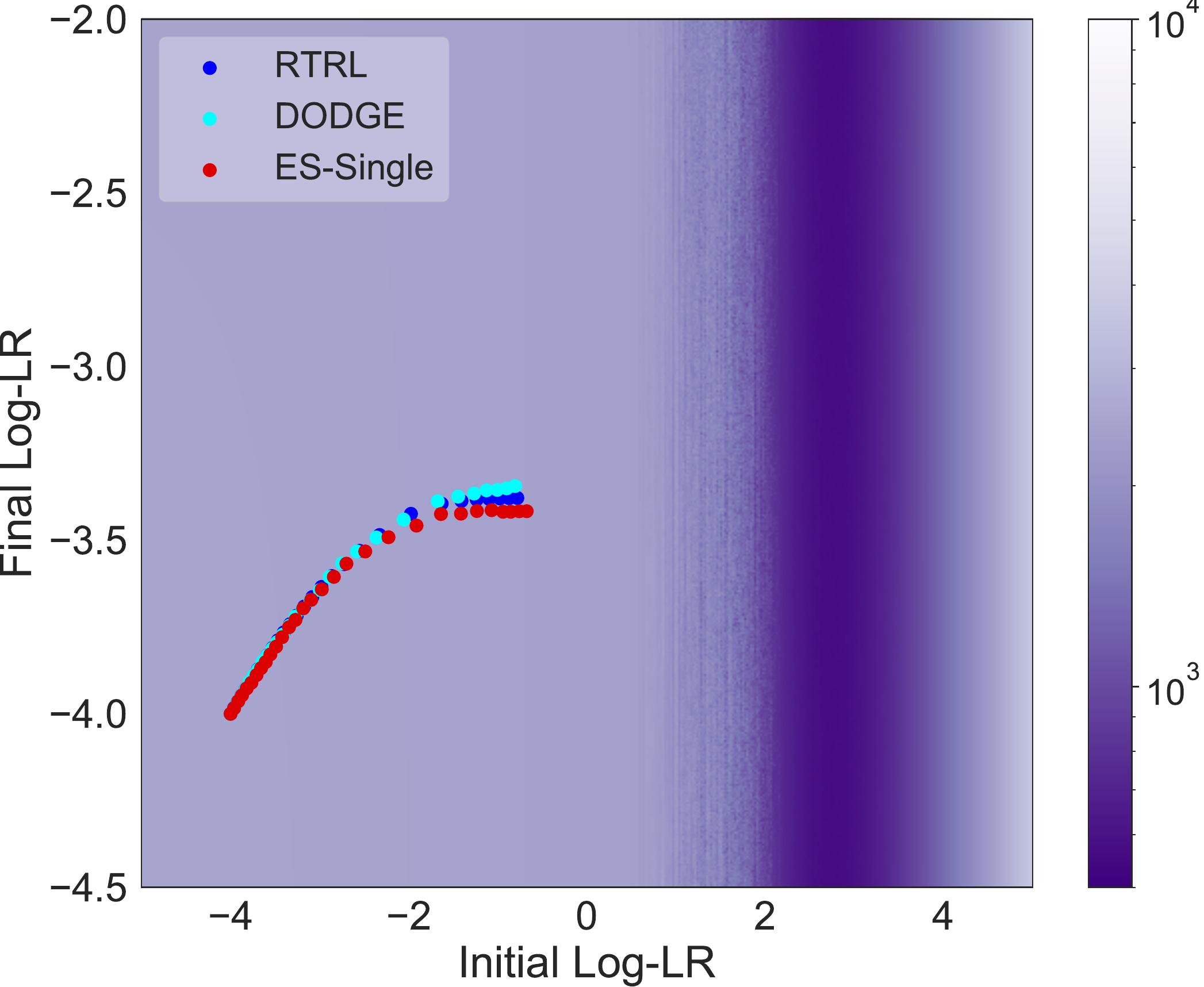}
        \caption{When using 1000 direction vectors, both DODGE and ES-Single become approximately equivalent to RTRL. Note that here we use a small perturbation scale $\sigma$=1e-4 for ES-Single.}
        \label{fig:}
    \end{subfigure}
    \caption{Comparing ES-Single and DODGE on the toy 2D regression task from Figure~\ref{fig:toy-regression}. (a) We use the same sequence of random directions for ES-Single and DODGE, where a single direction is sampled at the start of each inner problem and kept fixed over all partial unrolls. Both methods use truncations of length $K=1$, with total inner problem length $T=100$. Over the course of an inner problem, DODGE and ES-Single update the outer parameters in a subspace spanned by the direction vector; this leads to the line segments in Fig. (a), each of which shows the progress made during one inner problem. At the end of an inner problem, a new direction is sampled, leading to the piecewise linear structure. As the perturbation scale $\sigma$ goes to 0, ES-Single becomes nearly identical to DODGE, and is not able to traverse the chaotic regions of the meta-loss landscape. Increasing the perturbation scale allows ES-Single to cross the chaos and reach the optimal meta-objective value. In Fig. (b), we show that increasing the number of direction vectors used DODGE and ES-Single brings both methods very close to exact RTRL. Darker regions represent lower meta-objective values.}
    \label{fig:dodge}
\end{figure}

\subsection{Comparison to Hypergradient Descent.}
Some algorithms have been proposed to tune optimization hyperparameters (such as learning rates) online during a single training run (e.g., one inner problem as opposed to several), in particular hypergradient descent (HD)~\cite{baydin2017online} and ``Gradient Descent: The Ultimate Optimizer'' (GDTUO)~\cite{chandra2022gradient}. Both of these adapt optimization hyperparameters based on a 1-step lookahead meta-objective. However, as shown in~\cite{wu2018understanding}, backpropagation through a 1-step unroll may suffer from truncation bias. HD and GDTUO are conceptually similar to our vanilla truncated ES baseline, as they aim to minimize the loss after taking $K$ gradient steps (with $K=1$); thus, they can be interpreted as gradient-based analogues of truncated ES. In contrast, ES-Single and PES yield unbiased gradient estimates that do not suffer from truncation bias.
Here, we used the Github repository of~\citet{chandra2022gradient} (\href{https://github.com/kach/gradient-descent-the-ultimate-optimizer}{https://github.com/kach/gradient-descent-the-ultimate-optimizer}), and applied their method to our task from Section~\ref{sec:hyperparameter-optimization}: tuning the learning rate and decay factor used to train an MLP.

\begin{figure}[t]
    \centering
    \begin{subfigure}[t]{0.46\linewidth}
        \includegraphics[width=0.8\linewidth]{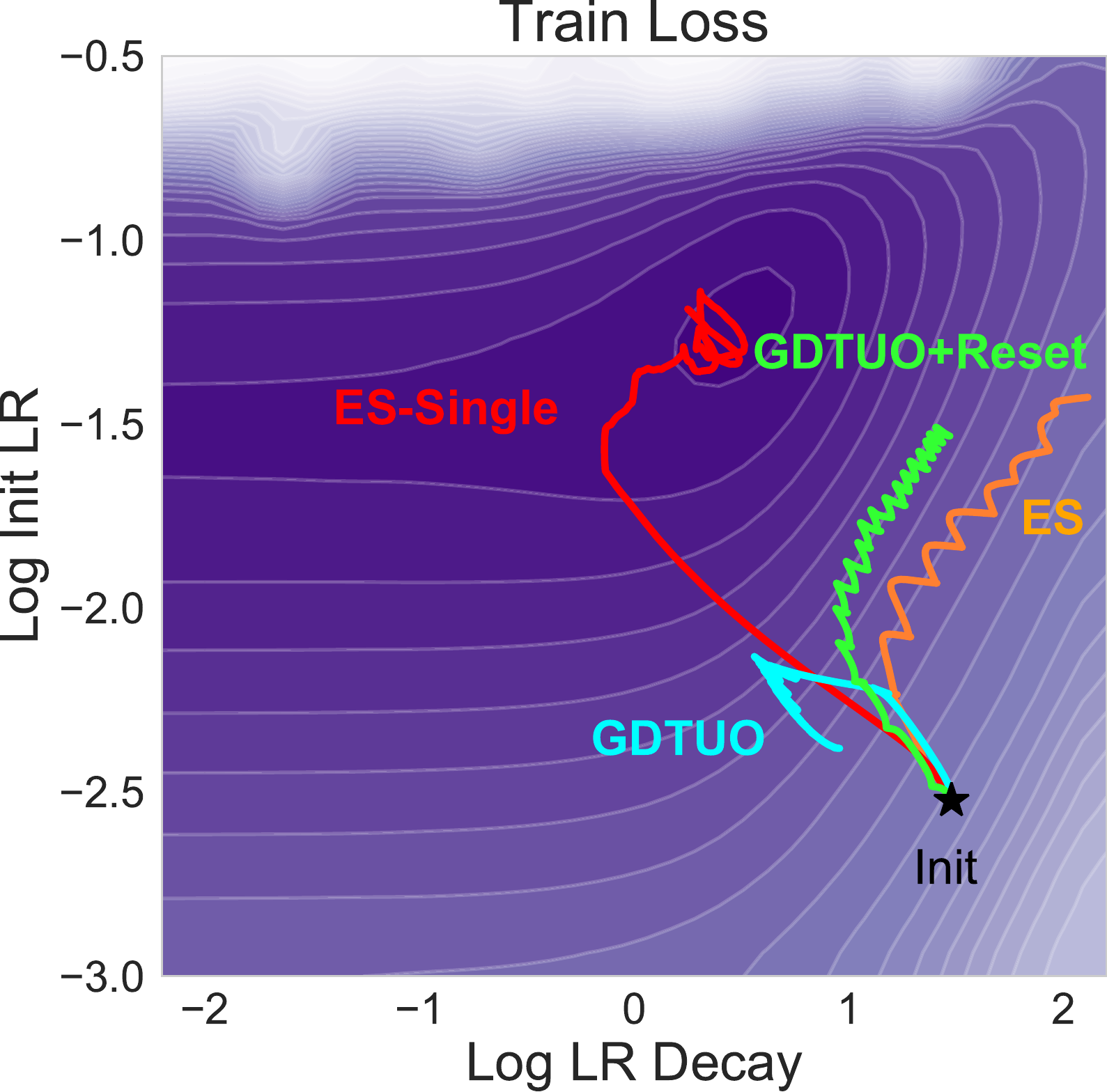}
        \caption{Here, we show the trajectories of vanilla truncated ES and ES-Single, alongside the trajectories of two variants of GDTUO. Darker colors denote lower (better) loss values.}
        \label{fig:hd-comparison-meta-opt}
    \end{subfigure}
    \hfill
    \begin{subfigure}[t]{0.49\linewidth}
        \includegraphics[width=0.85\linewidth]{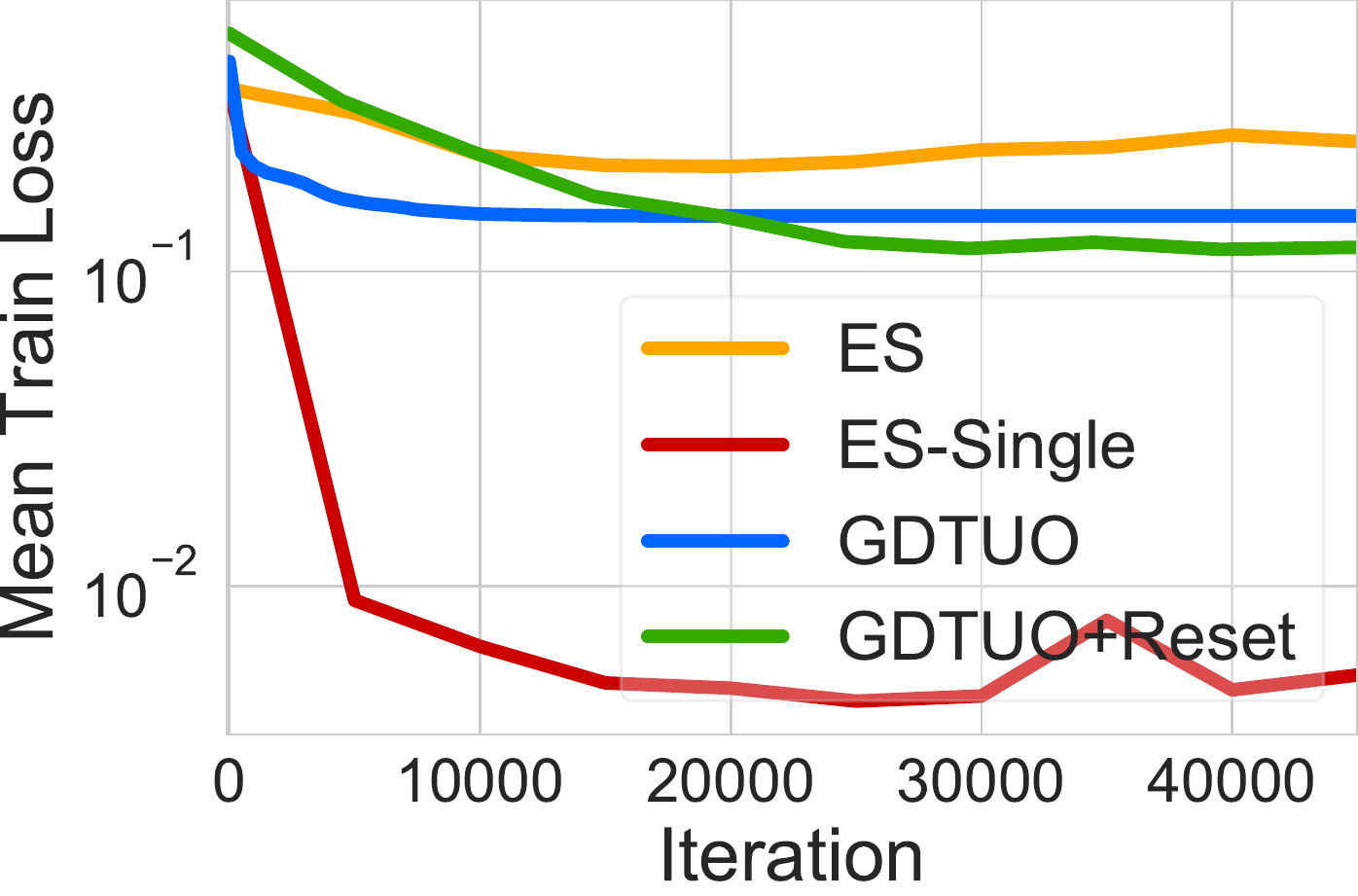}
    \caption{Mean training loss values using the same setup as Section~\ref{sec:hyperparameter-optimization}, comparing ES, ES-Single, GDTUO, and GDTUO with inner problem resetting.}
    \label{fig:hd-comparison-obj}
    \end{subfigure}
    \caption{Tuning the learning rate and decay factor for SGDm, used to optimize an MLP on MNIST, with the same setup as Section~\ref{sec:hyperparameter-optimization}. We compare the meta-optimization trajectories and training losses that result from using vanilla truncated ES, ES-Single, and two variants of GDTUO~\cite{chandra2022gradient} that differ with respect to whether the inner problem is reset after $T$ steps.}
\end{figure}

We considered two scenarios for GDTUO: 1) never resetting the inner problem (e.g., the online learning setting they use); and 2) resetting the inner problem every $T$ iterations, while continuing to optimize the outer parameters, which mimics our truncated ES setting. Figures~\ref{fig:hd-comparison-meta-opt} and~\ref{fig:hd-comparison-obj} show that GDTUO behaves similarly to truncated ES, both in terms of the meta-optimization trajectory and the training loss achieved. ES-Single outperforms GDTUO on this task.

\section{Proofs}
\label{app:proofs}

\subsection{Proof of Unbiasedness}
\label{app:unbiased-proof}

\begin{tcolorbox}[colback=tabblue!15,boxrule=0pt,colframe=white,coltext=black,arc=2pt,outer arc=0pt,valign=center]%,valign=center]
\begin{proposition}[ES-Single is unbiased]
Assume that $L(\boldtheta)$ is quadratic and $\nabla_{\boldtheta} L(\boldtheta)$ exists. Then, the ES-Single gradient estimator with antithetic sampling is unbiased, that is, $\text{bias}(\hat{\boldg}^{\text{ES-Single}}) = \mathbb{E}_{\boldepsilon} \left[ \hat{\boldg}^{\text{ES-Single}} \right] - \nabla_{\boldtheta} L(\boldtheta) = 0$.
\end{proposition}
\end{tcolorbox}
\begin{proof}
By assumption, $L(\boldtheta)$ is quadratic, and thus is equivalent to its second-order Taylor series approximation:
\begin{align}
    L(\boldtheta + \boldepsilon)
    =
    L(\boldtheta) + \boldepsilon^\top \nabla_{\boldtheta} L(\boldtheta) + \frac{1}{2} \boldepsilon^\top \nabla_{\boldtheta}^2 L(\boldtheta) \boldepsilon
\end{align}
The antithetic gradient estimator is $\mathbb{E}_{\boldepsilon} \left[ \boldepsilon (L(\boldtheta + \boldepsilon) - L(\boldtheta - \boldepsilon)) \right]$.
We can simplify this expression by noting that:
\begin{align}
    \boldepsilon \left( L(\boldtheta + \boldepsilon) - L(\boldtheta - \boldepsilon) \right)
    &=
    \boldepsilon \left[ L(\boldtheta) + \boldepsilon^\top \nabla_{\boldtheta} L(\boldtheta) + \frac{1}{2} \boldepsilon^\top \nabla_{\boldtheta}^2 L(\boldtheta) \boldepsilon - L(\boldtheta) + \boldepsilon^\top \nabla_{\boldtheta} L(\boldtheta) - \frac{1}{2} \boldepsilon^\top \nabla_{\boldtheta}^2 L(\boldtheta) \boldepsilon \right] \\
    &=
    \boldepsilon \boldepsilon^\top \nabla_{\boldtheta} L(\boldtheta)
\end{align}
Thus, the ES-Single gradient is the following Monte Carlo estimate:
\begin{align}
    \gessingle = \frac{1}{\sigma^2 N} \sum_{i=1}^N \boldepsilon_i \boldepsilon_i^\top \nabla_{\boldtheta} L(\boldtheta)
\end{align}
Taking the expectation of this expression, we have:
\begin{align}
    \mathbb{E}_{\boldepsilon} \left[ \gessingle \right]
    =
    \frac{1}{\sigma^2 N} \sum_{i=1}^N \underbrace{\mathbb{E}_{\boldepsilon} \left[ \boldepsilon_i \boldepsilon_i^\top \right]}_{= \sigma^2 \boldI} \nabla_{\boldtheta} L(\boldtheta)
    =
    \frac{1}{N} \sum_{i=1}^N \nabla{\boldtheta} L(\boldtheta)
    =
    \nabla_{\boldtheta} L(\boldtheta)
\end{align}
Thus, $\text{bias}(\hat{\boldg}^{\text{ES-Single}}) = \mathbb{E}_{\boldepsilon} \left[ \hat{\boldg}^{\text{ES-Single}} \right] - \nabla_{\boldtheta} L(\boldtheta) = 0$.
\end{proof}

\subsection{Variance}
\label{app:variance-proof}

\begin{tcolorbox}[colback=tabblue!15,boxrule=0pt,colframe=white,coltext=black,arc=2pt,outer arc=0pt,valign=center]%,valign=center]
\begin{proposition}[ES-Single Variance]
The total variance of ES-Single using antithetic sampling is $\tr(\Var(\hat{\boldg}^{\text{ES-Single}})) = (P + 1) \| \nabla_{\boldtheta} L(\boldtheta) \|^2$, where $P$ is the dimensionality of $\boldtheta$.
\end{proposition}
\end{tcolorbox}
\begin{proof}
We measure the total variance of the ES-Single estimator, defined as:
\begin{align}
    \tr(\Var(\hatg))
    =
    \tr\left( \mathbb{E}\left[ \hatg \hatg^\top \right] - \mathbb{E}\left[ \hatg \right] \mathbb{E}\left[ \hatg \right]^\top \right)
    =
    \mathbb{E}\left[ \hatg^\top \hatg \right] - \mathbb{E}\left[ \hatg \right]^\top \mathbb{E} \left[ \hatg \right] \label{eq:total-variance}
\end{align}
We assume that the loss $L$ is quadratic, and that we use antithetic samples to estimate the gradient.
Here, we consider a single particle pair for simplicity, $N=1$, such that the estimator can be written as $\hatg = \frac{1}{\sigma^2} \boldepsilon \boldepsilon^\top \nabla_{\boldtheta} L(\boldtheta)$.
Because $\hatg$ is an unbiased estimator, its expectation is equal to the true gradient, and thus the second term in Eq.~\ref{eq:total-variance} is $\mathbb{E}\left[ \hatg \right]^\top \mathbb{E} \left[ \hatg \right] = \| \nabla_{\boldtheta} L(\boldtheta) \|^2$.
For the first term, we have:
\begin{align}
    \mathbb{E} \left[ \hatg^\top \hatg \right]
    &=
    \frac{1}{\sigma^4} \mathbb{E}_{\boldepsilon}\left[ \nabla_{\boldtheta} L(\boldtheta)^\top \boldepsilon \boldepsilon^\top \boldepsilon \boldepsilon \nabla_{\boldtheta} L(\boldtheta) \right] \\
    &=
    \frac{1}{\sigma^4} \nabla_{\boldtheta} L(\boldtheta)^\top \mathbb{E}_{\boldepsilon} \left[ \boldepsilon \boldepsilon^\top \boldepsilon \boldepsilon^\top \right] \nabla_{\boldtheta} L(\boldtheta) \label{eq:second-term}
\end{align}
By Isserlis' theorem (see~\citet{maheswaranathan2019guided}), we have $\mathbb{E}_{\boldepsilon} \left[ \boldepsilon \boldepsilon^\top \boldepsilon \boldepsilon^\top \right] = \tr(\Sigma) \Sigma + 2 \Sigma^2$, where $\Sigma$ is the covariance of the perturbation distribution.
Because our perturbations are sampled from an isotropic Gaussian, $\boldepsilon \sim \mathcal{N}(\boldzero, \sigma^2 \boldI)$, we have $\Sigma = \sigma^2 \boldI$, and thus:
\begin{align}
    \mathbb{E}_{\boldepsilon} \left[ \boldepsilon \boldepsilon^\top \boldepsilon \boldepsilon^\top \right]
    &=
    \tr(\Sigma) \Sigma + 2 \Sigma^2 \\
    &=
    \tr(\sigma^2 \boldI) \sigma^2 \boldI + 2 (\sigma^2 \boldI)^2 \\
    &=
    P \sigma^4 \boldI + 2 \sigma^4 \boldI \\
    &=
    (P + 2) \sigma^4 \boldI
\end{align}
where $P$ is the dimensionality of $\boldtheta$.
Plugging this into Eq.~\ref{eq:second-term}, we have:
\begin{align}
    \mathbb{E} \left[ \hatg^\top \hatg \right]
    &=
    \frac{1}{\sigma^4} \nabla_{\boldtheta} L(\boldtheta)^\top \mathbb{E}_{\boldepsilon} \left[ \boldepsilon \boldepsilon^\top \boldepsilon \boldepsilon^\top \right] \nabla_{\boldtheta} L(\boldtheta) \\
    &=
    \frac{1}{\sigma^4} \nabla_{\boldtheta} L(\boldtheta)^\top \left[ (P+2) \sigma^4 \boldI \right] \nabla_{\boldtheta} L(\boldtheta) \\
    &=
    (P + 2) \| \nabla_{\boldtheta} L(\boldtheta) \|^2
\end{align}
Thus, the total variance of the ES-Single estimator is:
\begin{align}
    \tr(\Var(\hatg))
    =
    (P + 2) \| \nabla{\boldtheta} L(\boldtheta) \|^2 - \| \nabla_{\boldtheta} L(\boldtheta) \|^2
    =
    (P + 1) \| \nabla_{\boldtheta} L(\boldtheta) \|^2
\end{align}
From this, we see that the variance increases linearly in the dimensionality of $\boldtheta$, but does not depend on the number of partial unrolls per inner problem.
\end{proof}

\clearpage

\section{Stochastic Computation Graphs}
\label{app:stochasticcomp}

We have an input node $\theta$ that gives rise to a sampled variable $\tilde{\theta}$ which is re-applied over all time steps of the inner problem.
Each state $\bolds_t$ depends deterministically on the sampled parameter $\tilde{\theta}$ and the previous state $\bolds_{t-1}$.
The losses at each timestep, $L_t$, are the cost nodes, and the objective is to minimize $L = \sum_{t=1}^T L_t$.
We leverage Theorem 1 from~\citet{schulman2015gradient}, which gives a general expression to compute the gradient of the sum of cost nodes in such a stochastic computation graph:
\begin{align}
    \frac{\partial}{\partial \theta} \mathbb{E} \left[ \sum_{c \in \mathcal{C}} c \right]
    =
    \mathbb{E} \left[
    \sum_{w \in \mathcal{S}, \theta \prec^D w} \left( \frac{\partial}{\partial \theta} \log p(w \mid \text{DEPS}_w) \right) \hat{Q}_w
    +
    \underbrace{\sum_{c \in \mathcal{C}, \theta \prec^D c} \frac{\partial}{\partial \theta} c(\text{DEPS}_c)}_{= 0}
    \right]
\end{align}
Here, $\mathcal{C}$ is the set of cost nodes (the $L_t$); $\mathcal{S}$ is the set of stochastic nodes (in our case, $\mathcal{S} = \{ \tilde{\theta} \}$); $\text{DEPS}_w$ denotes the set of nodes that $w$ depends on; $a \prec^D b$ denotes a deterministic dependence of node $a$ on node $b$ (this holds if there are no stochastic nodes along the path from $a$ to $b$); and $\hat{Q}_w$ is the sum of cost nodes downstream of node $w$.

In our computation graph, $\theta$ does not deterministically influence the cost nodes $L_t$, so the second term inside the expectation is 0.
For Figure~\ref{fig:stochastic-computation-graph}, we have:
\begin{align}
    \frac{\partial}{\partial \theta} \mathbb{E} \left[ \sum_{t=1}^T L_t \right] = \mathbb{E} \left[ \left( \frac{\partial}{\partial \theta} \log p(\tilde{\theta} \mid \theta) \right) \hat{Q}_{\tilde{\theta}} \right]
\end{align}
Because we sample isotropic Gaussian perturbations to $\theta$, we have $p(\tilde{\theta} \mid \theta) = \mathcal{N}(\tilde{\theta} ; \theta, \sigma^2 I) = \frac{1}{\sqrt{2 \pi} \sigma} \exp\left( \frac{- (\tilde{\theta} - \theta)^2}{2 \sigma^2} \right)$.
The gradient of the log probability is:
\begin{align}
    \frac{\partial}{\partial \theta} \log p(\tilde{\theta} \mid \theta)
    =
    \frac{\partial}{\partial \theta} \left( - \frac{1}{2} \log(2\pi) - \log \sigma - \frac{(\tilde{\theta} - \theta)^2}{2 \sigma^2} \right)
    =
    \frac{1}{\sigma^2} ( \tilde{\theta} - \theta )
\end{align}
Using the reparameterization trick, we substitute $\tilde{\theta} = \theta + \epsilon$ where $\epsilon \sim \mathcal{N}(0, \sigma^2 I)$, which yields $\frac{\partial}{\partial \theta} \log p(\tilde{\theta} \mid \theta) = \frac{1}{\sigma^2}(\theta + \epsilon - \theta) = \frac{1}{\sigma^2} \epsilon$.
The sum of cost nodes downstream of $\tilde{\theta}$ is $\sum_{t=1}^T L_t$ because all cost nodes are downstream of $\tilde{\theta}$.
Thus, Theorem 1 from~\citet{schulman2015gradient} gives the following unbiased gradient estimator for the ES-Single computation graph, which is equivalent to the full-unroll ES gradient estimate:
\begin{align}
    \frac{\partial}{\partial \theta} \mathbb{E} \left[ \sum_{t=1}^T L_t \right] = \frac{1}{\sigma^2} \mathbb{E}_{\boldepsilon} \left[ \boldepsilon \sum_{t=1}^T L_t \right]
\end{align}

\paragraph{Vanilla Truncated ES.}
\begin{wrapfigure}[11]{r}{0.25\linewidth}
    \vspace{-0.2cm}
    \centering
    \includegraphics[width=\linewidth]{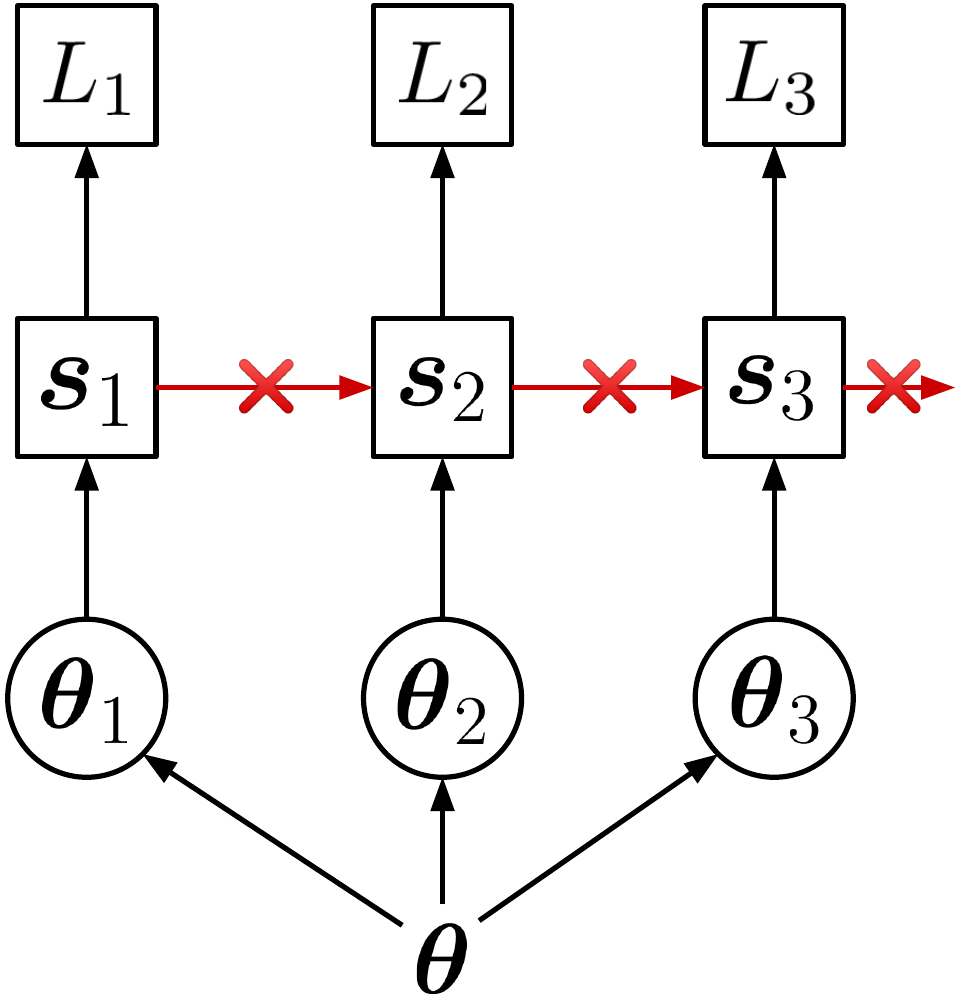}
    \caption{\small Stochastic computation graph for vanilla truncated ES, which does not include the recurrent connections between states.}
    \label{fig:trunc-es-scg}
\end{wrapfigure}
Figure~\ref{fig:trunc-es-scg} shows the stochastic computation graph corresponding to vanilla truncated ES, which shares similar structure to the PES computation graph (Figure~\ref{fig:stochastic-computation-graph}), but does not include recurrent connections between successive states $\bolds_{t-1}, \bolds_t$.
This ignores crucial structure about the problem, e.g., that it consists of a sequence of unrolls rather than a set of independent minimization problems over separate objectives $L_t$.
Once again leveraging Theorem 1 from~\citet{schulman2015gradient}, we derive the following gradient estimator for truncated ES:
\begin{align}
    \frac{\partial}{\partial \theta} \mathbb{E} \left[ \sum_{t=1}^T L_t \right]
    &=
    \mathbb{E}_{\epsilon} \left[ \sum_{t=1}^T \underbrace{\left( \frac{\partial}{\partial \theta} \log p(\theta_t \mid \theta) \right)}_{= \frac{1}{\sigma^2} \epsilon_t} \underbrace{\hat{Q}_{\theta_t}}_{L_t} \right] \\
    &= 
    \mathbb{E}_{\epsilon} \left[ \sum_{t=1}^T \frac{1}{\sigma^2} \epsilon_t L_t \right]
    =
    \frac{1}{\sigma^2} \sum_{t=1}^T \mathbb{E}_{\epsilon_t} \left[ \epsilon_t L_t(\theta + \epsilon_t) \right]
\end{align}

\section{Generalization of PES and ES-Single}
\label{app:generalization-es-pes}

\begin{figure}[H]
    \centering
    \includegraphics[width=\linewidth]{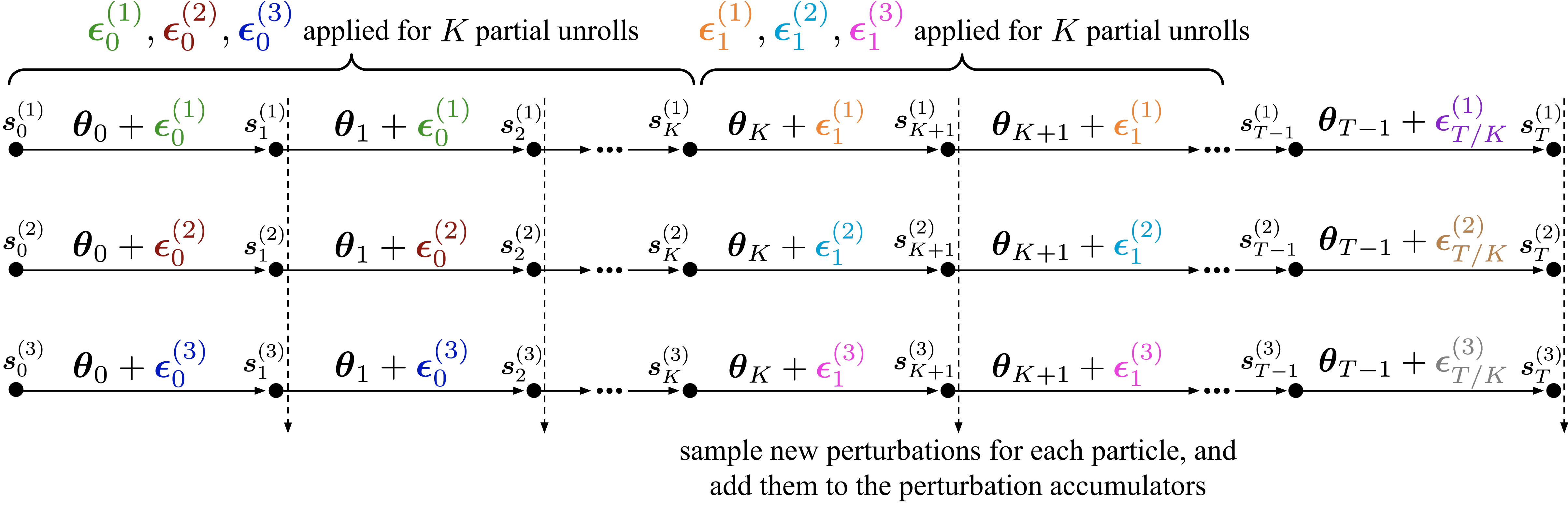}
    \vspace{-0.4cm}
    \caption{Computation graph for a generalization of ES-Single and PES, in which the interval at which we update the outer parameters is decoupled from the interval at which we sample new perturbations and update the perturbation accumulator for each particle. In particular, vanilla PES in each partial unroll samples a new perturbation, updates the accumulator, and updates the outer parameters; in contrast, here we update the outer parameters multiple times using the same perturbation, and only update the perturbations (and the accumulators) every $K$ outer parameter updates.}
    \label{fig:es-single-pes-gen}
\end{figure}

Here, we propose a generalization of both ES-Single and PES, that decouples the interval at which we update the perturbation accumulator from the interval at which we update the outer parameters.
Recall that ES-Single has constant variance regardless of the number of partial unrolls per inner problem (Figure~\ref{fig:variance}).
For long-horizon inner problems optimized using short truncations---yielding a large number of partial unrolls per problem---ES-Single can have substantially lower variance than PES.
However, for certain scenarios (including the the real sequence from the PTB dataset) and unroll lengths (e.g. such that we have $\sim 10$ unrolls per problem), PES has lower variance than ES-Single (Figure~\ref{fig:variance}).
This motivated us to consider an algorithm that generalizes both PES and ES-Single.
In particular, we introduce another hyperparameter, $\Omega$, that specifies the meta-update interval, while $K$ denotes the interval at which new perturbations are sampled and at which the perturbation accumulator is updated.
Many algorithms of interest can be obtained as special cases, by setting $K$ and $\Omega$ appropriately.
Let $T$ be the length of a full inner problem.
Then, 1) if $K = \Omega = T$, we recover full-unroll ES; 2) if $K = \Omega$ and $K < T$, we recover PES; 3) if $K=T$ and $\Omega < T$, we recover ES-Single; and 4) if $K, \Omega < T$ and $\Omega < K$, we obtain a new estimator with a combination of the properties of ES-Single and PES.
Algorithm~\ref{fig:general-algorithm} formally describes the latter case.
Similarly to PES and ES-Single, separate states $\bolds^{(i)}$ are maintained for each particle over the course of an inner problem.
Like PES, a perturbation accumulator $\boldxi^{(i)}$ is maintained for each particle, and 
However, rather than sampling perturbations $\boldepsilon^{(i)}$ per particle for each partial unroll, new perturbations are only sampled every $M$ partial unrolls.
That is, the same perturbation is re-applied for $M$ consecutive unrolls.
Correspondingly, the perturbation accumulator is only updated once every $M$ unrolls.

\begin{figure*}[t]
\vspace{-0.3cm}
\begin{minipage}[t]{0.48\textwidth}
\begin{algorithm}[H]
  \caption{Truncated Evolution Strategies (ES) applied to partial unrolls of a computation graph.}
  \label{alg:es}
\begin{algorithmic}
    \State \textbf{Input:} $\bolds_0$, initial state
    \State \hspace{2.7em} $K$, truncation length for partial unrolls
    \State \hspace{2.7em} $N$, number of particles
    \State \hspace{2.9em} $\sigma$, standard deviation of perturbations
    \State \hspace{2.9em} $\alpha$, learning rate for outer optimization
    \State Initialize $\bolds = \bolds_0$ ${\color{white}\bolds^{(i)} = \bolds_0}$
    \While {inner problem not finished}
        \State $\ges \gets \boldzero$
        \For{$i=1,\dots, N$}
            \State $\boldepsilon^{(i)} = 
                \left\{\begin{array}{lcl}
                	\text{draw from } \mathcal{N}(0, \sigma^2 I) &  & i \text{ odd} \\
                	-\boldepsilon^{(i-1)} &  & i \text{ even}
                \end{array}\right.$
            \State $\hat{L}_K^{(i)} \gets \text{unroll}(\bolds, \boldtheta + \boldepsilon^{(i)}, K)$
            \State $\ges \gets \ges + \boldepsilon^{(i)} \hat{L}_K^{(i)}$
        \EndFor
        \State $\ges \gets \frac{1}{N \sigma^2} \ges$
        \State $\bolds \gets \text{unroll}(\bolds, \boldtheta, K)$
        \State $\boldtheta \gets \boldtheta - \alpha \ges$
    \EndWhile
\end{algorithmic}
\end{algorithm}
\end{minipage}
\hfill
\begin{minipage}[t]{0.48\textwidth}
\begin{algorithm}[H]
  \caption{Generalization of ES-Single and PES, with an arbitrary re-sampling interval $M$.}
  \label{alg:}
\begin{algorithmic}
    \State \textbf{Input:} $\bolds_0$, initial state
    \State \hspace{2.7em} $K$, truncation length for partial unrolls
    \State \hspace{2.7em} $M$, re-sampling interval
    \State \hspace{2.7em} $N$, number of particles
    \State \hspace{2.9em} $\sigma$, standard deviation of perturbations
    \State \hspace{2.9em} $\alpha$, learning rate for outer optimization
    \State Initialize ${\color{dkred}\bolds^{(i)}} = \bolds_0$ for {\color{dkred} $i \in \{1, \dots, N\}$}
    \State {\color{dkred}Initialize $\boldxi^{(i)} \gets \boldzero$ for $i \in \{1, \dots, N \}$}
    \hspace{-0.5cm}\While {inner problem not finished, iteration $j$}
        \If{${\color{dkred}j \mod M = 0}$}
            \For{${\color{dkred}i=1,\dots,N}$}
            \State ${\color{dkred}\boldepsilon^{(i)} = 
                        \left\{\begin{array}{lcl}
                        	\text{draw from } \mathcal{N}(0, \sigma^2 I) &  & i \text{ odd} \\
                        	-\boldepsilon^{(i-1)} &  & i \text{ even}
                        \end{array}\right.}$
            \State ${\color{dkred}\boldxi^{(i)} \gets \boldxi^{(i)} + \boldepsilon^{(i)}}$
            \EndFor
        \EndIf
        \State $\gesgen \gets \boldzero$
        \For{$i=1,\dots, N$}
            \State {\color{dkred} $\bolds^{(i)}$}, $\hat{L}_K^{(i)} \gets \text{unroll}({\color{dkred} \bolds^{(i)}}, \boldtheta + \boldepsilon^{(i)}, K)$
            \State $\gesgen \gets \gesgen + {\color{dkred} \boldxi^{(i)}} \hat{L}_K^{(i)}$
        \EndFor
        \State $\gesgen \gets \frac{1}{N \sigma^2} \gesgen$
        \State {\color{white}$s \gets \text{unroll}(\bolds, \theta, K)$}
        \State $\boldtheta \gets \boldtheta - \alpha \gesgen$
    \EndWhile
\end{algorithmic}
\end{algorithm}
\end{minipage}
\caption{\textbf{A comparison of vanilla ES and the generalized form of PES and ES-Single}, applied to partial unrolls of a computation graph.
The conditional statement for $\boldepsilon^{(i)}$ is used to implement antithetic sampling.
Differences between the two algorithms are {\color{dkred} highlighted in red.}
While ES samples different perturbations for each particle in each partial unroll, ES-Single re-applies the same perturbation over a sequence of partial unrolls, and every $M$ unrolls, it updates the perturbation accumulator and re-samples the perturbations.
}
\label{fig:general-algorithm}
\vspace{-0.2cm}
\end{figure*}

\paragraph{General Stochastic Computation Graph.}
In Figure~\ref{fig:general-computation-graph}, we provide the stochastic computation graph for the generalized estimator, that re-uses the same outer parameter perturbations for a sequence of $K$ partial unrolls before re-sampling the outer parameters.
The resulting unbiased gradient estimator is:
\begin{align}
    \frac{\partial}{\partial \boldtheta} \mathbb{E} \left[ \sum_{t=1}^T L_t \right]
    &=
    \mathbb{E} \left[ \sum_{t=1}^{T/K} \left( \frac{\partial}{\partial \boldtheta} \log p(\boldtheta_t \mid \boldtheta) \right) \hat{Q}_{\boldtheta_t} \right] \\
    &=
    \frac{1}{\sigma^2} \mathbb{E} \left[ \sum_{t=1}^{T/K} \boldepsilon_t \left( \sum_{\tau=kt+1}^T L_{\tau} \right) \right] \\
    &=
    \frac{1}{\sigma^2} \mathbb{E} \left[ \boldepsilon_1 (L_1 + L_2 + \cdots + L_K) + (\boldepsilon_1 + \boldepsilon_2) (L_{K+1} + \cdots + L_{2K}) + \cdots + (\boldepsilon_1 + \cdots + \boldepsilon_{T/K}) L_T \right] \\
    &=
    \frac{1}{\sigma^2} \mathbb{E} \left[ \sum_{t=1}^{T/K} \left( \sum_{\tau=1}^t \boldepsilon_\tau \right) \left( L_{tK+1} + \cdots + L_{2tK} \right) \right]
\end{align}

This estimator has variance equivalent to PES where the number of unrolls per inner problem is $K$.
The main benefit is that it allows for more frequent updates to the outer parameters, while maintaining this fixed variance, which is determined by a hyperparameter that can be tuned independently of the update frequency.

\begin{figure}[H]
    \centering
    \includegraphics[width=0.7\linewidth]{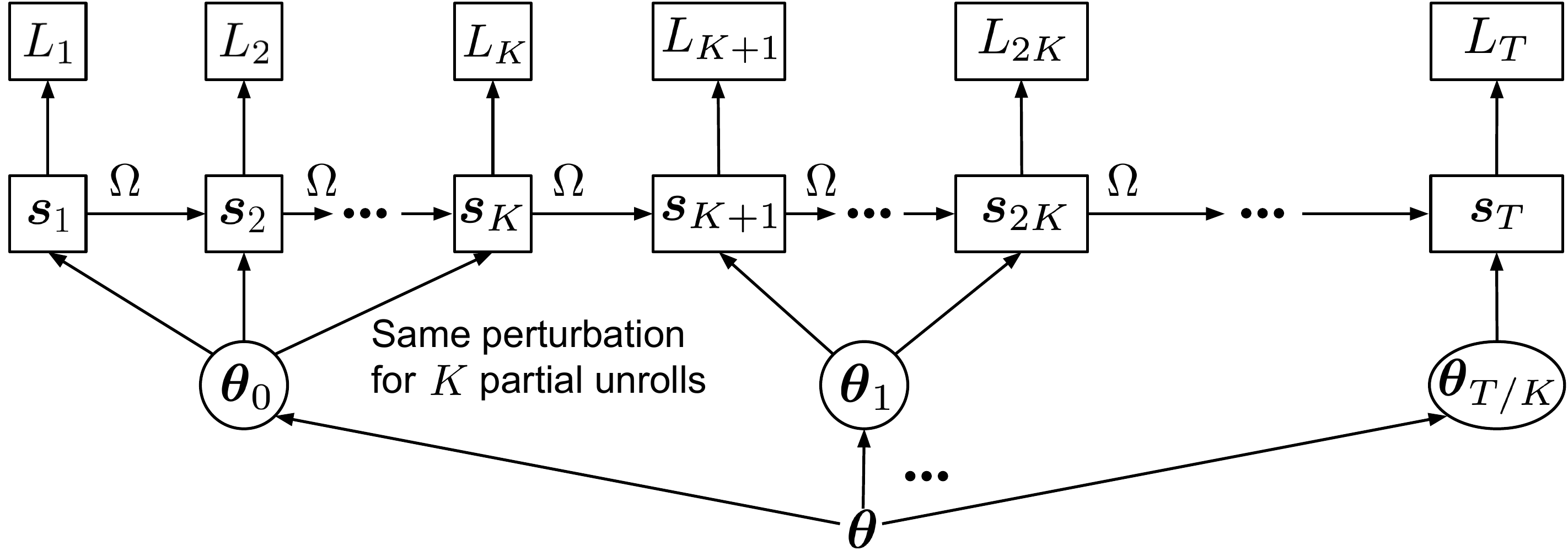}
    \caption{Computation graph corresponding to a generalization of ES-Single and PES, where we re-apply the same perturbed outer parameters (and use the same accumulated perturbations) over a sequence of $K$ partial unrolls of the inner problem. After each sequence of unrolls, the perturbations are re-sampled and the perturbation accumulator is updated. Note that the computation graphs for ES-Single and PES shown in Figure~\ref{fig:stochastic-computation-graph} are special cases corresponding to two extremes: 1) in ES-Single, the same perturbation is applied across all partial unrolls, yielding a single stochastic node $\tilde{\boldtheta}$ for the inner problem---in this case, the perturbation accumulators are equivalent to the perturbations themselves (e.g., they are the sum of only one term); and 2) in PES, a different perturbation is applied in each partial unroll---yielding $T$ stochastic nodes $\{ \boldtheta_t \}_{t=1}^T$ per inner problem---and the accumulators are updated after each unroll.
    }
    \label{fig:general-computation-graph}
\end{figure}

\section{Variance Analysis of a Generalized Estimator}
\label{app:bias-variance-generalized}

\begin{wrapfigure}[11]{r}{0.35\linewidth}
    \vspace{-0.5cm}
    \centering
    \includegraphics[width=\linewidth]{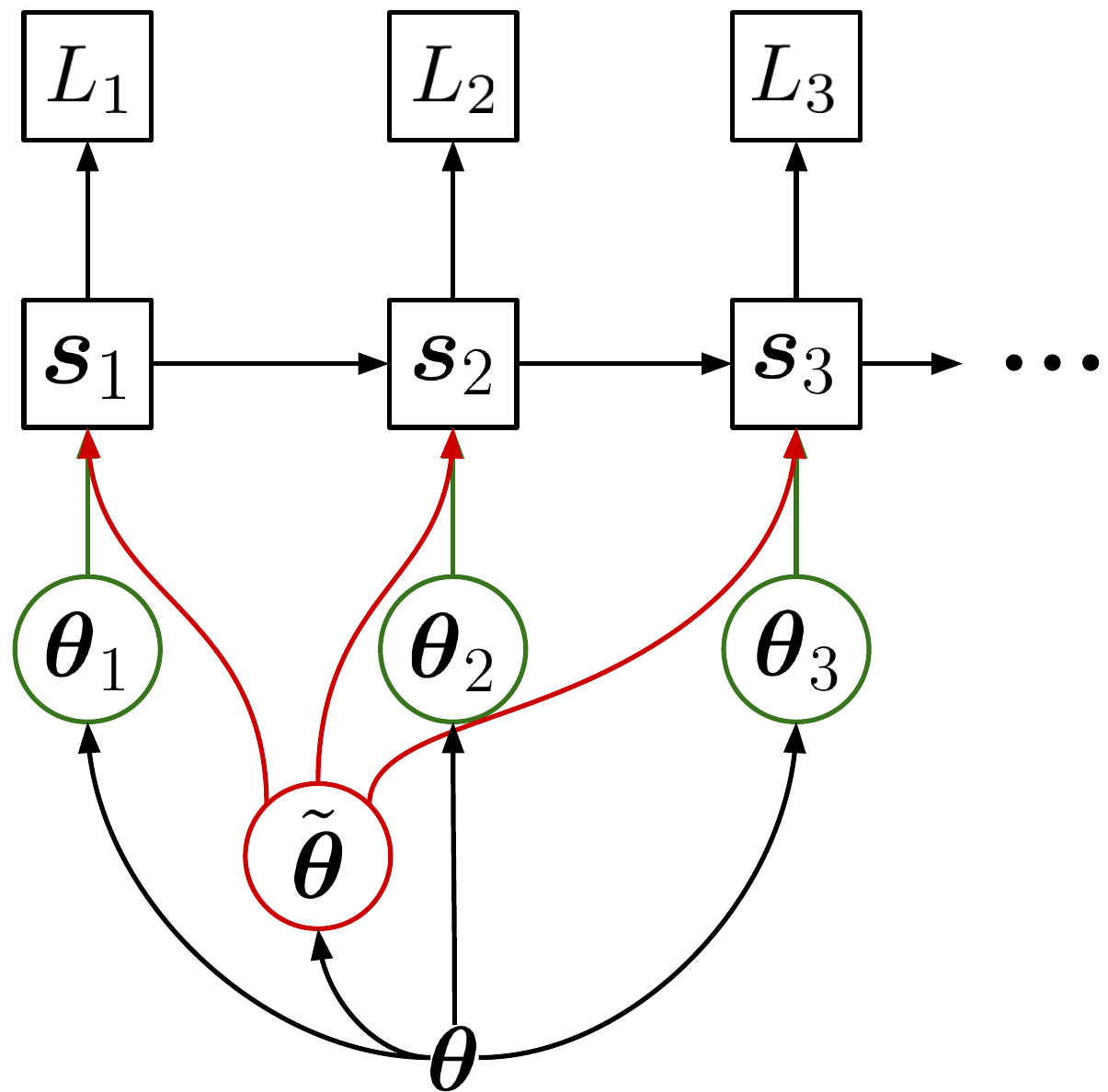}
    \caption{\small Stochastic computation graph for a generalization of PES and ES-Single, where we consider two types of perturbations: 1) a single perturbation that is applied in each unroll over the course of a full inner problem; and 2) a separate perturbation sampled for each partial unroll.}
    \label{fig:es-pes-generalization}
\end{wrapfigure}

In this section, we introduce an estimator that generalizes ES-Single and PES by combining a single perturbation that is kept fixed across all partial unrolls (as in ES-Single) with perturbations that are sampled independently for each partial unroll (as in PES).
One way to obtain this estimator is to consider the stochastic computation graph in Figure~\ref{fig:es-pes-generalization}.
Here, each stochastic node $\boldtheta_t$ and $\tilde{\boldtheta}$ depends only on $\boldtheta$, e.g., $\text{DEPS}_{\boldtheta_t} = \{ \boldtheta \} \forall t$ and $\text{DEPS}_{\tilde{\boldtheta}} = \{ \boldtheta \}$.
Then, the gradient estimator is:
\begin{align}
    \hat{\boldg}
    &=
    \mathbb{E} \left[ \sum_{w \in \mathcal{S}, \boldtheta \prec^D w} \left( \frac{\partial}{\partial \boldtheta} \log p(w \mid \text{DEPS}_w \right) \hat{Q}_w \right] \\
    &=
    \mathbb{E} \left[ \frac{\partial}{\partial \boldtheta} \log p(\tilde{\boldtheta} \mid \boldtheta) \hat{Q}_{\tilde{\boldtheta}} + \sum_{t=1}^T \left( \frac{\partial}{\partial \boldtheta} \log p(\boldtheta_t \mid \boldtheta) \right) \hat{Q}_{\boldtheta_t} \right] \\
    &=
    \mathbb{E} \left[ \frac{1}{\sigma^2} \boldepsilon_S \left( \sum_{t=1}^T L_t \right) + \sum_{t=1}^T \frac{1}{\sigma^2} \boldepsilon_t \left( \sum_{\tau=t}^T L_\tau \right) \right] \\
    &=
    \frac{1}{\sigma^2} \mathbb{E} \left[ \sum_{t=1}^T \left( \boldepsilon_S + \sum_{\tau=1}^t \boldepsilon_\tau \right) L_t \right]
\end{align}

In the following, we adopt notation from~\citet{vicol2021unbiased}: rather than writing the loss at step $t$ as a function of the parameters $\boldtheta$ and state $\bolds_t$, $L_t(\bolds_t, \boldtheta)$, we drop the dependence on $\bolds_t$ and explicitly denote the dependence of $L_t$ on the sequence of applications of $\boldtheta$ over time, $L_t(\boldtheta_1, \boldtheta_2, \dots, \boldtheta_t)$.
This allows us to keep track of how the applications of $\boldtheta$ contribute to the total loss gradient.
In addition, we denote by $\Theta$ a matrix whose rows are the per-timestep parameters $\boldtheta_t$, where $\boldtheta_t = \boldtheta, \forall t$.
Then, we can write $L_t(\Theta)$ as shorthand for $L_t(\boldtheta_1, \dots, \boldtheta_t)$.
Finally, we use $\boldxi_t$ to denote the PES perturbation accumulator, that sums the perturbations up to step $t$, $\boldxi_t = \sum_{\tau=1}^t \boldepsilon_\tau$.
Assuming that the objective is quadratic, and that we use antithetic sampling, we can write the estimator as:
\begin{align} \label{eq:es-gen}
    \gesgen
    =
    \frac{1}{\alpha^2 \sigma^2 + \beta^2 \sigma^2}
    \mathbb{E}_{\boldepsilon} \left[ \sum_{t=1}^T \left( \alpha \boldepsilon_s + \beta \sum_{\tau=1}^t \boldepsilon_t \right) \text{vec}(\alpha \boldepsilon_s + \beta \boldepsilon_{1..t})^\top \nabla_{\text{vec}(\Theta_{1..t})} L_t(\Theta) \right]
\end{align}
The expression $\text{vec}(\alpha \boldepsilon_s + \beta \boldepsilon_{1..t})$ uses broadcasting to add the single perturbation $\boldepsilon_s$ to each of the $t$ per-unroll perturbations in $\boldepsilon_{1..t}$.
This estimator generalizes ES-Single and PES: setting $\alpha=1, \beta=0$ recovers ES-Single, while setting $\alpha=0, \beta=1$ recovers PES.
Note that:
\begin{align}
    \text{vec}(\alpha \boldepsilon_s + \beta \boldepsilon_{1..t})^\top \nabla_{\text{vec}(\Theta_{1..t}} L_t(\Theta)
    &=
    \sum_{\tau=1}^t \left( \alpha \boldepsilon_s + \beta \boldepsilon_\tau \right)^\top \nabla_{\boldtheta_\tau} L_t(\Theta)
\end{align}
% ==================================================

\subsection{Variance}

We assume that the inner problem is quadratic, and that we use antithetic sampling.
Given a single particle pair for antithetic sampling, we have the following estimator:
\begin{align}
    \gesgen
    =
    \frac{1}{\alpha^2 \sigma^2 + \beta^2 \sigma^2} \sum_{t=1}^T (\alpha \boldepsilon_s + \beta \boldxi_t) \text{vec}(\alpha \boldepsilon + \beta \boldepsilon_{1..t})^\top \nabla_{\text{vec}(\Theta_{1..t})} L_t(\Theta)
\end{align}

Similarly to~\citet{maheswaranathan2019guided} and~\citet{vicol2021unbiased}, we quantify the variance of $\gesgen$ using the total variance $\tr(\Var(\gesgen))$:
\begin{align}
    \tr(\Var(\gesgen))
    &=
    \tr \left( \mathbb{E}_{\boldepsilon} \left[ \gesgen \gesgen^\top \right] - \mathbb{E}\left[ \gesgen \right] \mathbb{E}\left[ \gesgen \right]^\top \right) \\
    &=
    \underbrace{\mathbb{E}_{\boldepsilon} \left[ \gesgen^\top \gesgen \right]}_{\num{1}} - \underbrace{\mathbb{E}_{\boldepsilon} \left[ \gesgen \right]^\top \mathbb{E}_{\boldepsilon} \left[ \gesgen \right]}_{\num{2}}
\end{align}
Term $\num{2}$ is simple, because $\gesgen$is unbiased, so $\mathbb{E}_{\boldepsilon} \left[ \gesgen \right] = \nabla_{\boldtheta} L(\Theta)$.
Thus,
\begin{align}
    \num{2} = \mathbb{E}_{\boldepsilon} \left[ \gesgen \right]^\top \mathbb{E}_{\boldepsilon} \left[ \gesgen \right]
    =
    \nabla_{\boldtheta} L(\Theta)^\top \nabla_{\boldtheta} L(\Theta)
    =
    \| \nabla_{\boldtheta} L(\Theta) \|^2
\end{align}

To deal with term $\num{1}$, we will decompose $\gesgen^\top \gesgen$ into simpler expressions and use the linearity of expectation to compute each component.
We will use the shorthand $\boldv_t \equiv \text{vec}(\alpha \boldepsilon_s + \beta \boldepsilon_{1..t})$ and $\boldg_t \equiv \nabla_{\text{vec}(\Theta_{1..t})} L_t(\Theta)$.
Note that $\boldv_t^\top \boldg_t = \sum_{\tau=1}^t (\alpha \boldepsilon_s + \beta \boldepsilon_\tau)^\top \nabla_{\boldepsilon_\tau} L_t(\Theta)$.
Next, we expand out $\gesgen^\top \gesgen$ using this shorthand:
\begin{align}
    \gesgen^\top \gesgen
    &=
    \frac{1}{(\alpha^2 \sigma^2 + \beta^2 \sigma^2)} \left( \sum_{t=1}^T (\alpha \boldepsilon_s + \beta \boldxi_t) \boldv_t^\top \boldg_t \right)^\top \left( \sum_{t=1}^T (\alpha \boldepsilon_s + \beta \boldxi_t) \boldv_t^\top \boldg_t \right) \\
    % &=
    % \frac{1}{(\alpha^2 \sigma^2 + \beta^2 \sigma^2}
    % \left( (\alpha \boldepsilon_s + \beta \boldxi_1) \boldv_1^\top \boldg_1 + \cdots + (\alpha \boldepsilon_s + \beta \boldxi_T) \boldv_T^\top \boldg_T \right)^\top \left( (\alpha \boldepsilon_s + \beta \boldxi_1) \boldv_1^\top \boldg_1 + \cdots + (\alpha \boldepsilon_s + \beta \boldxi_T) \boldv_T^\top \boldg_T \right) \\
    &=
    \frac{1}{(\alpha^2 \sigma^2 + \beta^2 \sigma^2)} \left[ \underbrace{\boldg_1^\top \boldv_1 (\alpha \boldepsilon_s + \beta \boldxi_1)^\top (\alpha \boldepsilon_s + \beta \boldxi_1) \boldv^\top \boldg_1}_{\num{a}}
    +
    \underbrace{\boldg_1^\top \boldv_1 (\alpha \boldepsilon_s + \beta \boldxi_1)^\top (\alpha \boldepsilon_s + \beta \boldxi_2) \boldv_2^\top \boldg_2}_{\num{b}} + \cdots
    \right]
\end{align}

There are two types of terms in this expression: terms of type $\num{a}$, that have the form $\boldg_i^\top \boldv_i (\alpha \boldepsilon_s + \beta \boldxi_i)^\top (\alpha \boldepsilon_s + \beta \boldxi_i) \boldv_i^\top \boldg_i$, and terms of type $\num{b}$ that have the form $\boldg_i^\top \boldv_i (\alpha \boldepsilon_s + \beta \boldxi_i)^\top (\alpha \boldepsilon_s + \beta \boldxi_j) \boldv_j^\top \boldg_j$ where $i \neq j$.

\subsubsection{Terms of type $\num{a}$.}
First, note that $(\alpha \boldepsilon_s + \beta \boldxi_i)^\top (\alpha \boldepsilon_s + \beta \boldxi_i)$ can be expanded as follows:
\begin{align}
    (\alpha \boldepsilon_s + \beta \boldxi_i)^\top (\alpha \boldepsilon_s + \beta \boldxi_i)
    &=
    \alpha^2 \boldepsilon_s^\top \boldepsilon_s + 2 \alpha \beta \boldepsilon_s^\top (\boldepsilon_1 + \cdots + \boldepsilon_i) + \beta^2 \left( \sum_{m=1}^i \boldepsilon_m^\top \boldepsilon_m + \sum_{m \leq i, n \leq i, m \neq n} \boldepsilon_m^\top \boldepsilon_n \right)
\end{align}
To simplify notation, we will use the shorthand $W = \sum_{m=1}^i (\alpha \boldepsilon_s + \beta \boldepsilon_m)^\top \nabla_{\boldtheta_m} L_i(\Theta)$.
Then, for terms of type $\num{a}$, we need to compute:
\begin{align}
    \underbrace{W^\top (\alpha^2 \boldepsilon_s^\top \boldepsilon_s) W}_{\num{i}}
    +
    \underbrace{W^\top (2 \alpha \beta \boldepsilon_s^\top (\boldepsilon_1 + \cdots + \boldepsilon_i)) W}_{\num{ii}}
    +
    \underbrace{\beta^2 W^\top \left( \sum_{m=1}^i \boldepsilon_m^\top \boldepsilon_m \right) W}_{\num{iii}}
    +
    \underbrace{\beta^2 W^\top \left( \sum_{m \leq i, n \leq i, m \neq n} \boldepsilon_m^\top \boldepsilon_n \right) W}_{\num{iv}}
\end{align}

\paragraph{Term $\num{i}$.}
There are two types of terms that have non-zero expectation: ones with the structure $\boldepsilon_s \boldepsilon_s^\top \boldepsilon_s \boldepsilon_s^\top$, and ones with the structure $\boldepsilon_m \boldepsilon_s^\top \boldepsilon_s \boldepsilon_m^\top$.
The contribution from the first type of term is:
\begin{align}
    \alpha^4 \sigma^4 (P + 2) \left( \sum_{m=1}^i \nabla_{\boldtheta_m} L_i(\Theta) \right)^\top \left( \sum_{m=1}^i \nabla_{\boldtheta_m} L_i(\Theta) \right)
\end{align}
And the contribution from the second type of term is:
\begin{align}
    \alpha^2 \beta^2 \sigma^4 P \sum_{m=1}^i \nabla_{\boldtheta_m} L_i(\Theta)^\top \nabla_{\boldtheta_m} L_i(\Theta)
\end{align}

\paragraph{Term $\num{ii}$.}
Here, we have two types of terms that have non-zero expectations: ones with the structure $\boldepsilon_s \boldepsilon_s^\top \boldepsilon_m \boldepsilon_m^\top$ and ones with the structure $\boldepsilon_m \boldepsilon_s^\top \boldepsilon_m \boldepsilon_s^\top$.
We will have $i$ terms of each of these sub-types.
The total contribution of these terms is:
\begin{align}
    4 \alpha^2 \beta^2 \sigma^4 \sum_{m=1}^i \left( \sum_{m'=1}^i \nabla_{\boldtheta_{m'}} L_i(\Theta) \right)^\top \nabla_{\boldtheta_m} L_i(\Theta)
    =
    4 \alpha^2 \beta^2 \sigma^4 \left( \sum_{m=1}^i \nabla_{\boldtheta_m} L_i(\Theta) \right)^\top  \left( \sum_{m=1}^i \nabla_{\boldtheta_m} L_i(\Theta) \right)
\end{align}

\paragraph{Term $\num{iii}$.}
Here, we will have non-zero terms of three types: $\boldepsilon_s \boldepsilon_m^\top \boldepsilon_m \boldepsilon_s^\top$, $\boldepsilon_m \boldepsilon_m^\top \boldepsilon_m \boldepsilon_m^\top$, and $\boldepsilon_m \boldepsilon_n^\top \boldepsilon_n \boldepsilon_m^\top$.
The contribution from the first type of term is:
\begin{align}
    \alpha^2 \beta^2 \sigma^4 P i \left( \sum_{m=1}^i \nabla_{\boldtheta_m} L_i(\Theta) \right)^\top \left( \sum_{m=1}^i \nabla_{\boldtheta_m} L_i(\Theta) \right)
\end{align}
The contribution from the second type of term is:
\begin{align}
    \beta^4 \sigma^4 (P + 2) \sum_{m=1}^i \nabla_{\boldtheta_m} L_i(\Theta)^\top \nabla_{\boldtheta_m} L_i(\Theta)
\end{align}
And the contribution from the third type of term ($\boldepsilon_m \boldepsilon_n^\top \boldepsilon_n \boldepsilon_m$) is:
\begin{align}
    \beta^4 \sigma^4 P \sum_{n=1}^i \left( \sum_{m \leq i, m \neq n} \nabla_{\boldtheta_m} L_i(\Theta)^\top \nabla_{\boldtheta_n} L_i(\Theta) \right)
\end{align}

\paragraph{Term $\num{iv}$.}
Here, the nonzero terms arise from two structures, $\boldepsilon_m \boldepsilon_m^\top \boldepsilon_n \boldepsilon_n^\top$ and $\boldepsilon_m \boldepsilon_n^\top \boldepsilon_m \boldepsilon_n^\top$.
The combined contribution from both types of terms is:
\begin{align}
    2 \beta^4 \sigma^4 \sum_{m \leq i, n \leq i, m \neq n} \nabla_{\boldtheta_m} L_i(\Theta)^\top \nabla_{\boldtheta_n} L_i(\Theta)
\end{align}

\subsubsection{Terms of Type $\num{b}$.}
Next, we are interested in terms of type $\num{b}$, which have the form $\boldg_i^\top \boldv_i (\alpha \boldepsilon_s + \beta \boldxi_i)^\top (\alpha \boldepsilon_s + \beta \boldxi_j) \boldv_j^\top \boldg_j$ where $i \neq j$.
In this subsection, we will denote the minimum of $i$ and $j$ by $r \equiv \min(i, j)$.
We can expand $(\alpha \boldepsilon_s + \beta \boldxi_i)^\top (\alpha \boldepsilon_s + \boldxi_j)$ as follows:
\begin{align}
    (\alpha \boldepsilon_s + \beta \boldxi_i)^\top (\alpha \boldepsilon_s + \beta \boldxi_j)
    &=
    (\alpha \boldepsilon_s + \beta (\boldepsilon_1 + \cdots + \boldepsilon_i))^\top (\alpha \boldepsilon_s + \beta (\boldepsilon_1 + \cdots + \boldepsilon_j)) \\
    &=
    \alpha^2 \boldepsilon_s^\top \boldepsilon_s + \alpha \beta \boldepsilon_s^\top (\boldepsilon_1 + \cdots + \boldepsilon_i) + \alpha \beta \boldepsilon_s^\top (\boldepsilon_1 + \cdots + \boldepsilon_j) + \\
    &\quad + \beta^2 (\boldepsilon_1 + \cdots + \boldepsilon_i)^\top (\boldepsilon_1 + \cdots + \boldepsilon_j) \\
    &=
    \underbrace{\alpha \boldepsilon_s^\top \boldepsilon_s}_{\num{i}} + \underbrace{\alpha \beta \boldepsilon_s^\top (\boldepsilon_1 + \cdots + \boldepsilon_i)}_{\num{ii}} + \underbrace{\alpha \beta \boldepsilon_s^\top (\boldepsilon_1 + \cdots + \boldepsilon_j)}_{\num{iii}} \\
    &\quad +
    \beta^2 \Bigg(
    \underbrace{\sum_{m=1}^r \boldepsilon_m^\top \boldepsilon_m}_{\num{iv}}
    +
    \underbrace{\sum_{m \leq i, n \leq j m \neq n} \boldepsilon_m^\top \boldepsilon_n}_{\num{v}}
    \Bigg)
\end{align}

\paragraph{Term $\num{i}$.}
We have two types of terms with non-zero expectations: $\boldepsilon_s \boldepsilon_s^\top \boldepsilon_s \boldepsilon_s^\top$ and $\boldepsilon_m \boldepsilon_s^\top \boldepsilon_s \boldepsilon_m^\top$.
In particular, we have a single instance of the former type of term, and $r$ instances of the latter type.
The total contribution of the first term is:
\begin{align}
    \alpha^4 \sigma^4 (P + 2) \left( \sum_{m=1}^i \nabla_{\boldtheta_m} L_i(\Theta) \right)^\top \left( \sum_{n=1}^j \nabla_{\boldtheta_n} L_j(\Theta) \right)
\end{align}
The contribution from terms of the second type is:
\begin{align}
    \alpha^2 \beta^2 \sigma^4 P \sum_{m=1}^{r} \nabla_{\boldtheta_m} L_i(\Theta)^\top \nabla_{\boldtheta_m} L_j(\Theta)
\end{align}

\paragraph{Term $\num{ii}$.}
Here, two types of terms have non-zero expectations: we have $r$ terms of type $\boldepsilon_m \boldepsilon_s^\top \boldepsilon_m \boldepsilon_s^\top$ and $r$ terms of type $\boldepsilon_s \boldepsilon_s^\top \boldepsilon_m \boldepsilon_m^\top$.
The total contribution of both terms is:
\begin{align}
    \alpha^2 \beta^2 \sigma^4 \sum_{m=1}^r \left( \sum_{m'=1}^i \nabla_{\boldtheta_{m'}} L_i(\Theta) \right) \nabla_{\boldtheta_m} L_j(\Theta)
    +
    \alpha^2 \beta^2 \sum_{m=1}^r \nabla_{\boldtheta_m} L_i(\Theta)^\top \left( \sum_{m'=1}^j \nabla_{\boldtheta_{m'}} L_j(\Theta) \right)
\end{align}

\paragraph{Term $\num{iii}$.}
Term $\num{iii}$ is symmetrical to term $\num{ii}$, swapping $i$ for $j$.
Thus, its contribution is identical to that of term $\num{ii}$:
\begin{align}
    \alpha^2 \beta^2 \sigma^4 \sum_{m=1}^r \left( \sum_{m'=1}^i \nabla_{\boldtheta_{m'}} L_i(\Theta) \right) \nabla_{\boldtheta_m} L_j(\Theta)
    +
    \alpha^2 \beta^2 \sum_{m=1}^r \nabla_{\boldtheta_m} L_i(\Theta)^\top \left( \sum_{m'=1}^j \nabla_{\boldtheta_{m'}} L_j(\Theta) \right)
\end{align}

\paragraph{Term $\num{iv}$.}
Here, we have three types of terms with non-zero expectation: $r$ terms of the form $\boldepsilon_s \boldepsilon_m^\top \boldepsilon_m \boldepsilon_s^\top$, $r$ terms of the form $\boldepsilon_m \boldepsilon_m^\top \boldepsilon_m \boldepsilon_m^\top$, and $r(r-1)$ terms of the form $\boldepsilon_n \boldepsilon_m^\top \boldepsilon_m \boldepsilon_n^\top$ where $n \neq m$.

The contribution from the first type, $\boldepsilon_s \boldepsilon_m^\top \boldepsilon_m \boldepsilon_s^\top$, is:
\begin{align}
    \alpha^2 \beta^2 \sigma^4 P \sum_{m=1}^r \left( \sum_{m'=1}^i \nabla_{\boldtheta_{m'}} L_i(\Theta) \right)^\top \left( \sum_{n'=1}^j \nabla_{\boldtheta_{n'}} L_j(\Theta) \right)
\end{align}

The contribution from the second type, $\boldepsilon_m \boldepsilon_m^\top \boldepsilon_m \boldepsilon_m^\top$, is:
\begin{align}
    \beta^4 \sigma^4 (P + 2) \sum_{m=1}^r \nabla_{\boldtheta_m} L_i(\Theta)^\top \nabla_{\boldtheta_m} L_j(\Theta)
\end{align}

The contribution from the third type, $\boldepsilon_n \boldepsilon_m^\top \boldepsilon_m \boldepsilon_n^\top$ where $n \neq m$, is:
\begin{align}
    \beta^4 \sigma^4 P \sum_{n=1}^r \left( \sum_{m \leq r, m \neq n} \nabla_{\boldtheta_m} L_i(\Theta)^\top \nabla_{\boldtheta_m} L_j(\Theta) \right)
\end{align}

\paragraph{Term $\num{v}$.}
Here, we have two types of terms with non-zero expectation: $\boldepsilon_n \boldepsilon_m^\top \boldepsilon_n \boldepsilon_m^\top$ and $\boldepsilon_m \boldepsilon_m^\top \boldepsilon_n \boldepsilon_n^\top$.
The contribution of these terms is:
\begin{align}
    \beta^4 \sigma^4 \sum_{m \leq i, n \leq j, m \neq n} \nabla_{\boldtheta_m} L_i(\Theta)^\top \nabla_{\boldtheta_n} L_j(\Theta)
    +
    \beta^4 \sigma^4 \sum_{m \leq i, n \leq j, m \neq n} \nabla_{\boldtheta_n} L_i(\Theta)^\top \nabla_{\boldtheta_m} L_j(\Theta)
\end{align}

\paragraph{Combining Terms.}
Overall, we have $T$ terms of type $\num{a}$, e.g., with structure $\boldg_i^\top \boldv_i (\alpha \boldepsilon_s + \beta \boldxi_i)^\top (\alpha \boldepsilon_s + \beta \boldxi_i) \boldv_i^\top \boldg_i$, with total contribution:
\begin{align}
    & \sum_{i=1}^T
    \Bigg(
    \alpha^4 \sigma^4 (P + 2) \left( \sum_{m=1}^i \nabla_{\boldtheta_m} L_i(\Theta) \right)^\top \left( \sum_{m=1}^i \nabla_{\boldtheta_m} L_i(\Theta) \right) \\
    &\qquad + 
    \alpha^2 \beta^2 \sigma^4 P \sum_{m=1}^i \nabla_{\boldtheta_m} L_i(\Theta)^\top \nabla_{\boldtheta_m} L_i(\Theta)
    +
    4 \alpha^2 \beta^2 \sigma^4 \left( \sum_{m=1}^i \nabla_{\boldtheta_m} L_i(\Theta) \right)^\top  \left( \sum_{m=1}^i \nabla_{\boldtheta_m} L_i(\Theta)
    \right) \\
    &\qquad + 
    \alpha^2 \beta^2 \sigma^4 P i \left( \sum_{m=1}^i \nabla_{\boldtheta_m} L_i(\Theta) \right)^\top \left( \sum_{m=1}^i \nabla_{\boldtheta_m} L_i(\Theta) \right)
    +
    \beta^4 \sigma^4 (P + 2) \sum_{m=1}^i \nabla_{\boldtheta_m} L_i(\Theta)^\top \nabla_{\boldtheta_m} L_i(\Theta)
    \\
    &\qquad +
    \beta^4 \sigma^4 P \sum_{n=1}^i \left( \sum_{m \leq i, m \neq n} \nabla_{\boldtheta_m} L_i(\Theta)^\top \nabla_{\boldtheta_n} L_i(\Theta) \right)
    +
    2 \beta^4 \sigma^4 \sum_{m \leq i, n \leq i, m \neq n} \nabla_{\boldtheta_m} L_i(\Theta)^\top \nabla_{\boldtheta_n} L_i(\Theta)
    % &\qquad +
    % \beta^4 \sigma^4 (P + 2) \sum_{m=1}^i \nabla_{\boldtheta_m} L_i(\Theta)^\top \nabla_{\boldtheta_m} L_i(\Theta)
    \Bigg)
\end{align}

In addition, we have several terms of type $\num{b}$, e.g., with structure $\boldg_i^\top \boldv_i (\alpha \boldepsilon_s + \beta \boldxi_i)^\top (\alpha \boldepsilon_s + \beta \boldxi_j)^\top \boldv_j^\top \boldg_j$, with total contribution:
\begin{align}
    \sum_{i \neq j}
    \Bigg(
    &
    \alpha^4 \sigma^4 (P + 2) \left( \sum_{m=1}^i \nabla_{\boldtheta_m} L_i(\Theta) \right)^\top \left( \sum_{n=1}^j \nabla_{\boldtheta_n} L_j(\Theta) \right)
    +
    \alpha^2 \beta^2 \sigma^4 P \sum_{m=1}^{r} \nabla_{\boldtheta_m} L_i(\Theta)^\top \nabla_{\boldtheta_m} L_j(\Theta) \\
    &\quad +
    2 \alpha^2 \beta^2 \sigma^4 \sum_{m=1}^r \left( \sum_{m'=1}^i \nabla_{\boldtheta_{m'}} L_i(\Theta) \right) \nabla_{\boldtheta_m} L_j(\Theta)
    +
    2 \alpha^2 \beta^2 \sum_{m=1}^r \nabla_{\boldtheta_m} L_i(\Theta)^\top \left( \sum_{m'=1}^j \nabla_{\boldtheta_{m'}} L_j(\Theta) \right) \\
    &\quad +
    \alpha^2 \beta^2 \sigma^4 P \sum_{m=1}^r \left( \sum_{m'=1}^i \nabla_{\boldtheta_{m'}} L_i(\Theta) \right)^\top \left( \sum_{n'=1}^j \nabla_{\boldtheta_{n'}} L_j(\Theta) \right) \\
    &\quad +
    \beta^4 \sigma^4 (P + 2) \sum_{m=1}^r \nabla_{\boldtheta_m} L_i(\Theta)^\top \nabla_{\boldtheta_m} L_j(\Theta)
    +
    \beta^4 \sigma^4 P \sum_{n=1}^r \left( \sum_{m \leq r, m \neq n} \nabla_{\boldtheta_m} L_i(\Theta)^\top \nabla_{\boldtheta_m} L_j(\Theta) \right) \\
    &\quad +
    \beta^4 \sigma^4 \sum_{m \leq i, n \leq j, m \neq n} \nabla_{\boldtheta_m} L_i(\Theta)^\top \nabla_{\boldtheta_n} L_j(\Theta)
    +
    \beta^4 \sigma^4 \sum_{m \leq i, n \leq j, m \neq n} \nabla_{\boldtheta_n} L_i(\Theta)^\top \nabla_{\boldtheta_m} L_j(\Theta)
    \Bigg)
\end{align}

Setting $\alpha=0, \beta=1$ recovers the PES variance, while setting $\alpha=1, \beta=0$ recovers the variance of ES-Single.
Note that most of the terms in the variance of this generalized estimator include coefficient $\beta$; thus, intuitively one would expect that setting $\beta=0$ (for ES-Single) has reduced variance relative to PES.

\section{Code}
\label{app:code}
Code Listing~\ref{lst:code} provides a self-contained JAX implementation of the ES-Single gradient estimator, that reproduces the result in Figure~\ref{fig:toy-regression}.

\begin{lstlisting}[language=Python, caption={Self-contained implementation of ES-Single in JAX, for the 2D regression problem in Figure~\ref{fig:toy-regression}.}, label={lst:code}]
from functools import partial
import jax
import jax.numpy as jnp

import optax

def loss(x):
    """Inner loss."""
    return jnp.sqrt(x[0]**2 + 5) - jnp.sqrt(5) + jnp.sin(x[1])**2 * \
              jnp.exp(-5*x[0]**2) + 0.25*jnp.abs(x[1] - 100)

# Gradient of inner loss
loss_grad = jax.grad(loss)

def update(state, i):
    """Performs a single inner problem update, e.g., a single unroll step.
    """
    (L, x, theta, t_curr, T, K) = state
    lr = jnp.exp(theta[0]) * (T - t_curr) / T + jnp.exp(theta[1]) * t_curr / T
    x = x - lr * loss_grad(x)
    L += loss(x) * (t_curr < T)
    t_curr += 1
    return (L, x, theta, t_curr, T, K), x

@partial(jax.jit, static_argnames=('T', 'K'))
def unroll(x_init, theta, t0, T, K):
    """Unroll the inner problem for K steps.

    Args:
      x_init: the initial state for the unroll
      theta: a 2-dimensional array of outer parameters (log_init_lr, log_final_lr)
      t0: initial time step to unroll from
      T: maximum number of steps for the inner problem
      K: number of steps to unroll
    
    Returns:
      L: the loss resulting from the unroll
      x_curr: the updated state at the end of the unroll
    """
    L = 0.0
    initial_state = (L, x_init, theta, t0, T, K)
    state, outputs = jax.lax.scan(update, initial_state, None, length=K)
    (L, x_curr, theta, t_curr, T, K) = state
    return L, x_curr

@partial(jax.jit, static_argnames=('T', 'K', 'sigma', 'N'))
def es_single_grad(key, xs, theta, t0, T, K, sigma, N):
    """Compute ES-Single gradient estimate.

    Args:
      key: JAX PRNG key
      xs: Nx2 array of particles/states to be updated
      theta: a 2-dimensional array of outer parameters (log_init_lr, log_final_lr)
      t0: initial time step for the current unroll
      T: maximum number of steps for the inner problem
      K: truncation length for the unroll
      sigma: standard deviation of the Gaussian perturbations
      N: number of perturbations (as N//2 antithetic pairs)
    
    Returns:
      theta_grad: ES-Single gradient estimate
      xs: Nx2 array of updates particles/states
    """
    # Generate antithetic perturbations
    pos_perts = jax.random.normal(key, (N//2, theta.shape[0])) * sigma  # Antithetic pos
    neg_perts = -pos_perts  # Antithetic neg
    perts = jnp.concatenate([pos_perts, neg_perts], axis=0)

    # Unroll the inner problem for K steps using the antithetic perturbations of theta
    L, xs = jax.vmap(unroll, in_axes=(0,0,None,None,None))(xs, theta + perts, t0, T, K)
    # Compute the ES-Single gradient estimate
    theta_grad = jnp.mean(perts * L.reshape(-1, 1) / (sigma**2), axis=0)
    return theta_grad, xs

T = 100      # Total inner problem length
K = 10       # Truncation length for partial unrolls
N = 100      # Number of particles in total (N//2 antithetic pairs)
sigma = 0.1  # Standard deviation of perturbations

t = 0
theta = jnp.log(jnp.array([0.01, 0.01]))
x = jnp.array([1.0, 1.0])
xs = jnp.ones((N, 2)) * jnp.array([1.0, 1.0])

optimizer = optax.adam(1e-2)
opt_state = optimizer.init(theta)

key = jax.random.PRNGKey(3)
for i in range(10000):
    if t >= T:
        # Reset the inner problem: the inner iteration, inner parameters, and random key
        key, skey = jax.random.split(key)
        t = 0
        xs = jnp.ones((N, 2)) * jnp.array([1.0, 1.0])
        x = jnp.array([1.0, 1.0])

    theta_grad, xs = es_single_grad(key, xs, theta, t, T, K, sigma, N)
    
    updates, opt_state = optimizer.update(theta_grad, opt_state)
    theta = optax.apply_updates(theta, updates)
    
    t += K

    if i % 100 == 0:
        # Run a full unroll for evaluation
        L, _ = unroll(jnp.array([1.0, 1.0]), theta, 0, T, T)
        print(i, jnp.exp(theta), theta_grad, L)
\end{lstlisting}

\end{document}